\documentclass[twoside,11pt]{article}

\usepackage{jmlr2e_backup}

\newcommand{\dataset}{{\cal D}}

\usepackage[utf8]{inputenc} 
\usepackage[T1]{fontenc}    
\usepackage{hyperref}       
\usepackage{url}            
\usepackage{booktabs}       
\usepackage{amsfonts}       
\usepackage{nicefrac}       
\usepackage{microtype}      
\usepackage{amsmath}
\usepackage{bbm}
\usepackage{pifont}
\newcommand{\cmark}{\textcolor{green}{\ding{51}}}%
\newcommand{\xmark}{\textcolor{red}{\ding{55}}}%
\newcommand{\omark}{\textcolor{orange}{\ding{108}}}%
\usepackage{cleveref}
\usepackage{pgffor}
\usepackage{graphicx}
\usepackage{float}
\usepackage{subcaption}

\usepackage[utf8]{inputenc}
\usepackage[table]{xcolor}
\usepackage{array}
\usepackage{multirow}
\usepackage[usestackEOL]{stackengine}
\usepackage{titlesec}

\newcolumntype{x}[1]{>{\centering\arraybackslash}m{#1}}
\newcolumntype{y}[1]{>{\raggedright\arraybackslash}m{#1}}

\makeatletter
\renewcommand\paragraph{\@startsection{paragraph}{4}{\z@}%
                                     {-3.25ex\@plus -1ex \@minus -.2ex}%
                                     {1.5ex \@plus .2ex}%
                                     {\normalfont\normalsize\bfseries}}
\titleformat{\subsubsection}
       {\normalfont\fontsize{11}{17}\bfseries}{\thesubsubsection}{1em}{}

\titleformat{\section}
       {\normalfont\fontsize{12}{17}\bfseries}{\thesection}{1em}{}

\makeatother

\usepackage{amsthm}
\theoremstyle{plain}
\newtheorem{theorem}{Theorem}[section]

\newtheorem{lemma}[theorem]{Lemma}

\newtheorem{definition}[theorem]{Definition}

\title{Generalization bounds for deep learning}

\author{\name Guillermo Valle-P\'erez \\
  \email guillermo.valle@dtc.ox.ac.uk \\
  \addr Department of Theoretical Physics\\
  University of Oxford\\
  \AND
  \name Ard A. Louis \\
  \email ard.louis@physics.ox.ac.uk \\
  \addr Department of Theoretical Physics\\
  University of Oxford\\
}

\editor{}

\begin{document}

\maketitle

\begin{abstract}
  Generalization in deep learning has been the topic of much recent theoretical and empirical research. Here we introduce desiderata for techniques that predict generalization errors for deep learning models in supervised learning. Such predictions should 
   1) scale correctly with data complexity; 2) scale correctly with training set size; 3) capture differences between architectures; 4) capture differences between optimization algorithms; 5) be quantitatively not too far from the true error (in particular, be non-vacuous); 6) be efficiently computable; and 7) be rigorous.  We focus on generalization error upper bounds, and  introduce a categorisation of bounds depending on assumptions on the algorithm and data. We  review a wide range of existing approaches, from classical VC dimension to recent PAC-Bayesian bounds, commenting on how well they perform against the desiderata.  
   We next use a function-based picture to derive a marginal-likelihood PAC-Bayesian bound.  This bound is, by one definition, 
   optimal up to a multiplicative constant in the asymptotic limit of large training sets, as long as the learning curve follows a power law, which is typically found in practice for deep learning problems.  Extensive empirical analysis   demonstrates that our marginal-likelihood PAC-Bayes bound fulfills desiderata 1-3 and 5. The results for 6 and 7 are promising, but not yet fully conclusive, while only desideratum 4 is currently beyond the scope of our bound.  Finally, we comment on why this function-based bound performs significantly better than current parameter-based PAC-Bayes bounds.
\end{abstract}

\section{Introduction}

While deep learning can provide remarkable empirical performance in many machine learning tasks,  a theoretical account of  this success remains elusive.  One of the main metrics of performance in machine learning is generalization, or the ability to predict correctly on new unseen data. Statistical learning theory offers a theoretical framework to understand generalization, with the main tool being the derivation of upper bounds on the generalization error.
However, it is now widely acknowledged  that ``classical'' approaches in learning theory, based on worst-case analyses, are not sufficient to explain generalization of deep learning models, see e.g.~\citet{bartlett1998sample,zhang2016understanding}. 
This observation has spurred a large amount of work on applying and developing other learning theory approaches to deep learning, some of which are showing promising results. However, current approaches are still lacking in many ways: generalization error bounds are often vacuous or offer little explanatory power by showing trends opposite to the true error as one varies important quantities. For example, \citet{nagarajan2019uniform} show that some common spectral norm-based bound predictions increase with increasing training set size, while the measured error decreases.   Moreover, many bounds are published without being sufficiently tested, so that it can be hard to assess their strengths and weaknesses. 

There is therefore a need for a more systematic approach to studying generalization bounds.  Two important recent studies have taken up this challenge.  Firstly, \citet{jiang2019fantastic} preformed extensive empirical tests of the predictive power of over $40$ different generalization measures for a range of commonly used hyperparameters.  
They mainly studied flatness and norm-based measures, and compared them by reporting an average of the performance.
In their comparisons, they focused mostly on training hyperparameters.  Architectures were only compared by changing width and depth of an architecture resembling the Network-in-Network \citep{gao2011robustness}, and they did not systematically study  variation on the data (dataset complexity or training set size). 

Secondly, in a follow-up study \citet{dziugaite2020search}, argue that average performance is not a sufficient metric of generalization performance.   They empirically demonstrate that  generalization measures can exhibit correct correlations for certain parameter changes and seeds, but badly fail for others.  In addition to some of the hyperparameter changes studied by  \citet{jiang2019fantastic}, they also consider the effect of changing dataset complexity and training set size.

In this paper we also perform a systematic analysis of generalization bounds.  We first describe, 
in \cref{desiderata}, a set of seven qualitatively different desiderata 
which, we argue, should  be used to guide the kind of experiments\footnote{Here, `experiment' can be interpreted in the formal sense which \citet{dziugaite2020search} introduce} needed to fully explore the quality of the predictions of a theory of the generalization.
These include four desiderata for making correct predictions when varying architecture, dataset complexity, dataset size, or optimizer choice, as well as three further desiderata, which describe the importance of quantitative agreement,  computational efficiency  and theoretical rigour.    Our set of desiderata is broader than those used in previous tests of generalization theories, and will, we hope, lead to a better understanding of the strengths and weaknesses of competing approaches.

Rather than empirically testing the many bounds that can be found in the literature,  we take a more general theoretical approach here.  We first, in \cref{generalization_theory},
present an overall framework for classifying the many different approaches to deriving frequentist bounds of the generalization error. This framework is then used to systematically organize the discussion of the different approaches in the literature.
For each approach we review existing results in the literature and compare performance for the seven desiderata we propose. Such an analysis allows us to draw some more general conclusions about the kinds of strategies for deriving generalization bounds that may be most successful. In this context we also note, that for some proposed bounds, there is not enough empirical evidence in the literature to determine whether or not they satisfy our desiderata. This lacuna highlights the importance of the proposals put forward by \citet{jiang2019fantastic} and \citet{dziugaite2020search} of large scale empirical studies, which we also advocate for in this paper.

Inspired by our analysis of existing bounds we present in \cref{pac_bayes_bound} a high-probability version of the realizable PAC-Bayes bound introduced by \citet{mcallester1998some} and applied to deep learning in \citet{valle2018deep}.
Because our bound, which is derived from a function-based picture,  is directly proportional to the Bayesian evidence or marginal likelihood of the model, we call it the \emph{marginal likelihood PAC-Bayes bound}. 
While our bound has frequentist origins,   the connection between marginal likelihood and generalization is an important theme in Bayesian analyses, and can be traced back to the work of MacKay and Neal \citep{mackay1992practical,neal1994priors,mackay2003information}, as well as the early work on PAC-Bayes \citep{shawe1997pac,mcallester1998some}.  See also an important recent discussion of this connection in \citet{wilson2020bayesian}.

Recent large-scale empirical work has shown that learning curves for deep-learning have an extensive power-law regime~\citep{hestness2017deep,spigler2019asymptotic,kaplan2020scaling,henighan2020scaling}. 
 The exponents in the power-law depend  mainly on the data rather than on architectures, with smaller exponents for more complex datasets.   Under the assumption of a power-law asymptotic behaviour with training set size, we are able to prove that our marginal-likelihood PAC-Bayes bound is asymptotically optimal up to a constant in the limit of large training sets.  
 
In order to test our marginal-likelihood PAC-Bayes bound against the full set of desiderata, we perform over 1,000  experiments  comparing 19 different architectures, ranging from simple fully connected networks (FCNs) and convolutional neural networks (CNNs) to more sophisticated architectures such as Resnet50, Densenet121 and Mobilenetv2. We are thus including variations in architecture hyperparameters such as pooling type, the  number of layers, skip connections, etc\ldots.  These networks are  tested on
5 different computer vision datasets of varying complexity.    For each of these 95 architecture/data combinations, we train on a range of training set sizes to study learning curves.   We additionally study the effect of label corruption on two of the data sets, MNIST and CIFAR-10.   Our bound provides non-vacuous  generalization bounds, which are also tight enough to provide a reasonably good quantitative guide to generalization.  It also does remarkably well at predicting qualitative trends in the generalization error as a function of data complexity, architecture (including, for example, differences between average and max pooling), and training set size. In particular,  we find that we can estimate the value of the learning curve power law exponent for different datasets.  

In our concluding section, we argue that the use of a function based perspective is a key to understanding the relatively good performance of our marginal likelihood bound across so many desiderata. This approach contrasts with many other PAC-Bayesian bounds in the deep learning theory literature,  which use distributions in parameter space, mostly as a way to measure flatness.  We also discuss potential weaknesses of our approach, such as its inability to capture the effect of DNN width or of optimiser choice (but see~\citet{mingard2020sgd}).  
Finally, we  argue that capturing trends in generalization when  architecture and data are varied is not only theoretically interesting, but also important for many applications including architecture search and data augmentation.

\section{Desiderata for predictive theories of generalization error}
\label{desiderata}

The fundamental question of what constitutes a ``good'' theory of generalization in deep learning is a profound and contested one. \citet{jiang2019fantastic} and \citet{dziugaite2020search} both consider this question and suggest that a good theory should capture the causal mechanisms behind generalization. \citet{jiang2019fantastic} use a measure of conditional independence to estimate the causal relation between different complexity measure and generalization. \citet{dziugaite2020search} consider a stronger notion that tries to capture whether a theory predicts well the generalization error over \emph{all} possible experimental interventions. To formalize this notion they look at distributional robustness which quantifies the worst-case predictive performance of a theory.

Here we take a slightly different approach to this question by defining a seven key desiderata that a ``good'' theory of generalization should satisfy.
This set is broader than that considered by \citet{jiang2019fantastic} and \citet{dziugaite2020search} in that it considers desiderata beyond simply the predictive performance.

Furthermore, rather than focusing only on the average case or worst-case performance of a predictive theory, we argue that a precise formulation of the performance of a generalization measure necessarily depends on the application. Therefore in our discussions in \cref{comparing_bounds}, and in the extensive experiments in \cref{experimental_results}, we aim to paint a fine-grained picture of how well the theory predicts generalization for different experimental settings, and how it fares at other important desiderata.

 We focus our attention on generalization error upper bounds, but the same ideas could apply to other types of theories that aim to predict generalization error (for example, those based on Bayesian assumptions on the data distribution).  
 
 The seven desiderata which we propose a good predictive theory of generalization for deep learning should aim to satisfy are listed below:

\begin{description}

    \item [\textbf{D.1}]\textbf{data complexity:} The predicted error should correctly scale with data complexity. In other words, the predicted error should correlate well with the true error when the dataset is changed. For example, a fixed DNN, for a fixed training set size, will typically have higher error for CIFAR10 than for MNIST, and an even higher error for a label-corrupted dataset. The bound should capture such differences.
    
    \item [\textbf{D.2}]\textbf{training set size:} The predicted error should correctly scale with training set size. That is, the predicted error should follow the same trend as the true error as the number of training examples increases. For example, it has been found  empirically that the generalization error often follows a power law decay with training set size with an exponent which depends on the dataset, but not strongly on the architecture
    \citep{hestness2017deep,novak2019neural,rosenfeld2019constructive,kaplan2020scaling}. 
    
    \item [\textbf{D.3}]\textbf{architectures:} The predicted error should capture differences in generalization between architectures. Different architectures can display significant variation in generalization performance. For example, CNNs with pooling tend to generalize better than CNNs without pooling on image classification tasks. The predicted error should aim to predict these differences. Furthermore, one of the more puzzling properties of  DNNs is that the performance depends only weakly on the number of parameters used for an architecture,  provided the system is large enough~\citep{belkin2019reconciling,nakkiran2019deep}.   Since this question of why DNNs generalize so well in the overparameterised regime is one of the central questions in the theory of DNNs, it is particularly important that a predicted error can reproduce this  relative insensitivity to the number of parameters. 
    
    \item [\textbf{D.4}]\textbf{optimization algorithms:} The predicted error should capture differences in generalization between different optimization algorithms. Different optimization algorithms used in training DNNs, as well as different choices of training hyperparameters such as SGD batch size and learning rates, or regularization techniques, can lead to differences in generalization, which the theory should aim to predict.

    \item [\textbf{D.5}]\textbf{non-vacuous:} The predicted error should be quantitatively close to the true error. For generalization error upper bounds, it is commonly required that the error is less than $100\%$, as any upper bound higher than that is satisfied by definition and thus can be considered ``vacuous''. Beyond this requirement, the bound should aim to be as close as possible to the true error (a property often referred to as being a \emph{tight} bound).
    
    \item [\textbf{D.6}]\textbf{efficiently computable:} The predicted error should be efficiently computable. This requirement is particularly important if the prediction or bound is to be useful in practical applications (for example, in architecture search). It is of little use having a bound that cannot in practice be calculated.
    
    \item [\textbf{D.7}]\textbf{rigorous:} The prediction should be rigorous. For the case of upper bounds, this means that the bound comes with a theorem that guarantees its validity under a well-specified set of assumptions, which should be met on the domain where it is applied. In practice, theoretical results are often applied to domains where the assumptions aren't known to hold (or are known \emph{not} to hold), but rigorous guarantees offer the highest level of confidence on one's predictions, and should be aimed for whenever possible.
\end{description}

A theory of generalization error should aim to satisfy as many of these desiderata as possible. However, the trade-offs that one is willing to make depend on the domain of application of the theory. For example, a rigorous proof may be more valuable for theoretical insight than for a practical application, where computational efficiency may be more important. As another example, for an application to neural architecture search one may care more about the scaling with architecture, while for an application to decision making for data collection one may care more about correct scaling with training set size.

\section{Classifying generalization bounds for deep learning}\label{generalization_theory}

In this section, we provide an overview of existing approaches to predict generalization error.   After some preliminary sections  describing notation and definitions we present a general taxonomy which classifies generalization error upper bounds based on the key assumptions they make.

\subsection{Notation for the supervised learning problem}\label{notation}

Supervised learning deals with the problem of predicting outputs from inputs, given a set of example input-output pairs. The inputs live in an input domain $\mathcal{X}$, and the outputs belong to an output space $\mathcal{Y}$. We will mostly focus on the binary classification setting where $\mathcal{Y} = \{0,1\}$\footnote{See \cref{binary_vs_multiclass} for a brief comment on extensions to the multi-class setting for the generalization bounds we study here}. We assume that there is a \emph{data distribution} $\mathcal{D}$ on the set of input-output pairs $\mathcal{X}\times\mathcal{Y}$. The \emph{training set} $S$ is a sample of $m$ input-output pairs sampled i.i.d. from $\mathcal{D}$, $S=\{(x_i,y_i)\}_{i=1}^m \sim \mathcal{D}^m$, where $x_i \in \mathcal{X}$ and $y_i \in \mathcal{Y}$.
We will refer to the data distribution as \emph{noiseless} if it can be factored as $\mathcal{D}((x,y)) = \mathcal{D}_X(x)\mathcal{D}_{Y|X}(y|x)$ where $\mathcal{D}_{Y|X}$ is deterministic, that is it has support on a single value in $y=f(x)$. In this case $f$ is called the \emph{target function}.

We define a \emph{loss function} $L: \mathcal{Y}\times\mathcal{Y} \to \mathbb{R}$, which measures how ``well'' a prediction $\hat{y}\in \mathcal{Y}$ matches an observed output $y\in \mathcal{Y}$, by assigning to it a score $L(\hat{y}, y)$ which is high when they don't match.  In supervised learning, we define a \emph{hypothesis}\footnote{This is sometimes also called a predictor, or classifier, in the context of supervised learning} $h$ as a function from inputs to outputs $h: \mathcal{X} \to \mathcal{Y}$ and the \emph{risk} $R$ of a hypothesis as the expected value of the loss of the predicted outputs on new samples $R(h)=\mathbf{E}_{(x,y)\sim\mathcal{D}}[L(h(x),y)]$. For the case of classification, where $\mathcal{Y}$ is a discrete set, we focus on the \emph{0-1 loss} function (also called \emph{classification loss}) $L(\hat{y}, y) = \mathbbm{1}[\hat{y}\neq y]$, where $\mathbbm{1}$ is the indicator function. We define the \emph{generalization error} $\epsilon$ as the expected risk using this loss, which equals the probability of misclassification.
\begin{equation}
    \epsilon(h) = \mathbf{P}_{(x,y)\sim \mathcal{D}}[h(x) \neq y]
\end{equation}
This is the central quantity which we study here. 
We also define the \emph{training error} $\hat{\epsilon}(h,S)=\frac{1}{m}\sum_{(x,y)\in S}\mathbbm{1}[h(x)\neq y]$.  For a more general loss function, we also define the \emph{empirical risk}  $\widehat{R}(h,S)=\frac{1}{m}\sum_{(x,y)\in S}L(h(x),y)$, as the empirical average of the loss function over the training set.  To simplify notation, we often simply write $\widehat{R}(h)$ and $\hat{\epsilon}(h)$ for $\widehat{R}(h,S)$ and $\hat{\epsilon}(h,S)$, respectively.

Finally, we define a \emph{learning algorithm} to be a mapping $\mathcal{A}: (\mathcal{X}\times\mathcal{Y})^* \to \mathcal{Y}^\mathcal{X}$ from training sets of any size to hypothesis functions. For simplicity, we mainly describe deterministic hypotheses and learning algorithms, but most results should be easily generalizable to stochastic versions. A stochastic learning algorithm will map training sets to a probability distribution over hypothesiss, while a stochastic hypothesis maps inputs in $\mathcal{X}$ to proability distributions over outputs in $\mathcal{Y}$. For PAC-Bayes we will consider stochastic learning algorithms.

The supervised learning problem consists of finding a learning algorithm $\mathcal{A}$ that produces low generalization error $\epsilon(h)$ (in expectation, or with high probability), under certain well defined assumptions (including none)  on the data distribution $\mathcal{D}$.

\subsection{General PAC learning framework}

The modern theory of statistical learning theory originated with Leslie Valiant's probably approximately correct (PAC) framework \citep{valiant1984theory}, and Vapnik and Chervonenkis' uniform convergence analysis \citep{vapnik1968uniform,vapnik1974theory,vapnik1995nature}.   
Here we describe a general frequentist formulation of generalization error bounds of the type studied in supervised learning theory. 
Although the original PAC framework  included the condition of computational efficiency of the learning algorithm in its definition, the term is used more widely now \citep{guedj2019primer}. In the most general case, the PAC learning framework we present here proves confidence bounds that establish a relationship between observed quantities (derived from the training set) and the unobserved generalization error $\epsilon(h)$, which is the real quantity of interest. Specifically we consider bounds on  $D$, a function measuring the difference between generalization and training error (which we refer to as the \emph{generalization gap}),  which state that, under some (or no) assumptions on the data distribution $\mathcal{D}$ and algorithm $\mathcal{A}$,  the following bound holds:
\begin{equation}
    \label{general-pac}
   \forall \mathcal{D},~\mathbf{P}_{S\sim \mathcal{D}^m}[D(\epsilon(\mathcal{A}(S)), \hat{\epsilon}(\mathcal{A}(S))) \leq \frac{f_\mathcal{A}(S)}{m^\alpha}] \geq 1-\delta
\end{equation}
with probability at least  $1-\delta \in [0,1]$. Here the probability $\mathbf{P}_{S\sim \mathcal{D}^m}$ is for the event inside the square brackets when $S$ is sampled from $\mathcal{D}^m$. $f_\mathcal{A}$ is a function which, following common practice in the literature, we will call \emph{capacity}. The capacity  measures some notion of the ``complexity'' of the algorithm, the data, or both. $\delta$ is called the confidence parameter and measures the probability with which the bound may fail because of getting an ``unusual'' training set. Finally,  the exponent $0 < \alpha \leq 1$.  The most common measure of generalization gap $D$ is the absolute difference between the generalization and training error, but in some general PAC-Bayesian analyses $D$ may be any convex function \citep{rivasplata2020pac}. The capacity $f_\mathcal{A}$ can also take many forms. Some common examples include the VC dimension (section \ref{vc_dimension}) and the KL divergence of a posterior and prior for PAC-Bayes bounds (section \ref{srm_bounds}). Sometimes the dependence on training set is fully absorbed into $f_\mathcal{A}$ so that $1/m^\alpha$ is omitted.

We also distinguish two types of bounds
\begin{itemize}
    \item \emph{generalization gap bounds} are bounds on $D(\epsilon(\mathcal{A}(S)), \hat{\epsilon}(\mathcal{A}(S))) = |\epsilon(h)-\hat{\epsilon}(h)|$ or some other measure of discrepancy $D$ between the generalization and training error
    \item \emph{generalization error bounds} are bounds on a function of $\epsilon(h)$ alone.
\end{itemize}
Note that a generalization gap bound immediately implies a generalization error bound, but not necessarily vice versa. If $\hat{\epsilon}(h)=0$ (realizability assumption), the distinction between these bounds doesn't exist.

Finally, in order to simplify notation, we typically omit dependence of the capacity on  the parameter $\delta$ throughout the paper.

\subsection{A classification of generalization bounds}

In this section we classify the main types of bounds which are possible under the general PAC framework described above, according to the different assumptions they make on the data distribution or algorithm, and on the different quantities the capacity $f_\mathcal{A}(S)$ may depend on.

\subsubsection{According to assumptions on the data $\mathcal{D}$}

The PAC framework is characterized by bounds which make as few assumptions about the data distribution $\mathcal{D}$ as possible. There are two main approaches: agnostic and realizable bounds.
\begin{itemize}
  
\item \textbf{Agnostic bounds}. For \emph{agnostic} or \emph{distribution-free} bounds, no assumption is made on the data distribution. In this case, the No Free Lunch theorem \citep{wolpert1994relationship} implies that we can't guarantee a small generalization error, but may still be able to guarantee a small difference between the training and generalization error (generalization gap), or a small generalization error for some training sets (for data-dependent bounds).

\item \textbf{Realizable bounds}. For \emph{realizable} bounds, the data distribution and the algorithm are assumed to be such that $\hat{\epsilon}(\mathcal{A}(S))=0$ for any $S$ which has non-zero probability. This is saying that the algorithm is always able to fit the data perfectly. Note that this is a combined assumption about the algorithm (for instance the expressivity of its hypothesis class) and the data distribution. If the algorithm is \emph{fully expressive} (able to express any function), and minimizes the training error (it belongs to the class of \emph{empirical risk minimization} (ERM) algorithms), then this condition doesn't put any constraint on $\mathcal{D}$ beyond being noiseless.

For the realizable case, bounds usually take the form
\begin{equation}
    \label{general-realizable-pac}
   \forall \mathcal{D},~\mathbf{P}_{S\sim \mathcal{D}^m}\left[\epsilon(\mathcal{A}(S)) \leq \frac{f_\mathcal{A}(S)}{m}\right] \geq 1-\delta
\end{equation}
where we have omitted (without loss of generality), the exponent $\alpha$, because realizable bounds have $\alpha=1$ in most cases (The reason for this exponent comes fundamentally from the non-Gaussian behaviour of the tail of the binomial distribution when the mean is close to zero~\citep{langford2005tutorial}.).  
  
\end{itemize}
Statistical learning theory has focused on these two classes of assumptions, because they are considered to be minimal, and cover most cases of interest. However, other assumptions on $\mathcal{D}$ are sometimes used in more advanced analyses, typically involving data-dependent and non-uniform bounds (see below). 

In this context it is interesting to compare a frequentist to a Bayesian approach to assumptions on the  data.  In the former case, one only assumes that $\mathcal{D}$ belongs to a restricted  set of possible data distributions. This approach  naturally leads to to studying the worst case generalization over that set.   In a Bayesian approach, one assumes a prior over data distributions, and then studies the typical or average case generalization.  Most of bounds we will describe here use the frequentist approach, but there have been some interesting results using a Bayesian prior over data distributions, see  \cref{average_generalization_error}. 

\subsubsection{According to assumptions on the algorithm $\mathcal{A}$}

As is the case for $\mathcal{D}$, supervised learning theory often makes very minimal assumptions on the learning algorithm $\mathcal{A}$, though there are also a rich set of algorithm-dependent analyses. The minimal assumptions made on $\mathcal{A}$ is that it outputs hypotheses within a set of hypothesiss called the \emph{hypothesis class} $\mathcal{H}$. As this is the minimal assumption often made, we will refer to it as ``algorithm-independent'' to distinguish it from the approaches which make stronger assumptions. We thus classify bounds on the following two classes.
\begin{itemize}
    \item 
\textbf{Algorithm-independent bounds}. These bounds only assume that $\mathcal{A}(S)\in\mathcal{H}$, for all $S$\footnote{or for all $S$ within a set of $S$ which have support $1-\delta$, as in \citet{nagarajan2019uniform}}. Furthermore, the capacity $f_\mathcal{A}(S)$ can only depend on $h=\mathcal{A}(S)$ and $S$, and not in a general way on $\mathcal{A}$\footnote{We include this constraint because otherwise there wouldn't be any useful distinction with algorithm-dependent bounds. We could just take a set of algorithm-dependent bounds that cover all algorithms with outputs in $\mathcal{H}$, and make an ``algorithm-independent bound''. We choose the definition to avoid this possibility.}. This definition means that bounds in this class must bound the generalization gap or the generalization error\footnote{Under these definitions, realizable bounds should technically be algorithm-dependent, because they assume something beyond the hypothesis class of the algorithm -- they assume that the algorithm is ERM. However, we will still refer to bounds that only add this extra assumption as algorithm dependent, as the ERM assumption may be considered as being relatively weak. However, this distinction becomes important when discussing some subtleties, like those in \cref{non_uniform_algor_dep}.} for the worst case hypothesis $h$ from a hypothesis class $\mathcal{H}$. We can therefore write bounds in this class in the following way
 
\begin{equation}
\label{nonuniform-convergence}
   \forall \mathcal{D},~\mathbf{P}_{S\sim \mathcal{D}^m}\left[\forall h\in\mathcal{H},~D(\epsilon(h),\hat{\epsilon}(h)) \leq \frac{f(S,h)}{m^\alpha}\right] \geq 1-\delta
\end{equation}

Algorithm-independent bounds are commonly classified in two classes, according to the dependence of the capacity $f(S,h)$ on $h$:
\begin{itemize}
    \item \textbf{Uniform convergence bounds} (or \emph{uniform bounds} for short) are algorithm-independent bounds where the capacity is independent of $h$. This includes VC dimension bounds (\cref{vc_dimension}) and Rademacher complexity bounds (\cref{rademacher_complexity}).
    Here the nomenclature "uniform" means the value of the bound is independent of the hypothesis.
    
    \item \textbf{Non-uniform convergence bounds} (or \emph{non-uniform bounds} for short) are algorithm-independent bounds where the capacity depends on $h$. Common examples include bounds for structural risk minimization (section \ref{srm_bounds}).
\end{itemize}

\item 
\textbf{Algorithm-dependent bounds}. This type of bounds generalizes the above class of bounds by considering stronger, or more general, assumptions on $\mathcal{A}$, as well as more general dependence of the capacity on $\mathcal{A}$. We can express this general class of bounds as
\begin{equation}
    \label{nonuniform-convergence2}
   \forall \mathcal{A}\in\mathfrak{A}, ~ \forall \mathcal{D},~\mathbf{P}_{S\sim \mathcal{D}^m}\left[D(\epsilon(\mathcal{A}(S))), \hat{\epsilon}(\mathcal{A}(S))) \leq \frac{f_\mathcal{A}(S)}{m^\alpha}\right] \geq 1-\delta
\end{equation}
where $\mathfrak{A}$ is a set of algorithms that represents our assumptions on $\mathcal{A}$. An example of this class of bounds are stability-based bounds (\cref{stability_bounds}), which rely on the assumption that the algorithm's output $\mathcal{A}(S)$ doesn't depend strongly on $S$. The case where $\mathfrak{A}=\{\mathcal{A}\}$ corresponds to the analysis of a particular algorithm. We will refer to this special case as \emph{algorithm-specific bounds}. However, in almost all cases, analyses of particular algorithms rely only on a subset of the properties of the algorithm, so that the results often apply more generally and are not restricted to a specific algorithm.

\end{itemize} 

\subsubsection{According to dependence of capacity on training set $S$}

Another  classification on bounds is according to the dependence of the capacity $f$ on the training dataset $S$. For any of the classes of bounds discussed above, we further distinguish the following two classes of bounds.
\begin{itemize}
    \item 
\textbf{Data-independent bounds} are bounds in which the capacity does not depend on $S$. Common examples include VC dimension bounds, and the simple (algorithm-independent) bounds based on SRM (e.g.\ \cref{srm1}, and \cref{srm2}).

\item 
\textbf{Data-dependent bounds} are bounds in which the capacity depends on $S$. Common examples are Rademacher complexity bounds. The PAC-Bayes bounds for Bayesian algorithms (such as the one we present in \cref{pac_bayes_theorem}) is another example.
\end{itemize}

Note that in this paper we do not consider explicitly data-distribution-dependent bounds, because we only consider bounds that depend on quantities derivable from $S$. However, we will consider the behaviour of bounds under different assumptions on the data distribution $\mathcal{D}$. In particular, we could look at the expected value of the bound, or its whole distribution, for different $\mathcal{D}$. This is the sense in which we will discuss ``data distribution dependence'' in this paper. Note that only data-dependent bounds, which depend on $S$, can depend on the data distribution in this sense.

\subsubsection{Further comments on non-uniform bounds and algorithm-dependent bounds}\label{non_uniform_algor_dep}

A common way to derive algorithm-dependent bounds is to start with non-uniform convergence bounds, and then make further assumptions on the algorithm which restrict the set of $h$ that may be returned for a given $S$. We can then use the maximum value of the bound among those $h$ as an algorithm-dependent bound valid for algorithms which satisfy the assumptions. Note that this makes the bound automatically data-dependent. This can be expressed by bounds of the form
\begin{equation}
\label{nonuniform-convergence-algorithm-dependent}
   \forall \mathcal{D},~\mathbf{P}_{S\sim \mathcal{D}^m}\left[\forall h\in\mathcal{H}(S),~D(\epsilon(h),\hat{\epsilon}(h)) \leq \frac{f(S,h)}{m^\alpha}\right] \geq 1-\delta
\end{equation}
where $\mathcal{H}(S)$ is a set which includes $\mathcal{A}(S)$ for the class of algorithms $\mathfrak{A}$ which we are considering. This is what is done, for example, for margin bounds (section \ref{margin_bounds}) for max-margin classifiers, like SVMs, where we assume the algorithm will output an $h$ which maximizes margin, and then plug that condition into a non-uniform margin bound (by setting $\mathcal{H}(S)$ to be the set of max margin classifiers for $S$).

Non-uniform bounds are often designed with particular assumptions on the algorithm and dataset in mind. This is because the value of non-uniform bounds depends on both $\mathcal{A}$ and $\mathcal{D}$. This in turn means that the notion of optimality for non-uniform (and algorithm and data-dependent bounds in general) should also depend on $\mathcal{A}$ and $\mathcal{D}$ as argued in \cref{optimality_nonuniform}.

In the applications to deep learning theory we study, non-uniform convergence bounds are only used as a way to obtain algorithm-dependent bounds, though we still present the fundamental non-uniform bounds in section \ref{srm_bounds} as useful background for the algorithm-dependent bounds based on them.

In a recent influential paper, \citet{nagarajan2019uniform} showed that for SGD-trained networks, the tightest\footnote{``Tightest'' in their analysis refers to the fact that their bound is specific for the particular algorithm they study (SGD-trained DNNs), and the particular dataset (which was synthetically constructed)} double-sided bounds based on uniform convergence give loose bounds for certain families of data distributions. In appendix G.4 they extend this result to include all algorithm-dependent bounds for which the capacity is only allowed to depend on the output of the algorithm $\mathcal{A}(S)$ and has no other data-dependence beyond this, while in Appendix J, they show that these limitations also apply to standard deterministic PAC-Bayes bounds based on the general KL-divergence-based PAC-Bayes bound (\cref{pac_bayes_general}). The intuition behind these results is that this kind of bounds can encode little information about the algorithm beyond the hypothesis class, as they can not explicitly capture the dependence of $\mathcal{A}(S)$ on $S$.

Most algorithm-dependent bounds derived from non-uniform bounds that we will study are based on data-independent non-uniform bounds. This fact automatically makes the algorithm-dependent bounds have a capacity that only depends on $\mathcal{A}(S)$, and so they would suffer from the limitations pointed out in \citet{nagarajan2019uniform}. Inspired by this result, we chose to classify algorithm-dependent bounds into those that are based on non-uniform convergence and those that are not.

It is also worth noting that there are several different ways to ``get around'' the limitations in \citet{nagarajan2019uniform}. 1) The analysis of \citet{nagarajan2019uniform} finds a family of distributions where the class of bounds we discuss above fail. However, it doesn't discard the possibility that these bounds may give tight predictions for other data distributions. 2) We can allow an algorithm-dependent bound to depend on the data in other ways than via $h$  (e.g.\ by having explicit dependence on $S$ as in \citet{shawe1998structural,shawe1997pac}) 3) \citet{negrea2019defense} showed that one can still seize uniform convergence for the distributions that \citet{nagarajan2019uniform} study by bounding the risk difference between $\mathcal{A}(S)$ and an $h$ in a class for which uniform convergence gives tight bounds. 4) We can consider non-double-sided bounds.  Bounds derived from an analysis assuming realizability can satisfy 2) and 4). They can satisfy 2) because they only guarantee convergence on a hypothesis class which depends on the data (the set of $h\in\mathcal{H}$ with $0$ error). Bounds considered by \citet{nagarajan2019uniform}, even if they can depend on $h$, are worst-case over training sets. However, an analysis which assumes realizability can bound the generalization gap only for those training sets $S$ where $\hat{\epsilon}(h)=0$. Realizable bounds can satisfy 4) because they can use the one-sided version of the Chernoff bound for $0$ mean \citep{langford2005tutorial}, as can be seen for example in the Corollary 2.3 in \citet{shalev2014understanding}. Note that realisable bounds derived from agnostic bounds (by setting $\hat{\epsilon}(h)=0$) will still suffer from the limitations that \citet{nagarajan2019uniform} point out, because agnostic bounds themselves do not satisfy the conditions above. Therefore, only bounds which take full advantage of the realizable assumption may avoid the limitations in \citet{nagarajan2019uniform}.

We note that the marginal likelihood PAC-Bayes bound we present in \cref{pac_bayes_bound} is based on a realizable analysis using one-sided bounds and thus avoids the limitations of double-sided bounds discussed above.  In fact, our bound holds with high probability over the Bayesian posterior, rather than universally over a whole family of hypothesis-dependent posteriors, as usually considered by deterministic PAC-Bayes bounds \citep{nagarajan2019uniform}.

\subsubsection{Overview of bounds}

\begin{table}[H]
\begin{center}
\begin{tabular}{ |p{0.06\linewidth}|y{0.22\linewidth}!{\vrule width 1pt}y{0.22\linewidth}!{\vrule width 1pt}y{0.22\linewidth}!{\vrule width 1pt}y{0.22\linewidth}| }
\cline{2-5}
\multicolumn{1}{c|}{}&\multicolumn{2}{c!{\vrule width 1pt}}{\Centerstack{Algorithm-independent\\ (section \ref{algorithm-independent-bounds})}} & \multicolumn{2}{c|}{\Centerstack{Algorithm-dependent\\ (section \ref{algorithm-dependent-bounds})}}\\
\cline{2-5}
\multicolumn{1}{c|}{}& \multicolumn{1}{x{0.22\linewidth}!{\vrule width 1pt}}{Based on uniform convergence}& \multicolumn{2}{c!{\vrule width 1pt}}{\Centerstack{Based on non-uniform\\ convergence}} & \multicolumn{1}{x{0.22\linewidth}|}{\Centerstack{Other}}\\
\hline
{\small \raisebox{-0.2\normalbaselineskip}[0pt][0pt]{\rotatebox[origin=c]{90}{\shortstack{Data-\\ independent}}}}& VC dimension bound\textsuperscript{*} (section \ref{vc_dimension}) & SRM-based bounds\textsuperscript{$\dagger$} (section \ref{srm_bounds})& \multicolumn{1}{c!{\vrule width 1pt}}{-} & uniform stability bounds\textsuperscript{$\ddagger$} and compression bounds\textsuperscript{$\mathsection$} (section \ref{stability_bounds})\\
\noalign{{\hrule height 1pt}}
\multicolumn{1}{|c|}{{\small \raisebox{0\normalbaselineskip}[0pt][0pt]{\rotatebox[origin=c]{90}{\shortstack{Data-dependent}}}}}& Rademacher complexity bound\textsuperscript{$\mathparagraph$} (section \ref{rademacher_complexity})& data-dependent SRM-based bounds\textsuperscript{**} (section \ref{srm_bounds}) &{margin bounds\textsuperscript{$\dagger\dagger$} (\ref{margin_bounds}),\newline sensitivity-based bounds\textsuperscript{$\ddagger\ddagger$} (section \ref{sensitivity_bounds}),\newline NTK-based bounds\textsuperscript{$\mathsection\mathsection$} (section \ref{bounds_ntk}), \newline other PAC-Bayes bounds\textsuperscript{$\mathparagraph\mathparagraph$} (section \ref{other_pac_bayes})} & {non-uniform stability bounds\textsuperscript{***} (section \ref{stability_bounds}),\newline marginal-likelihood PAC-Bayes bound\textsuperscript{$\dagger\dagger\dagger$} (\cref{pac_bayes_bound})}\\
\hline
\end{tabular}
\end{center}
\caption{Classification of the main types of generalization bounds treated in this paper. Roughly speaking, the number of assumptions grows going from left to right, and from top to bottom. Note that, as we discussed in \cref{non_uniform_algor_dep}, algorithm dependent bounds based on non-uniform convergence are automatically data-dependent, which is why there is an empty cell.\\
\small\textsuperscript{*}\citet{vapnik1974theory,blumer1989learnability,bartlett2017nearly}\\
\small\textsuperscript{$\dagger$}\citet{vapnik1995nature,mcallester1998some}\\
\small\textsuperscript{$\ddagger$}\citet{bousquet2002stability,pmlr-v48-hardt16,mou2018generalization}\\
\small\textsuperscript{$\mathsection$}\citet{littlestone1986relating,brutzkus2018sgd}\\
\small\textsuperscript{$\mathparagraph$}\citet{bartlett2002rademacher}\\
\small\textsuperscript{**}\citet{shawe1998structural,shawe1997pac}\\
\small\textsuperscript{$\dagger\dagger$}\citet{bartlett1997valid, bartlett1998sample,bartlett2017spectrally,neyshabur2018a,golowich2017size,neyshabur2018towards,barron2019complexity}\\
\small\textsuperscript{$\ddagger\ddagger$}\citet{neyshabur2017exploring,dziugaite2017computing,arora2018stronger,banerjee2020randomized}\\
\small\textsuperscript{$\mathsection\mathsection$}\citet{arora2019fine,cao2019generalization}\\
\small\textsuperscript{$\mathparagraph\mathparagraph$}\citet{zhou2018non,dziugaite2018data}\\
\small\textsuperscript{***}\citet{kuzborskij2017data}\\
\small\textsuperscript{$\dagger\dagger\dagger$}\citet{valle2018deep}
}
\label{tab:multicol}
\end{table}

In the next sections of this paper, we will describe the major families of generalization error bounds that have been applied to DNNs.  While we don't claim that the list is exhaustive, we tried to cover all the major approaches to generalization bounds. 

In \cref{tab:multicol} we present a high-level overview of where different general classes of bounds found in the literature fit within the classification introduced above. It also lists which bounds we treat explicitly in the rest of the paper, and where they sit in our taxonomy. Thus the table helps illustrate what kinds of general assumptions go into the different bounds.  

Given this hierarchy of the main types of bounds, we next turn to a comparison of their performance.   As expected, the overall empirical performance of the bounds improves as more assumptions are added.


\section{Comparing existing bounds against desiderata}\label{comparing_bounds} 
In this section we use the taxonomy from \cref{generalization_theory} (illustrated in table~\cref{tab:multicol})  to organise a discussion on how different bounds fare against the desiderata proposed in \cref{desiderata}.
 We use a \xmark~ when there is strong evidence that bounds in this family fail to satisfy most important aspects of the desiderata, \cmark~ when there is strong evidence that bounds in the family satisfy most important aspects of the desiderata, and \omark~ otherwise. We are aware these are not formally defined notions, and the marks should just be taken as an aid for the reader.
 
\subsection{Algorithm-independent bounds}\label{algorithm-independent-bounds}

\subsubsection{Data-independent uniform convergence bounds: VC dimension}\label{vc_dimension}

One of the iconic results in the theory of generalization is the notion of uniform convergence introduced by Vapnik and Chervonenkis \citep{vapnik1974theory} which, expressed in the language of PAC learning \citep{blumer1989learnability}, considers \emph{data-independent} uniform convergence bounds, where the capacity $f_\mathcal{H}(S)$ doesn't depend on $S$, but only on the hypothesis class $\mathcal{H}$.  The main result of this theory is that the optimal bound of this form (up to a multiplicative fixed constant) for the generalization gap, in the case of binary classification is
\begin{equation}
\label{vc-agnostic-bound}
   \forall \mathcal{D},~\mathbf{P}_{S\sim \mathcal{D}^m}\left[\sup_{h\in\mathcal{H}}|\epsilon(h) - \hat{\epsilon}(h)| \leq C\sqrt{\frac{\text{VC}(\mathcal{H}) + \ln{\frac{1}{\delta}}}{m}}\right] \geq 1-\delta
\end{equation}
for some constant $C$, and where $\text{VC}(\mathcal{H})$ is a combinatorial quantity called the Vapnik-Chervonenkis dimension \citep{shalev2014understanding}, which depends on the hypothesis class $\mathcal{H}$ alone. In the realizable case, they also proved that the optimal realizable data-independent uniform bound is
\begin{equation}
\label{vc-bound}
   \forall \mathcal{D},~\mathbf{P}_{S\sim \mathcal{D}^m}\left[\sup_{h\in\mathcal{H}_0(S)}\epsilon(h) \leq C\frac{\text{VC}(\mathcal{H}) + \ln{\frac{1}{\delta}}}{m}\right] \geq 1-\delta
\end{equation}
for some constant $C$, and where $\mathcal{H}_0(S)$ is the set of all $h \in \mathcal{H}$ with zero training error on $S$. The particular realizability assumption here is that $\mathcal{D}$ should be such that for all $S$, $\mathcal{H}_0(S)$ is non-empty.

How does this bound do at the desiderata?
\begin{itemize}
    \item \textbf{D.1} \xmark~The bound is data-independent by construction, and therefore its value is the same for any data distribution or training set.
    \item \textbf{D.2} \xmark~The bound decreases with training set size, but a rate $O(1/m)$ which is independent of the dataset, unlike what is observed in practice ($O(m^{-\alpha})$ for a range of $m$, for $\alpha$ often significantly smaller than $1$). Recently, \citet{bousquet2020theory} pointed out that a fundamental reason why data-independent uniform convergence bounds don't capture the behaviour of learning curves is because the worst-case distribution $\mathcal{D}$ can depend on $m$, so that the VC bound bounds an `envelope' of the actual learning curves for individual $\mathcal{D}$, and this envelope may have a markedly different form with respect to $m$ than the individual learning curves.
    \item \textbf{D.3} \xmark~The VC dimension can capture differences in architectures. However, it doesn't appear to capture the correct trends. For example, VC dimension grows with the number of parameters \citep{baum1989size,bartlett2017nearly}, while for neural networks the generalization error tends to decrease (or at least not increase) with increased overparametrization \citep{neyshabur2018the}.
    \item \textbf{D.4} \xmark~The bound is only dependent on the algorithm via the hypothesis class. Therefore it won't capture any algorithm-dependent behaviour except for regularization techniques that restrict the hypothesis class.
    \item \textbf{D.5} \xmark~The VC dimension of neural networks used in modern deep learning is typically much larger than the number of training examples \citep{zhang2016understanding}, thus leading to vacuous VC-dimension bounds.
    \item \textbf{D.6} \cmark~Although computing the exact VC dimension of neural networks is intractable, there are good approximations and bounds which are easily computable \citep{bartlett2017nearly}.
    \item \textbf{D.7} \cmark~The VC dimension offers a rigorous bound with minimal assumptions. Therefore, its guarantees are rigorously applicable to many cases.
\end{itemize}

A common way to interpret the VC dimension bound is in terms of bias-variance tradeoff~\citep{neal2018modern,neal2019bias}, which is a simple  heuristic that is widely used in machine learning.  Bias-variance tradeoff captures the intuition that  there is a tradeoff between a model being too simple, when it cannot properly represent the data (large bias), and a model being too complex (large capacity) when it will tend to overfit, leading to large variance on unseen data. For the bound \cref{vc-agnostic-bound} we can identify $\hat{\epsilon}(h)$ as measuring the bias, and the term involving the VC dimension as indicative of the variance. Intuitively, increasing the VC dimension can make $\hat{\epsilon}(h)$ smaller by increasing the capacity, at the expense of higher variance (the second term), so that one may expect to see a U-shaped curve of generalization error versus model complexity.

Many empirical works have shown that, once in the overparameterized regime,  DNNs in fact typically show a monontonic decrease in generalization error as overparametrization increases, unlike what the VC dimension bound  would suggest \citep{lawrence1998size,neyshabur2018the,belkin2019reconciling}. As the VC dimension bound is optimal among data-independent algorithm-independent bounds, these results tell us that this class of bounds is fundamentally unable to explain the generalization of overparametrized neural networks.

Intuitively, it is not surprising that the VC dimension bound cannot capture this behaviour. As overparametrization increases, the model is able to express more functions, and so the worst-case generalization in the hypothesis class can only get worse. What the VC dimension measure (or for that matter naive applications of the bias-variance tradeoff heuristic in terms of overparameterization) is not capturing  is: 1) the strong inductive bias within the hypothesis class which DNNs have; 2) that the effective expressivity of the model can depend on the data. That is, we need to look for bounds that are algorithm-dependent and/or data-dependent. In the following section, we will look at algorithm-independent data-dependent bounds, and we will see that data-dependence alone is not enough, so that a bound which takes the inductive bias into account is necessary to explain generalization of overparamtrized DNNs.

\subsubsection{Data-dependent uniform convergence bounds: Rademacher complexity}\label{rademacher_complexity}

As a first step to include data dependence, we consider a classic data-dependent uniform convergence bound of the form   \cref{nonuniform-convergence} for algorithm independent uniform bounds, where $D$ is the absolute value function, and the capacity $f(S,h)$ is independent of $h$. It is given by:
\begin{equation}
\label{rademacher-bound}
    \mathbf{P}_{S\sim \mathcal{D}^m}\left[\sup_{h\in\mathcal{H}}|\epsilon(h) - \hat{\epsilon}(h)| \leq 2\mathcal{R}(L\circ \mathcal{H} \circ S) + 4c\sqrt{\frac{2\ln{(4/\delta)}}{m}}\right] \geq 1-\delta
\end{equation}
where $c$ is a constant, and $\mathcal{R}(L\circ \mathcal{H} \circ S)$ is the Rademacher complexity of the set of vectors $L\circ \mathcal{H} \circ S = \{(L(x_i,h(x_i)))_{i=1}^m:h \in \mathcal{H}\} \subseteq \mathbb{R}^m$ where $x_i$ are the input points in $S$ \citep{bartlett2002rademacher,shalev2014understanding}. There is also a lower bound which matches it up to a constant and up to an additive term of $O(1/\sqrt{m})$ \citep{bartlett2002rademacher,koltchinskii2011oracle}. Therefore whether this bound is the optimal data-dependent uniform generalization gap bound up to a constant, depends on the rate of the Rademacher complexity with $m$. Bounds on the Rademacher complexity tend to be of $O(1/\sqrt{m})$ but they often dominate the second term, so that the lower and the upper bound match in their first order behaviour,  suggesting that the bound in \cref{rademacher-bound} may often be close to optimal within this class of bounds.

Rademacher complexity bounds for neural networks typically rely on a bound on the norm of the weights \citep{bartlett2002rademacher}, which typically grows with overparametrization. The lesson we learn is that although Rademacher bounds are data-dependent, they are still worst-case over algorithms with hypothesis class $\mathcal{H}$. Given that DNNs can express functions that generalize badly \citep{zhang2016understanding}, it is thus not surprising that Rademacher complexity bounds are vacuous, and suffer from similar problems as VC dimension bounds.

We note that Rademacher complexity has also been used as a tool in analyses which take into account more properties of the algorithm, for example for bounds based on non-uniform convergence.  In this paper we will treat these other bounds separately, mainly in section \ref{margin_bounds}, when looking at margin bounds.  We will nevertheless briefly  comment on these margin bounds in the desiderata below, as they have been studied more thoroughly, and may give insights into Rademacher complexity more generally.

The performance of the  Rademacher complexity-based bounds on our desiderata can be found below: 
\begin{itemize}
    \item \textbf{D.1} \xmark~The bound is data-dependent, and could capture some dependence in the dataset. However, it only depends on the distribution over inputs, and therefore it can't depend on the complexity of the target function, when the input distribution is fixed. This is unlike real neural networks which generalize worse when the labels are corrupted \citep{zhang2016understanding}.
    \item \textbf{D.2} \xmark~Bounds on the Rademacher complexity are typically $O(1/\sqrt{m})$, thus not capturing the behaviour of learning curves. Furthermore, margin-based bounds (which are based on norm-based bounds to the Rademacher complexity), often \emph{increase} with $m$. See section \ref{margin_bounds}.
    \item \textbf{D.3} \xmark~As Rademacher complexity captures a notion of expressivity, similarly to VC dimension, it often grows with overparametrization and number of layers (see for example, the norm-based bounds on section \ref{margin_bounds}). It seems unlikely it could capture other architectural differences.
    \item \textbf{D.4} \xmark~Like the VC dimension bound, it is only dependent on the algorithm via the hypothesis class. The most studied hyperparameter to compare hypothesis classes in this context is the norms of the weights of the DNN. As we comment in section \ref{margin_bounds}, these bounds appear to anti-correlate with the true error when changing several common optimization hyperparameters.
    \item \textbf{D.5} \xmark~Like the VC dimension bound, these bounds are typically vacuous in the overparametrized regime. In fact, it has recently been shown that there are data distributions where the \emph{tightest} double-sided distribution-dependent uniform convergence bounds for several SGD-trained models are provably vacuous \citep{nagarajan2019uniform}. This implies that the distribution-dependent data-independent version of the Rademacher complexity bound \citep{shalev2014understanding}, and thus that the \cref{rademacher-bound} is also vacuous for that data distribution. Although for other data distributions the bounds may not be vacuous, this work suggest that uniform convergence bounds have some fundamental limitations (see \cref{non_uniform_algor_dep} for further discussion on this issue).
    \item \textbf{D.6} \cmark~Same as for VC dimension bounds.
    \item \textbf{D.7} \cmark~Same as for VC dimension bounds.
\end{itemize}

Although some fundamental limitations of VC dimension are overcome by Rademacher bounds, the bounds that currently exist for neural networks based on Rademacher complexity have very similar problems to those based on VC dimension. Although their capacity measure is data-dependent, for the problems in which DNNs are applied, these bounds still give vacuous predictions and grow with the number of parameters. This observation, combined with the fact that the bounds may be optimal among uniform generalization gap bounds, suggests that to overcome the limitations highlighted here we will need to consider  data-dependent non-uniform and algorithmic-dependent bounds.

Note that, although one can prove that the VC dimension (and perhaps for Rademacher complexity) bound is optimal within its class of bounds, for data-dependent bounds this won't be in general possible. As we discuss in \cref{optimality_nonuniform} whether one bound is tighter than another depends on what prior assumptions are made on the data distribution, so that a unique notion of optimality may not exist.








\subsection{Algorithm-dependent bounds}\label{algorithm-dependent-bounds}

In this section we consider  the major classes of algorithm-dependent bounds and how they fare at satisfying the desiderata. We focus on realizable bounds (which assume zero training error), because modern deep learning often works in this regime \citep{zhang2016understanding}. For many of the recent bounds for deep learning, the lack of experiments means that we can't conclusively answer whether they satisfy several of the desiderata. We hope that future work can fill these gaps. 

\subsubsection{Algorithm-dependent bounds based on non-uniform convergence}

We start by looking at algorithm-dependent bounds derived from non-uniform convergence bounds (see also \cref{non_uniform_algor_dep}). We begin by presenting in the next section the basic types of non-uniform bounds, before seeing the two main applications for deep learning in the next two sections: margin bounds, and sensitivity-based bounds. We will also briefly comment on some other approaches to apply PAC-Bayes to deep learning that have been proposed.

\paragraph{Basic non-uniform convergence bounds and structural risk minimization}\label{srm_bounds}


The simplest and most fundamental idea to make non-uniform bounds is related to a learning technique called structural risk minimization (SRM) developed by Vapnik and Chevonenkis \citep{vapnik1995nature}. The derivation of this bound is very similar  to the classic textbook PAC bound (see e.g.\ corollary 2.3 in \citet{shalev2014understanding}), but rather than using a uniform union bound, it uses a nonuniform union bound over the hypothesis class to  prove that for any countable hypothesis class $\mathcal{H}$ and any distribution over hypotheses $P$ \citep{mcallester1998some,shalev2014understanding}:
\begin{equation}
\label{srm1}
   \forall \mathcal{D},~\mathbf{P}_{S\sim \mathcal{D}^m}\left[\forall h\in\mathcal{H}~\textrm{such that }\hat{\epsilon}(h)=0,~\epsilon(h)  \leq \frac{\ln{\frac{1}{P(h)}}+\ln{\frac{1}{\delta}}}{m}\right] \geq 1-\delta.
\end{equation}
First, as expected, this bound  reduces to the standard uniform finite-class PAC bound \citep{valiant1984theory}, when  $P$ is a uniform distribution $P(h)=1/|\mathcal{H}|$. 
What \cref{srm1} tells us is that if we have a learning algorithm and a problem for which we have prior knowledge that some functions $h$ are more likely to be learned than others, then we can obtain tighter bounds by choosing a $P$ that assigns a higher value to these $h$. 
One way to intuitively understand how this knowledge affects the bound is to consider the limit where $P$ is highly concentrated on a subset of $\mathcal{H}$, and approximately uniform within that subset. Then  \cref{srm1} approaches the finite-class uniform PAC bound for a reduced hypothesis class and the capacity can be interpreted as measuring an ``effective size'' of the hypothesis class.

Another way to think about this is that choosing a $P$ amounts to ``betting'' that some $h$ are more likely to appear. If we are right, then our bound will be better in practice than the standard finite-class PAC bound, while if we are wrong and we get $h$ with low $P$, then this bound will perform worse. 
In other words, unlike the uniform finite-class PAC bound, \cref{srm1} depends on the $h$ we get, so that in order to evaluate its performance,  we need to take into account the data distribution and the algorithm (which together determine the probability distribution over $h$ which the learning algorithm outputs), and work with an expected value of the generalization error.

To define an expected value of the bound, we assume a distribution over data distributions $\mathcal{D}$, which we call the prior $\Pi$. For example, this may be fully supported on a single $\mathcal{D}$ if we know the distribution fully (perhaps we are looking only at the images in CIFAR10). In more real-world settings, we will have uncertainty over what the true distribution is, but we may believe that certain distributions (e.g.\ simpler ones) are more likely. We consider a stochastic learning algorithm which, for training set $S$, outputs hypotheses with a probability $Q(h|S)$ (called the \emph{posterior} (which need not be the Bayesian posterior). Under these two assumptions, it is not hard to see that, for a given distribution $P$ over hypotheses, the expected value of the capacity $\ln{\frac{1}{P(h)}}$ in the bound \cref{srm1} is given by
\begin{equation}
    \label{srm_exp}
    \mathbf{E}_{\mathcal{D}\sim \Pi,S\sim \mathcal{D}^m, h\sim Q(h|S)}\left[-\log{P(h)}\right]    =\mathbf{E}_{h \sim \tilde{Q}}\left[-\log{P(h)}\right] \equiv H(\tilde{Q},P)
\end{equation}
where $H(\cdot,\cdot)$ is the cross-entropy, and $\tilde{Q}(h)=\mathbf{E}_{\mathcal{D}\sim \Pi,S\sim \mathcal{D}^m}[Q(h|S)]$ is the posterior averaged over training sets and data distributions. The second equality follows from the definition of cross-entropy. We can immediately see that the optimal bound of this form, for a given prior $\Pi$ and algorithm with posterior $Q(h|S)$ is obtained by the choice $P=\tilde{Q}$, in which case we obtain $H(\tilde{Q})$ where $H$ is the entropy. This calculation formalizes the intuition that we should choose $P$ to be as close as possible to the probabilities of obtaining different $h$ for the algorithm and task at hand. It also strengthens the intuition that this bound is capturing a notion of effective size, as $H(\tilde{Q})$ is often interpreted as the logarithm of the effective number of elements where $\tilde{Q}$ is concentrated.

Furthermore, we can consider the case where $Q(h|S)$ is given by the Bayesian posterior for prior $\Pi$. In this case, $\tilde{Q}=\tilde{\Pi}$, where $\tilde{\Pi}$ is the prior distribution over hypotheses $h$, obtained by marginalizing over input distributions $\mathcal{D}_X$\footnote{Note that, unless we restricted to the noiseless case, these would be stochastic hypotheses corresponding to $\mathcal{D}_{Y|X}$ as defined in \cref{notation}.}. The average capacity in this case is $H(\tilde{\Pi})$. This is conceptually similar to the No Free Lunch theorem, in that it tells us that for the optimal algorithm, the bound only guarantees good generalization if we make enough assumptions about the data distribution (which corresponds to a low entropy $H(\tilde{\Pi})$). We note, however, that the Bayesian posterior may not give the optimal value of the bound, as can be seen from the fact that $H(\tilde{Q},P)$ is always lowered by making $Q(h|S)$ deterministic (as it would lower $H(\tilde{Q})$), while the Bayesian posterior is not deterministic in general.

The problem with the basic non-uniform SRM bound \cref{srm1} is that it does not capture the idea that some functions may be more similar to others, which quantities such as VC dimension and Rademacher complexity do capture\footnote{For example, the VC dimension of a set of functions which are very similar to each other will typically be lower than a set of very dissimilar functions}. The generalized version of SRM bound which we present below has the advantage of being non-uniform but being able to capture some notion of similarity among function within the subclasses $\mathcal{H}_i$.

A more commonly-used extension of the basic SRM bound considers dividing the (now potentially uncountable) hypothesis class $\mathcal{H}$ into a countable set of (usually nested) subclasses $\mathcal{H}_i$, $i\in \mathbb{N}$, such that $\bigcup_i \mathcal{H}_i = \mathcal{H}$. The result is that for any distribution $P$ over $\mathbb{N}$ \citep{shalev2014understanding}, we have:
\begin{equation}
\label{srm2}
   \forall \mathcal{D},~\mathbf{P}_{S\sim \mathcal{D}^m}\left[\forall i \in \mathbb{N}~\forall h\in\mathcal{H}_i~\textrm{such that }\hat{\epsilon}(h)=0,~\epsilon(h)  \leq \frac{\ln{\frac{1}{P(i)}}+f_i(S)+\ln{\frac{1}{\delta}}}{m}\right] \geq 1-\delta
\end{equation}
where $f_i(S)$ is any (potentially data-dependent) capacity for class $\mathcal{H}_i$ which guarantees uniform convergence within $\mathcal{H}_i$ (for example, a bound on its VC dimension or Rademacher complexity). Results of this form are proven in \citet{shawe1998structural} and \citet{shalev2014understanding}.

We can also compute the expected value of the bound \cref{srm2}, analogously to \cref{srm1}. For the numerator of the bound (ignoring the confidence term), we obtain $H(\tilde{Q}',P)+\mathbf{E}_{i\sim \tilde{Q}'}\left[f_{i}(S)\right]$, where $\tilde{Q}'(i)\equiv\textbf{P}_{h\sim\tilde{Q}}\left[\text{class}(h)=i\right]$ and $\text{class}(h)$ represents the index of the subclass $\mathcal{H}_{\text{class}(h)}$ to which $h$ belongs. Analogously to before, the optimal value of $P$ is given by $\tilde{Q}'$, and the Bayesian posterior will in general not result in the optimal average value of the bound.

One shortcoming of \cref{srm2} is that the decomposition of $\mathcal{H}$ into $\mathcal{H}_i$ has to be defined a priori, that is, it cannot depend on $S$. \citet{shawe1998structural} proposed an extension to the SRM framework which addressed this shortcoming, and defined a potentially infinite hierarchy of subclasses $\mathcal{H}_1(S) \subseteq \mathcal{H}_2(S) \subseteq ...$ which could depend on the data $S$. This framework includes as a special case the margin bounds we will see in section \ref{margin_bounds}.

\citet{shawe1997pac} applied the data-dependent SRM framework to obtain bounds for a parametrized model, where the capacity was related to the volume in parameter space  of a sphere contained within the set of parameters producing zero training error. This work inspired the development of the first PAC-Bayes bounds in \citet{mcallester1998some}\footnote{Although \citet{shawe1997pac} is often cited as a precursor to PAC-Bayes \citep{shawe2019primer}, it offers a distinct analysis (for example, it gives deterministic bounds rather than bounds on expected error), which as far as the authors know hasn't been shown to necessarily give stronger or weaker bounds than PAC-Bayes, and hasn't been applied to neural networks.}. These bounds apply for stochastic learning algorithms, and bound the expected value of the generalization error under the posterior $Q(h|S)$, uniformly over posteriors. The standard form of the general PAC Bayes bound was proven by \citet{maurer2004note} and states, for any distribution $P$ over $\mathcal{H}$,
\begin{equation}
\label{pac_bayes_general}
   \forall \mathcal{D},~\mathbf{P}_{S\sim \mathcal{D}^m}\left[\forall Q~ KL(\mathbf{E}_{h\sim Q}[\epsilon(h)],\mathbf{E}_{h\sim Q}[\hat{\epsilon}(h)])  \leq \frac{KL(Q||P)+\ln{\frac{1}{\delta}}+\ln{(2m)}}{m-1}\right] \geq 1-\delta
\end{equation}
where $KL(Q||P)$ is the KL-divergence between $Q$ and $P$. On the left hand side we use the standard abuse of notation to define $KL(a,b)\equiv a\ln{(a/b)}+(1-a)\ln{((1-a)/(1-b))}$, for $a,b\in[0,1]$.

This bound can be seen to generalize the SRM with data-dependent hierarchies of \citet{shawe1998structural}, where instead of ``hard'' subdivisions of $\mathcal{H}$ into $\mathcal{H}_i$, we consider all possible distributions $Q$ on $\mathcal{H}$. $KL(Q,P)$ is analogous to $\ln{\frac{1}{P(i)}}$ in \cref{srm2} in that it very roughly measures how much of the total probability mass of $P$ is in the high probability region of $Q$. The term which penalizes ``classes'' of $h$ which are too diverse, analogously to $f_i(S)$, is $\mathbf{E}_{h\sim Q}[\hat{\epsilon}(h)$, the average training error, because this is only small if the functions agree on $S$, which will only happen with high probability if they are sufficiently similar. These analogies between the data-dependent SRM bounds and PAC-Bayes are only intuitive, but we conjecture that a more formal connection might be possible.

The PAC-Bayes bound in \cref{pac_bayes_general} is one of the most general non-uniform data-dependent bounds, and its different applications give rise to sensitivity-based bounds (section \ref{sensitivity_bounds}) among others (\cref{other_pac_bayes}). In fact, margin bounds can also be derived from PAC-Bayes \citep{langford2003pac}.


We don't comment on the desiderata for SRM and the PAC-Bayes bounds described above, because they provide a general framework for the non-uniform data-dependent bounds we explore next. We will comment on the desiderata for these bounds individually.


\paragraph{Norm-based and margin bounds}
\label{margin_bounds}


A popular method to obtain data-dependent non-uniform bounds is the method of margin-based bounds. These usually start with \emph{norm-based bounds} which bound the Rademacher complexity of sub-classes of a hypothesis class parametrized by a parameter vector $w \in \mathcal{W}$ (\emph{weights}) corresponding to balls where $w$ has a bounded norm. One can then either express the bound under an assumption on the weight norms, or have the bound depend on the weight norms. The later case is done by applying an SRM-like bound to the family of hypothesis sub-classes corresponding to different weight norms. For \emph{margin bounds}, the bound is applied to a \emph{margin loss}, which upper bounds the 0-1 loss. The result is a bound on the generalization error which depends on a new data-dependent property called \emph{margin}, which measures how confidently the classifier is classifying examples. The bound, in the agnostic case, has the general form, for any margin $\gamma > 0$ \citep{shalev2014understanding},
\begin{equation}
\label{margin-bound}
    \forall \mathcal{D},~\mathbf{P}_{S\sim \mathcal{D}^m}\left[\forall w\in\mathcal{W},~\epsilon(w) \leq \hat{\epsilon}_\gamma(w) + c_1\frac{\lambda(w)}{\gamma\sqrt{m}} + c_2\sqrt{\frac{2\ln{(4/\delta)}}{m}}\right] \geq 1-\delta
\end{equation}
where $c_1$ and $c_2$ are constants (that may depend on the algorithm but not on $h$ or $S$, $\lambda(w)$ is a \emph{capacity measure} which usually measures the norm of the parameters $w$, and the margin error is defined as $\hat{\epsilon}_\gamma(h)=\frac{1}{m}\sum_{i=1}^m \mathbbm{1}[h(x_i)y_i<\gamma]$, for a hypothesis $h$. Note that we usually apply this to real-valued hypotheses where $\mathcal{Y}=\mathbb{R}$, and for which the classification error is defined as $\epsilon(h) = \mathbf{P}_{(x,y)\sim \mathcal{D}}[h(x)y < 0]$, where $y\in\{-1,1\}$ for binary classification. We abuse notation by writing $\epsilon(w)$ and $\hat{\epsilon}_\gamma(w)$ for the quantities evaluated at the hypothesis corresponding to parameter $w$. One can also make a bound that holds (non-uniformly) for all values of $\gamma$ by applying a weighted union bounds over discretized values of $\gamma$ \citep{shalev2014understanding}.

The margin loss $\hat{\epsilon}_\gamma(w)$ measures the amount of miss-classification errors plus some examples which were classified correctly but with low confidence (measured by $h(x)$ being smaller than $\gamma$). In the case of linear models, where $h(x) = w\cdot x$ and $\lambda(w) = |w|$, the ratio $\gamma'=\lambda(w)/\gamma$ measures a \emph{geometric margin}, and the margin loss measures the number of examples that are not on the right side of the classification boundary by a distance greater than $\gamma'$. Support vector machines are a famous example of an algorithm trying to maximize this geometric margin \citep{cortes1995support}.

For neural networks, margin bounds were originally developed based on the analysis of their fat-shattering dimension \citep{bartlett1997valid, bartlett1998sample}. These bounds depend on the $l_1$ norm of the weights, and thus typically grows with overparmetrization. The bounds also grow exponentially with depth. More recently, more complex norms of the weights were studied, as well as using the Lipschitz constant of the network as the capacity measure for margin-bounds \citep{bartlett2017spectrally,neyshabur2018a,golowich2017size,neyshabur2018towards,barron2019complexity}. These bounds have shown some correlation with the complexity of the data, but suffer from their (implicit or explicit) dependence on width and depth, and are vacuous \citep{neyshabur2017exploring,neyshabur2015norm,arora2018stronger}. They also show negative correlation with the generalization error when changing certain training hyperparmeters \citep{jiang2019fantastic}.


We now summarize the strengths and weaknesses of margin-based bounds on our desiderata. Note that most of our conclusions are based on empirical results on existing margin-based bounds, and may not be fundamental to the margin-based approach itself.

\begin{itemize}
    \item \textbf{D.1} \cmark~Margin-based bounds have shown to correlate with the true error when comparing CIFAR versus MNIST, and when comparing uncorrupted versus corrupted data \citep{bartlett2017spectrally,neyshabur2017exploring}. The correlation should be explored over a wider range of datasets and quantified more precisely.
    \item \textbf{D.2} \omark~The dependence on the training set size $m$ of margin-based bounds depend on how the capacity measure changes with $m$. In \cite{nagarajan2019uniform}, it has been shown that several types of the margin-based bounds proposed for deep neural networks actually increase with training set size! However, in \citet{dziugaite2020search}, other norm and margin-based generalization measures are found to correlate well with the error when changing training set size, suggesting that it may be possible to prove better bounds based on these measures.
    
    \item \textbf{D.3} \omark~While most margin-based bounds increase with layer width, some ($l_2$-path norm bound in \citet{neyshabur2017exploring} and the bound in \citet{neyshabur2018towards}) actually decrease with layer width. Regarding variations in depth, all the proposed bounds increase with number of layers, and most of them do so exponentially \citep{neyshabur2015norm,arora2018stronger,barron2019complexity}. However, the empirical results in \citet{jiang2019fantastic,maddox2020rethinking,dziugaite2020search} show that certain measures (e.g.\ path-norm) positively correlate with the error when varying depth. Furthermore, according to  \citet{maddox2020rethinking}, the log path-norm correlates with both width and depth significantly better than path-norm. While no bounds have been derived that scale like these measures, it may be promising to consider work in this direction.

    \item \textbf{D.4} \omark~\cite{jiang2019fantastic} show that margin-based bounds that have been proposed to date appear to often \emph{anti}-correlate with the generalization error when changing common optimization hyperparameters (dropout, learning rate, type of optimizer, etc.). On the other hand, in \citet{dziugaite2020search}, it was demonstrated that some measures such as path norm, often predict certain properties well when changing learning rate (but they didn't vary the other hyperparameters studied in \citet{jiang2019fantastic}).
    
    \item \textbf{D.5} \xmark~As far as the authors are aware, all of the margin-based bounds published to date for DNNs are vacuous \citep{neyshabur2017exploring,neyshabur2015norm,arora2018stronger}.
    \item \textbf{D.6} \cmark~Margin-based bounds are often based on a notion of the norm of the weights that is relatively efficiently computable.
    \item \textbf{D.7} \cmark~Proposed bounds are based on rigorous theorems with typically weak assumptions.
\end{itemize}




\paragraph{Bounds based on the neural tangent kernel}\label{bounds_ntk}

It was recently demonstrated that infinite-width DNNs, when trained by SGD with infinitesimal learning rate, evolve in function space as linear models with a kernel known as the neural tangent kernel (NTK) \citep{jacot2018neural,lee2019wide}. Several works, either inspired by NTK or not, have relied on the linearization of the dynamics for wide neural networks. The bounds are similar to the norm-based bounds in the previous section \ref{margin_bounds} in that they tend to bound the Rademacher complexity of hypothesis subclasses characterized by a norm, usually the norm of the deviation of the weights from initialiation. The analyses based on NTK are often also able to guarantee \emph{convergence} of the optimizer, so that we can also bound the empirical risk for a sufficiently large number of optimizer iterations. For instance, \citet{arora2019fine} proved that for  a sufficiently wide two-layer fully connected neural network and sufficiently many (full batch) gradient descent steps $t$, we have
\begin{equation}
\label{ntk-bound}
    \forall \mathcal{D},~\mathbf{P}_{S\sim \mathcal{D}^m}\left[R(h_t) \leq  \sqrt{\frac{2\mathbf{y}^T(\mathbf{H}^\infty)^{-1}\mathbf{y}}{m}} + O\left(\sqrt{\frac{\log{\frac{m}{\lambda_0 \delta}}}{m}}\right)\right] \geq 1-\delta
\end{equation}
where $h_t$ is the hypothesis learnt by the DNN, $\mathbf{y}$ is the vector of training outputs $y_i$, $\mathbf{H}^\infty$ is the Gram matrix for the inputs in the training set $S$, and $\lambda_0$ is a lower bound on the eigenvalues of $\mathbf{H}^\infty$. See \citet{arora2019fine} (Theorem 5.1) for the full statement of the result. The connection to the norm-based bounds comes from the NTK analysis they carried out, which showed that the Frobenius norm of the deviation of the weights from initializatoin is bounded by $\sqrt{\mathbf{y}^T(\mathbf{H}^\infty)^{-1}\mathbf{y}}$ plus higher order terms. They then used this bound to obtain a data-dependent bound on the Rademacher complexity. NTK analyses could provide tighter bounds on Rademacher complexity than those we saw in section \ref{margin_bounds}, because the NTK likely captures the relation between the parameters and the function the network implements more precisely than the analyses based on bounding the Lipschitz constant or similar quantities.

More recently, \citet{cao2019generalization} sharpened \citet{arora2019fine}'s result with a similar bound, but which applies to networks with any depth trained with SGD. In their bound, $\mathbf{H}^\infty$ is replaced by the NTK matrix, which gives a tighter bound.

\begin{itemize}
    \item \textbf{D.1} \cmark~The NTK-based capacity in \citet{arora2019fine} was shown to increase with amount of label corruption on MNIST and CIFAR, and on MNIST for the capacity measure in \citet{cao2019generalization}.
    \item \textbf{D.2} \omark~The authors are not aware of any work studying the dependence of NTK-based bounds on $m$. The $O(m^{-1/2})$ is not necessarily indicative because the numerator in the bound likely has non-trivial dependence on $m$.
    \item \textbf{D.3} \xmark~Like norm-based margin bounds, current NTK-bounds grow with depth \citep{cao2019generalization}. On the other hand, they show very little dependence on network width, for large enough width. This is a property of NTK analyses which matches empirical observations well.
    \item \textbf{D.4} \omark~The authors are not aware of any work studying the dependence of NTK-bounds on the optimization algorithm. Most analyses focus on vanilla versions of SGD with specific hyperparameter choices which help the theoretical analysis. Other optimization algorithms have been shown to have NTK limits (e.g.\ momentum in \citet{lee2019wide}), but we are not aware of generalization bounds for these.
    \item \textbf{D.5} \omark~The bounds in \citet{cao2019generalization} are non-vacuous (at least up to the dominant term), but they are not very tight (with values close to $1$ above $20\%$ label corruption).
    \item \textbf{D.6} \cmark~The NTK of fully connected networks has an analytical form which is efficiently computable \citep{lee2019wide}. However, for more complex architectures, it may be necessary to estimate the NTK limit (when it exists, see \citet{yang2019scaling}) via Monte Carlo \citep{novak2019neural}.
    \item \textbf{D.7} \cmark~Proposed bounds are based on rigorous theorems, though the assumptions on the algorithms and the width are sometimes hard to match in practice (too large width, or too small learning rate).
\end{itemize}

\paragraph{Sensitivity-based bounds}
\label{sensitivity_bounds}

Many generalization error bounds recently developed and applied to deep learning are based on the idea that neural networks whose outputs (or loss function values) are robust to perturbations in the weights may generalize better. This is linked to the observed phenomenon that flatter minima empirically generalize better than sharper minima \citep{hochreiter1997flat,Hinton:1993:KNN:168304.168306,zhang2018energy,keskar2016large}. At an informal level, it has been argued that the reason for this correlation is that flatter minima may correspond to simpler functions \citep{hochreiter1997flat,wu2017towards}. In particular, \citet{hochreiter1997flat} link flatness to generalization via the idea of \emph{minimum description length} (MDL) \citep{rissanen1978modeling}. MDL generalization bounds are formally equivalent to the simple SRM bound in \cref{srm1}, where $-\log{P(h)}$ is often interpreted as the length of the string representing $h$ under some prefix-free code \citep{shalev2014understanding}.

A more sophisticated argument linking flatness to generalization is found in the data-dependent SRM analysis of \citet{shawe1997pac}. As we mentioned in section \ref{srm_bounds}, they proved generalization bounds where the capacity was mainly controlled by the volume of a region (which they took to be a ball) in weight space in which the training error was zero. A larger volume corresponds to a flatter minimum. Note that one difference with previous work is that \citet{shawe1997pac} define flatness in terms of the classification error, rather than the loss function (for which one typically uses the Hessian as a measure of flatness).



Recent theoretical works studying the link between sensitivity, flatness, and generalization focus on PAC-Bayes analyses \citep{neyshabur2017exploring,jiang2019fantastic}. In the PAC-Bayes bound \cref{pac_bayes_general}, the $KL(Q,P)$ term is typically larger when the posterior $Q$ has a larger variance, so that it can ``overlap'' with the prior $P$ more. On the other hand to control the average training error $\mathbf{E}_{h\sim Q}[\hat{\epsilon}(h)]$, we need $Q$ to put most of its weight on regions of low error. If we combine these two considerations, the bound typically predicts best generalization for large regions of weight space with low error (flat minima). However, as shown in \cite{neyshabur2017exploring}, a more careful look at the PAC-Bayesian analysis suggests that flatness alone is not sufficient to control capacity and should be complemented with some other measure such as norm of the weights. In particular, if $Q$ is taken to be a Gaussian around the weights found after training, and $P$ is taken to be a Gaussian around the origin, then the $KL(Q,P)$ term also grows with the norm of the weights. This can be seen in the bound proposed in \cite{neyshabur2017exploring} which states that, for all $\mathcal{D}$, with probability at least $1-\delta$, over $S\sim \mathcal{D}^m$, we have
\begin{align}
\label{sensitivity-bound}
    \mathbb{E}_{\nu \sim N(0, \sigma)^{n}}\left[R\left(f_{w+\nu}\right)\right] - \widehat{R}\left(f_{w}\right) \leq \underbrace{\mathbb{E}_{\nu \sim \mathcal{N}(0, \sigma)^{n}}\left[\widehat{R}\left(f_{w+\nu}\right)\right]-\widehat{R}\left(f_{w}\right)}_{\text {expected sharpness }}+4 \sqrt{\frac{1}{m}(\underbrace{\frac{\|w\|_{2}^{2}}{2 \sigma^{2}}}_{\text {KL }}+\ln \frac{2 m}{\delta})}
\end{align}
where $R$ is the risk, $f_w=A(S)$ and $w$ are the function and weights, respectively, produced by the network after training, and $f_{w+\nu}$ is the function obtained by perturbing the weights of the network, $w$, by the noise $\nu$. $\sigma$ is a hyperparameter that can be chosen to take any value greater than $0$. The first two terms after the inequality are called \emph{expected sharpness} by \cite{neyshabur2017exploring}, and measures how much the loss increases in average by a perturbation of order $\sigma$ to the weights. \citet{neyshabur2017exploring} perform experiments that show that this bound correlates well with the true error when varying data complexity, and number of training examples, but not when varying the amount of overparametrization. By optimizing the posterior of a PAC-Bayes bound conceptually similar to \cref{sensitivity-bound}, \citet{dziugaite2017computing} also obtained bounds on the expected error under weight perturbations. Their results were noteworthy because the bounds were non-vacuous.

The bounds in \cite{neyshabur2017exploring,dziugaite2017computing} are bounding the expected value of the generalization error under perturbation of the weights, rather than the generalization error of the original network. Obtaining bounds on the latter is addressed by recent work on \emph{determinist} PAC-Bayes bounds \citep{nagarajan2018deterministic}. However their bounds are vacuous, and follow the wrong trends when varying depth and width.

In \cite{jiang2019fantastic}, it was found empirically that using a worst-case measure of sharpness, that measures the loss change along the worst weight perturbation of a certain magnitude, very similar to the one proposed in \cite{keskar2016large}, gives the best correlation (and best results in their causal analysis) among the many measures they tested.

\citet{arora2018stronger} developed another approach to prove generalization bounds based on the robustness of neural networks to perturbations of the weights. They showed that if the effects of perturbing the weights does not grow too much as it propagates through every layer\footnote{a condition which they formalized through a series of measurable quantities}, the network could be compressed to a network with fewer parameters, for which a generalization error bound could be given that was tighter than other proposed bounds, although still vacuous. They found that their bound correlates with the true error as it decreased during training. However, this bound has the disadvantage that it applies to the compressed network only, and not to the original network.

Recently, \citet{banerjee2020randomized} have developed novel deterministic bounds based on a de-randomization of a PAC-Bayes bound. Their bound is also based on the flatness of the minimum found after training (measured by Hessian eigenvalues), and also takes into account the distance moved in parameter space. They provided evidence that their bound correlates well with the test error when varying the training set size and label corruption. However, they didn't study the tightness of their bound (which depends on certain smoothness constants of a linearized version of the non-linear DNN).

\begin{itemize}
    \item \textbf{D.1} \cmark~\cite{neyshabur2017exploring,banerjee2020randomized} provided some evidence that their PAC-Bayesian bounds correlates with true error when varying the data complexity. The authors are not aware of similar results for the other measures.
    \item \textbf{D.2} \cmark~\cite{neyshabur2017exploring,banerjee2020randomized,dziugaite2020search} have shown evidence that different sensitivity-based PAC-Bayesian bounds correlate with true error when varying the training set size. However, a more quantitative comparison could be done, over more datasets and different architectures.
    \item \textbf{D.3} \omark~Both the worst-case and expected sharpness measures appear to correlate well with the true error when varying depth. However, only the worst-case sharpness appears to correlate with the error when varying width \citep{neyshabur2017exploring,jiang2019fantastic}. Furthermore, \citet{dziugaite2020search} show that although the average correlation with depth and width is good for some PAC-Bayes and sharpness measures, they are not robust, and all of them fail for a significant number of experiments. The bound in \citet{arora2018stronger} depends on quantities whose dependence on the architecture are hard to predict; however, the bound's explicit dependence on depth suggests that it may grow linearly with depth, unlike the empirical observations. Recently, \citet{maddox2020rethinking} showed that a measure of flatness known as effective dimensionality correlates better with the error than PAC-Bayes measures, when varying width and depth, suggesting that it may be a better measure than PAC-Bayes-based flatness measures to understand generalization.
    
    \item \textbf{D.4} \cmark~The worst-case sharpness appears to correlate well with the true error when varying several algorithm hyperparameters, while other sharpness measures correlate a bit worse \cite{jiang2019fantastic}. In \citet{dziugaite2020search}, it was shown that some flatness measures indeed correlate well with the error over most experiments.
    \item \textbf{D.5} \omark~Although the sharpness bounds in \cite{neyshabur2017exploring,jiang2019fantastic} are likely vacuous, \cite{dziugaite2017computing} showed that by optimizing the PAC-Bayesian posterior over a large family of Gaussians, non-vacuous bounds could be obtained. 
    
    \item \textbf{D.6} \omark~Some sharpness bounds studied in \cite{jiang2019fantastic} are efficiently computable. However, the more advanced ones such as the ones in \cite{dziugaite2017computing} require significant computational expense.
    
    \item \textbf{D.7} \omark~The bounds in \cite{neyshabur2017exploring,dziugaite2017computing,arora2018stronger} are based on rigorous theorems. However, they only apply to either the expected error under random perturbations of the weights, or to the compressed network. The deterministic PAC-Bayes bounds in \cite{nagarajan2018deterministic,banerjee2020randomized} apply to the deterministic error of the original network, but may be vacuous or not very tight. The worst-case sharpness measure which appears to correlate best with the generalization error \cite{jiang2019fantastic} lacks a rigorous theorem that explains this correlation.
\end{itemize}

\subsubsection{Other PAC-Bayes bounds}\label{other_pac_bayes}

There are many recent works applying PAC-Bayesian ideas to obtain generalization error bounds in novel ways. \citet{zhou2018non} proved non-vacuous generalization error bounds on compressed networks trained on large datasets. However, their bounds are still very loose, and their correlation with the true error hasn't been studied yet. \cite{dziugaite2018data} extended the PAC-Bayesian analysis to include data-dependent priors under the assumption that they are close to differentially-private priors. Their bounds are non-vacuous and apply to the expected value of the generalization error after training with SGLD \citep{welling2011bayesian}, but they are very computationally expensive, and have only been tested on a small synthetic dataset.





\subsection{Other algorithm-dependent bounds}\label{other_algo_dep_bounds}

We now consider other types of algorithm-dependent bounds, which are not based on non-uniform convergence. The main class of such bounds are stability-based bounds, under which we include compression and algorithmic stability bounds. These bounds include both data-independent and data-dependent bounds. The data-independent bounds will suffer from many of the same pitfalls as VC dimension bounds, as DNNs show generalization for some datasets but not others, while the data-dependent bounds are more promising.

\subsubsection{Stability-based bounds}
\label{stability_bounds}

Stability-based bounds offer an alternative way to obtain algorithm-dependent bounds, different from the non-uniform convergence SRM-like bounds. In fact, they even allow to obtain data-independent algorithm-dependent bounds. Stability analyses show that if the output of a learning algorithm depends weakly on the training set, then it can be shown to generalize. One approach of this kind was developed by \citet{littlestone1986relating}, who derived \emph{compression bounds}. They obtain data-independent bounds on the generalization error for learning algorithms whose output can be computed via a fixed function of only $k<m$ out of the $m$ training examples (which $k$ examples they are can depend on the training sample). For the realizable case, the bound is
\begin{equation}
\label{compression-bound}
    \mathbf{P}_{S\sim \mathcal{D}^m}[\epsilon(\mathcal{A}(S)) \leq \frac{8k\log{(m/\delta)}}{m}] \geq 1-\delta.
\end{equation}
See \citet{shalev2014understanding} for the formal statement and proof. These bounds are based on the general concept of `stability' described above because if the output of the algorithm only depends on a small subset of the training examples, it means the output is insensitive to changes in most of the outputs. A compression bound has been recently developed for two-layer neural networks trained with SGD on linearly-separable data, based on a proof that in this case SGD converges in a bounded number of non-zero weight updates, which therefore gives a bound on $k$ \citep{brutzkus2018sgd}.

The most common notion of stability, called \emph{algorithmic stability}, was related to generalization by \citet{bousquet2002stability}, and considers how sensitive the output of the learning algorithm is to removing a single example from the training sample. Most work on algorithmic stability has focused on the data-independent notion of \emph{uniform stability}, which has been used to obtain data-independent bounds for SGD \citep{pmlr-v48-hardt16,mou2018generalization}.

Because compression bounds and uniform stability are both data-independent, they can't capture the crucial data-dependence of generalization in deep learning which was pointed out, for example, by \citet{zhang2016understanding}. To this end, some recent extensions have looked at data-dependent notions of stability. \citet{kuzborskij2017data} applied this idea to obtain generalization error bounds for SGD-trained models. They obtain bounds of the form
\begin{equation}
\label{data-dep-stability-bound}
    \forall \mathcal{D},~\mathbf{E}_{(S\sim \mathcal{D}^m,A)}[R(\mathcal{A}(S)) - \hat{R}(\mathcal{A}(S))] \leq \mathcal{O}\left(\sqrt{c(R(w_1)-R^*)}\cdot \frac{\sqrt[4]{T}}{m} + c \sigma \frac{\sqrt{T}}{m}\right)
\end{equation}
where the expectation is also taken over the randomness in the algorithm $\mathcal{A}$ (as SGD is a stochastic algorithm), $c$ is a constant related to the step size of SGD, $w_1$ is the initial weights of the neural network, $R^*$ is the minimum risk achiveable by the hypothesis class of the algorithm, $T$ is the training time, and $\sigma$ is a bound on the variance of the SGD gradients.


The bound in \cref{data-dep-stability-bound} has several limitations which we discuss now. First of all, it is a bound on the expected value of the generalization error, which does not immediately imply a bound that holds both with high probability and logarithmic dependence\footnote{Note that the Markov inequality implies a high probability bound, but it has a polynomial $1/\delta$ dependence on the confidence parameter, which is expected to be far from optimal.} on the confidence parameter $\delta$, as the other bounds studied here \citep{shalev2010learnability,feldman2019high}. The bound applies for smooth convex losses, but the authors also provide a more complex bound for smooth non-convex losses. The smoothness is an important limitation, as most neural networks in practice use ReLU activations, making the loss surface non-smooth. However, in their experiment they use a CNN with max pooling, which gives a non-smooth loss surface, and their bounds still work well empirically. The bound also requires an estimate of $R(w_1)$. They estimate this with a validation set, which makes the bound not dependent on $S$ alone. However, in practice the initialization is usually random and independent of $S$, which means $R(w_1)$ could be estimated from the empirical loss on $S$. The bound can not be directly applied to classification error, which is not Lipschitz, but it can be applied to cross-entropy loss which in turn implies an upper bound on classification error. Perhaps the most serious limitation of this stability bound is that it assumes a single pass over the data $T\leq m$, which is not the usual case in practice as training and generalization error often decrease by training on several passes of the data.

Some recent works extended the data-dependent stability analysis of SGD-trained neural networks in different directions. \cite{london2017pac,li2019generalization} combined stability bounds with the PAC-Bayesian approach we be discussed later. \cite{zhou2019understanding} proved data-dependent stability bounds that apply to SGD with multiple passes over the data. However, their bound increases with training time (although logarithmically rather than polynomially as in \cite{pmlr-v48-hardt16,mou2018generalization}), contradicting the empirical result that generalization error appears to pleateau with training time \citep{hoffer2017train}. However, their results have not yet been empirically tested, so that it is hard to evaluate these bounds on the desiderata.

\begin{itemize}
    \item \textbf{D.1} \cmark~The data-dependent stability bound in \cite{zhou2019understanding} correlated with the true error when varying the amount of label corruption on three different datasets. \cite{li2019generalization} also showed that their Bayes stability analysis gives bound that are larger for randomly labelled CIFAR10 than for uncorrupted CIFAR10.
    \item \textbf{D.2} \omark~The explicit dependence of many stability bounds on $m$ is given by the classical $1/m$ or $\sqrt{1/m}$. However, data dependent bounds have quantities which may change with $m$ in complicated ways. Furthermore, the bounds for non-convex losses in \cite{kuzborskij2017data} have a more complicated data-dependence with explicit power law dependence on $m$. \cite{kuzborskij2017data} show a good correlation between their bound and the empirical generalization gap. However, this is only done for one-pass (online) SGD. To the best of our knowledge, no study has compared the $m$-dependence of the other bounds with the true error, for the usual case where SGD is trained over multiple passes to reach low training error.
    \item \textbf{D.3} \omark~The data-dependent stability bounds depend on empirical quantities of the loss surface and the behavior of SGD, both of which can in principle be affected by the choice of architecture. However, as far as the authors know this dependence hasn't been explored.
    \item \textbf{D.4} \xmark~ Most stability bounds grow with training time. However, empirically the opposite correlation is found, with longer training time leading to better generalization \cite{hoffer2017train,jiang2019fantastic}. \cite{charles2018stability} showed situations in which gradient descent (GD) is not uniformly stable but SGD is. However, whether SGD really generalizes better than GD is still a controversial topic \citep{hoffer2017train}.
    \item \textbf{D.5} \omark~Many stability bounds \cite{pmlr-v48-hardt16,mou2018generalization,zhou2019understanding} grow with training time, and thus will become vacuous for sufficiently large training times. \cite{kuzborskij2017data} have shown remarkably tight expected generalization gap bounds, which however only apply to one-pass SGD. These results are very promising, but further empirical analysis, and work on tight bounds for multi-pass SGD is still needed.
    \item \textbf{D.6} \cmark~Data-independent stability bounds are typically easy to compute. Data-dependent ones like the one in \cite{kuzborskij2017data} are harder (depending on empirical quantities like the Hessian and gradient sizes), but still applicable to reasonably sized problems
    \item \textbf{D.7} \omark~Proposed bounds are based on rigorous theorems. However, these often have assumptions which are not held in common practice (e.g.\ one-pass over the data, smoothness of loss function, linear separability).
\end{itemize}


\subsection{Other bounds and generalization measures}
\label{other_bounds}

\citet{jiang2018predicting} have recently empirically studied a measure of generalization based on features of the distribution of margins at different hidden layers. Their measure shows significantly better correlation between the bound and the error as data complexity and architecture is varied, than the margin-based measures in \citet{bartlett2017spectrally}. However, they provide limited results regarding the predictive power when changing individual features of the data or architecture, and instead provide an aggregate correlation score when they are all changed simultaneously. The success of the measure explored in \citet{jiang2018predicting} may be related to the correlation observed in \citet{valle2018deep,mingard2020sgd} between the prior probability of functions for Bayesian DNNs and the critical sample ratio (CSR), which was proposed as a complexity measure in \citet{arpit2017closer}. The critical sample ratio is an aggregate measure of the distances between input points and the decision boundary, like the distribution of input margins in \citet{jiang2018predicting}. Furthermore the correlation between the prior probability and generalization is established in \citet{valle2018deep,mingard2020sgd}, as well as our PAC-Bayes bound in \cref{pac_bayes_bound}, which helps understand why CSR may correlate with generalization.

\citet{wei2019data,wei2019improved} offered theoretical bounds based on a similar idea to \citet{jiang2018predicting}. They considered extending the notion of margins to all the layers of the DNN. However, no empirical evaluation was presented.

Recently, and in response to the work of \citet{nagarajan2019uniform}, \citet{negrea2019defense} proposed a method to apply uniform convergence to cases where it previously produced vacuous generalization gap bounds. The idea is to show that an algorithm produces a hypothesis with generalizatoin and empirical errors close to some hypothesis in a class with a uniform convergence property. This work has yet to be applied to deep learning, but could offer an interesting direction.

Finally, several measures investigated by \citet{jiang2019fantastic,maddox2020rethinking,dziugaite2020search}, some of which we have discussed in the sections on margin and sensitivity-based bounds, don't yet correspond to rigorous analyses of generalization, but some predict generalization better than existing bounds, and so are promising directions for future more rigorous analyses of generalization.

\subsection{generalization error predictions for specific data distributions}
\label{average_generalization_error}

There is recent work on predicting generalization error of DNNs by making assumptions on $\mathcal{D}$, rather than relying on frequentist PAC bounds which typically don't make any assumption on $\mathcal{D}$ (beyond assuming realizability). \citet{spigler2019asymptotic,bordelon2020spectrum} study kernel regression with miss-specified priors, based on \citet{sollich2002gaussian}\footnote{A similar analysis from a physics-inspired perspective has been presented in \citet{DBLP:journals/corr/abs-1906-05301}}. \citet{bordelon2020spectrum} further apply this idea to the neural tangent kernel (NTK) of a fully connected network, which approximates the behaviour of SGD-trained DNNs in the infinite width and infinitesimal learning rate limit \citep{jacot2018neural}, which seems to work well for finite-width but wide DNNs \citep{lee2019wide}. They apply this to MNIST by estimating the NTK eigenspectrum on a sample of MNIST, and then training the DNN on smaller samples. Their predicted generalization error closely follows the observed error of the SGD-trained DNN. As far as the authors are aware this is one of the most accurate predictions of the generalization error of DNNs based on well-established theory.



One of the limitations of the analysis in \citet{bordelon2020spectrum} is that it relies on knowing what the data distribution is, and in particular the eigenvalues of the NTK kernel, and the eigenspectrum of the target function (with respect to the eigen basis of the NTK kernel). This can be estimated by using a sufficiently large sample of the data, but it is not discussed in \citet{bordelon2020spectrum} how big the sample needs to be for the estimate to be accurate. They use a sample larger than the training set, which therefore makes this predictor fall outside the requirements of the kinds of predictons we have been considering (which only depend on $S$). However, the approach offers an analytical theory of generalization which can help with interpretability and gaining understanding of which properties of a DNN architecture lead to generalization for a particular dataset. The other limitation of the work in \citet{bordelon2020spectrum} is that the analysis only applies for MSE loss, which is not commonly used for classification (though the training with the two losses often results in DNNs with similar learned functions \citep{mingard2020sgd}).


\section{Marginal-likelihood PAC-Bayesian generalization error bound}
\label{pac_bayes_bound}

In the previous section, we saw that algorithm-independent or data-independent bounds are clearly insufficient to explain the generalization performance of DNNs because the hypothesis class of DNNs is too expressive, and the generalization strongly depends on the dataset, respectively. Furthermore, the main approaches for algorithm-dependent bounds are based on non-uniform convergence, which has been shown to have fundamental limitations in its ability to predict generalization in SGD-trained DNNs for some datasets \citep{nagarajan2019uniform}. Although there are ways around this limitation (see the discussion in \cref{non_uniform_algor_dep}), it suggests that looking at other approaches to obtain generalization bounds may be promising. Non-uniform stability bounds offer an interesting alternative to non-uniform convergence, but their empirical success so far is still limited.

Here we present a new deterministic realizable PAC-Bayes bound which applies to a DNN trained using Bayesian inference, with high probability over the posterior. We work in the same set up as \citet{valle2018deep} and \citet{mcallester1998some}. We consider binary classification, and a space of functions or hypotheses with codomain $\{0,1\}$. We consider a ``prior'' $P$ over the hypothesis space $\mathcal{H}$, and an algorithm which samples hypotheses according to the Bayesian posterior, with $0-1$ likelihood. To recall, we define generalization error as the probability of misclassification upon a new sample $\epsilon(h)=\textbf{P}_{x\sim\mathcal{D}}[h(x)\neq t(x)]$, where $t$ is the target function. In \cref{proof:pac_bayes_theorem}, we prove the following theorem:

\begin{theorem} (\textbf{marginal-likelihood PAC-Bayes bound})

\label{pac_bayes_theorem}
For any distribution $P$ on any hypothesis space $\mathcal{H}$ and any realizable distribution $\mathcal{D}$ on a space of instances we have, for $0< \delta \leq 1 $, and $0< \gamma \leq 1 $, that with probability at least $1-\delta$ over the choice of sample $S$ of $m$ instances, that with probability at least $1-\gamma$ over the choice of $h$:

$-\ln{(1-\epsilon(h))} < \frac{\ln{\frac{1}{P(C(S))}} + \ln{m}+ \ln{\frac{1}{\delta}} + \ln{\frac{1}{\gamma}}}{m-1} $

\noindent
where $h$ is chosen according to the posterior distribution $Q(h)=\frac{P(h)}{\sum_{h\in C(S)} P(h)}$, $C(S)$ is the set of hypotheses in $\mathcal{H}$ consistent with the sample $S$, and where $P(C(S))=\sum_{h\in C(S)} P(h)$
\end{theorem}

The proof is presented in \cref{proof:pac_bayes_theorem}. It closely follows that of the original PAC-Bayesian theorem by McAllister, with the main technical step relying on the quantifier reversal lemma of \citet{mcallester1998some}. Note that the bound is essentially the same as that of \cite{langford2001bounds}, except for the fact that it holds in probability and it adds an extra term dependent on the confidence parameter $\gamma$, which is usually negligible, but may be important when considering the effect of optimizer choice. The quantity $P(C(S))$ corresponds to the marginal likelihood, or Bayesian evidence of the data $S$, and we will also denote it by $P(S)$, to simplify notation.


In \cite{valle2018deep}, the authors interpreted $Q(h)$ as approximating the probability by which the stochastic algorithm (e.g.\ SGD) outputs hypothesis $h$ after training. The preceding bound relaxes this assumption, because it shows that in some sense, the bound holds for ``almost all'' of the zero-error region of parameter space. More precisely, it holds with high probability over the posterior. This suggests that SGD may not need to approximate the Bayesian inference as closely, for this bound to be useful. Nevertheless, \citet{mingard2020sgd} gave empirical results showing that, for DNNs, the distribution over functions that SGD samples from, approximates the Bayesian posterior rather closely. A fully rigorous generalization error bound for DNNs would need further analysis of SGD dynamics, but we believe these theoretical and empirical results strongly suggest that the PAC-Bayes bound should be applicable to SGD-trained DNNs.

Because it applies to the Bayesian posterior only, the bound in \cref{pac_bayes_theorem} does not apply universally over a large family of posteriors, like standard deterministic PAC-Bayes bounds do, which can be shown to sometimes give loose bounds \citep{nagarajan2019uniform}. Furthermore, as we will show in \cref{bayesian_optimality}, the bound is in a certain sense asymptotically optimal in the limit of large training set size.

We expect our  bound to give significantly tighter results than previous PAC-Bayes bounds applied to DNNs, because rather than working with parameters, our bound works directly with posteriors and priors in \emph{function space}. Since the parameter-function map \citep{valle2018deep} of DNNs is many-to-one, with a lot of parameter-redundancy, it is not hard to construct situations where $KL(Q_{\text{par}}||P_{\text{par}})$ between a parameter-space posterior $Q_{\text{par}}$ and prior $P_{\text{par}}$ is high, but $KL(Q||P)$ between the induced posterior and prior in function-space is low. In fact, in \cref{proof:kl_div_inequality}, we show that the following inequality holds
\begin{equation}
    KL(Q||P)\leq KL(Q_{\text{par}}||P_{\text{par}})
\end{equation}
which implies that it is always better (or at least not worse) to consider PAC-Bayes bounds in function space for parametrized models, if possible. Furthermore, in \cref{experimental_results}, we will empirically verify that our bound gives good predictions for SGD-trained DNNs, and satisfies most of our desiderata for a generalization error bound.  Thus our empirical results corroborate our expectation of better agreement above. 

\section{Optimality of data-dependent bounds}\label{optimality_nonuniform}

\subsection{General definitions for optimality} 

In \cref{vc_dimension} we saw that VC dimension bounds are provably optimal (up to a constant) among the set of data-independent algorithm-independet uniform bounds. The question of optimality is more difficult for the other types of bounds. As will be shown  below, the optimal data-dependent bound may depend on a chosen prior over data distributions, as well as the algorithm. We won't consider notions of optimality for data-independent non-uniform, or algorithmic-dependent bounds (e.g.\ uniform stability bounds). Instead,  the notion of optimality we treat here is for data-dependent algorithm-dependent bounds.

To explore this notion of optimality, we make the following two definitions, using a simplified notation for the right hand side of the inequality in \cref{general-pac}, which we refer to as the \emph{value} of the bound:
\begin{equation*}
    B(S)=\frac{f_\mathcal{A}(S)}{m^\alpha}.
\end{equation*}

\begin{definition}
A PAC bound of the form in \cref{general-pac} with value $B(S)$ is called \emph{distribution-admissible} or \emph{distribution-Pareto optimal} for algorithm $\mathcal{A}$ if there does not exist another bound for algorithm $\mathcal{A}$ with value $B'$ such that for all $\mathcal{D}$, for all $\delta>0$, and for all $m$, $\mathbf{E}_{S\sim\mathcal{D}^m}[B'(S)] \leq \mathbf{E}_{S\sim\mathcal{D}^m}[B(S)]$, and $\mathbf{E}_{S\sim\mathcal{D}^m}[B'(S)] < \mathbf{E}_{S\sim\mathcal{D}^m}[B(S)]$ for some $\mathcal{D}$,$\delta$,$m$.
\end{definition}

\begin{definition}
A PAC bound of the form in \cref{general-pac} with value $B(S)$ is called \emph{optimal with respect to a prior $\Pi$} over data distributions, for algorithm $\mathcal{A}$, for a training set size $m$ and confidence level $\delta$ if it minimizes the value of $\mathbf{E}_{\mathcal{D}\sim \Pi, S\sim \mathcal{D}^m}[B(S)]$ over valid PAC bounds for algorithm $\mathcal{A}$.
\end{definition}
We also say that the bound is optimal with respect to a prior $\Pi$ over data distributions, for algorithm $\mathcal{A}$, if it is optimal in the above sense for all $m$, and $\delta\in [0,1]$.

Any of these definitions can be extended to require optimality over a family of algorithms rather than a single one. Analogous definitions can also be made for data-dependent uniform bounds like \cref{rademacher-bound}.  For hypothesis classes where the Rademacher complexity dominates the $O(1/\sqrt{m})$ term (in expectation over $S$, for any $\mathcal{D}$) so that the lower bound closely matches the upper bound, then not only is \cref{rademacher-bound} (approximately) distribution-Pareto optimal, but it's the unique distribution-Pareto optimal uniform generalization gap bound, and therefore the optimal uniform generalization gap bound for any prior $\Pi$ (up to lower order terms).

For non-uniform data-dependent bounds, the authors are not aware of any result showing optimality. However, as we saw in section \ref{srm_bounds}, non-uniform bounds are typically based on choosing a ``prior'', and the expected value of the bound for different $\Pi$ depends heavily on the choice of this ``prior''. This suggests that perhaps there isn't a unique notion of optimality (i.e.\ the optimal bound could depend on the assumed true prior $\Pi$), but doesn't prove this negative result because these are only upper bounds. We will see in the next section that for some pairs of algorithms and priors over data distributions, the marginal-likelihood PAC-Bayes bound asymptotically matches (up to a constant) the expected generalization error (which is a lower bound on the expected value of any generalization error upper bound) and is thus asymptotically optimal (up to a constant) according to our definition. This result, combined with the empirical result that the average value of the bound empirically depends on the data distribution, also implies that the optimal PAC bound, for a fixed algorithm, depends on the prior over data distributions.

\subsection{Bayesian optimality of the marginal-likelihood PAC-Bayesian bound}
\label{bayesian_optimality}

In this section, we present a theorem to show that under certain conditions on the algorithm and distribution, the marginal-likelihood PAC-Bayes bound is tight. In particular, we will show that if the generalization error decreases as a power law with training set size $m$, the bound is asymptotically optimal up to a constant.

We again consider binary classification. Let $S_m=\{(x_i,y_i)\}_{i=1}^m$ be a training set of size $m$. Consider sampling one more sample $(x_{m+1},y_{m+1})$ to obtain $S_{m+1}=\{(x_i,y_i)\}_{i=1}^{m+1}$. We think of $S_m$ and $S_{m+1}$ as random variables with distributions determined by $\mathcal{D}$. We also define $P_e(x,y,S):=1-P(y|S;x)$ to be the probability that a Bayesian posterior $P$ conditioned on a training set $S$ and a test point $x$ predicts the incorrect label. We will denote $P_e(m):=P(x,y,S_m)$, where the dependence on $x,y,S_m$ is left implicit. Note that $\mathbf{E}_{(x,y)\sim \mathcal{D}}[P_e(m)] = \epsilon(\mathcal{A}(S_m))$ where $\mathcal{A}$ is a learning algorithm that returns a function sampled from the Bayesian posterior $P$ conditioned on training set $S_m$, and $\epsilon$ is the average generalization error (averaged both over posterior samples). We will let $\epsilon(m) = \epsilon(\mathcal{A}(S_m))$ where we omit the dependence on $\mathcal{A}$ and $S_m$ for brevity.


In the following we denote with angle brackets $\langle \cdot \rangle$ an average over $S_{m+1}$ (which includes average over $S_m$). Note that $\langle P_e(m) \rangle = \langle \epsilon(m) \rangle$, and it's also the same as $\epsilon(m)$ averaged over $S_m$, because $\epsilon(m)$ already includes an average over one test point sampled from $\mathcal{D}$.

\begin{lemma}\label{lemma:logP_epsilon}
Assume $\langle \epsilon(m) \rangle =o(1)$ as $m\to \infty$. Then, assuming that for some $E<1$, for all $S_{m+1}$, $P_e(m) \leq E$ , there exists a constant $C$, such that, as $m\to \infty$,
\begin{equation}
    \langle \log{P(S_m)} \rangle - \langle \log{P(S_{m+1})} \rangle \sim C \langle \epsilon(m) \rangle.
\end{equation}
Furthermore, if $\text{Var}(P_e(m))=o(\langle \epsilon \rangle)$, then $C=1$.
\end{lemma}

The proof is found in \cref{proof:logP_epsilon}.

\begin{theorem}\label{main_optimality_theorem}
 Assume $\langle \epsilon(m) \rangle \sim b m^{-\alpha}$ as $m\to \infty$, with $0< \alpha < 1$. Then, assuming that for some $E<1$, for all $S_{m+1}$, $P_e(m) \leq E$ , there exists a constant $C'$, such that, as $m\to \infty$, 
\begin{equation}
    -\langle \log{P(S_m)} \rangle \sim C' bm^{1-\alpha}
\end{equation}
Furthermore, if $\text{Var}(P_e(m))=o(\langle \epsilon \rangle)$, then $C'=\frac{1}{1-\alpha}$.
\end{theorem}

\begin{proof}
We can write $-\langle \log{P(S_m)} \rangle$ using a telescoping sum

\begin{align}-\langle \log{P(S_m)} \rangle &=-\langle \log{P(S_0)} \rangle+\sum_{m'=0}^{m-1} \left[-\langle \log{P(S_{m'+1})}\rangle - -\langle \log{P(S_m')} \rangle \right]\\
 &\sim C\sum_{m'=0}^{m-1} \langle \epsilon(m')\rangle\\
 &\sim \frac{Cb}{1-\alpha} m^{1-\alpha}
\end{align}

The second step uses \cref{lemma:logP_epsilon}, and the third uses the assumption in the theorem. Furthermore, from \cref{lemma:logP_epsilon}, the second part of the theorem directly follows.
\end{proof}

In \cref{main_optimality_theorem}, $\langle \epsilon(m)\rangle$ is the actual expected error averaged over training sets of size $m$. This is a lower bound for the training-set-averaged value of any upper bound on the expected error, such as the PAC-Bayes bound \citep{mcallester1998some}. It is also a lower bound on the average value of high-probability bounds like \cref{pac_bayes_theorem}. On the other hand, the PAC-Bayes bound has an expected value with a leading order behaviour given by
\begin{equation}
    \mathbf{E}_{S_m\sim \mathcal{D}^m}\left[\text{PB}(m)\right]\sim \frac{-\langle\log{P(S_m)}\rangle}{m}\sim C' m^{-\alpha},
\end{equation}
which up to a constant, matches the asymptotic behaviour of the lower bound given by $\langle \epsilon(m)\rangle$ itself. Therefore, we can conclude that the PAC-Bayes bound \cref{pac_bayes_bound} is asymptotically optimal (up to a constant, which may be computable a priori), for pairs of priors and learning algorithms that satisfy the theorem assumptions.

The main assumption of the theorem is that $\langle \epsilon(m)\rangle$ follows a power law behaviour asymptotically in $m$. As we mentioned in \cref{desiderata}, this has been empirically found to be the case for sufficiently expressive deep learning models. \citet{spigler2019asymptotic} could prove that the learning curve follows a power law for stationary Gaussian processes in a misspecified teacher-student scenario, and input instances distributed on a lattice. They could also compute the power law exponents analytically. More recently, \citet{bordelon2020spectrum} developed a theory of average generalization error learning curves (see \cref{average_generalization_error}) and proposed an explanation for the power law behaviour based on some assumptions on the data distribution.




The second assumption, which is only necessary to get a smaller value of the constant is that $\text{Var}(P_e(m))=o(\langle \epsilon(m) \rangle)$. This seems plausible in many situations. For large $m$, most of the variation in $P_e(m)$ should typically come from the choice of test point $x$ (on which $P_e(m)=P_e(x,y,S_m)$ implicitly depends) rather than the choice of $S_m$. We can consider two extreme cases. If $P_e(x,y,S_m)$ is $1$ for some $x$ and $0$ for all other $x$, then $P_e(x,y,S_m)$ is a Bernoulli variable, and the variance is of the same order as the mean. On the other extreme, if $P_e(x,y,S_m)=\epsilon(m)$ for all $x$, then the variance (coming from the choice of $x$) is $0$. It seems plausible that in many situations, we find something in between. 


In \cref{experimental_results}, we will show extensive empirical evidence that the PAC-Bayes bound follows a power-law behaviour with an exponent closely matching the empirically-measured test error. We sometimes observe deviations, but these are likely coming  from the use of the expectation-propagation (EP) approximation, which empirically appears to introduce systematic errors, as a power law \citep{mingard2020sgd}. Furthermore, we observe a positive correlation between the exponent $\alpha$ and the proportionality constant relating the bound and the error, $C'$, as predicted by \cref{main_optimality_theorem}.



%
%
%
%
%

\newpage 

\section{Experimental results for the marginal-likelihood PAC-Bayes bound}
\label{experimental_results}

In \cref{desiderata} we presented a series of seven desiderata for a generalization theory of deep learning, and in \cref{comparing_bounds}, we used this framework to compare a wide range of different bounds from  the literature.  In this section  we perform extensive experiments that are needed to test the marginal-likelihood PAC-Bayes bound against the desiderata, especially D.1-3 (performance when varying dataset, architecture, and training set size).  This empirical work provides an example of how to test a bound in detail against the desiderata. In \cref{marglik_bound_desiderata}, we will discuss the comparison of our bound against all the desiderata. 

The key quantity needed to compute the bound \cref{pac_bayes_theorem} for DNNs is the marginal likelihood, which measures the capacity term in the bound.  We follow the same approach as \citet{valle2018deep} and calculate the marginal likelihood using the Gaussian process (GP) approximation to DNNs (also called neural network GP, or NNGP), proposed in several recent works \citep{lee2017deep,matthews2018gaussian,novak2018bayesian,garriga2018convnets,yang2019tensor}, and which can be used to approximate Bayesian inference in DNNs. This approximation requires computing the GP kernel of the NNGP for the inputs in the training set $S$, which in the infinite width limit, equals the covariance matrix of the outputs of the DNN with random weights, for inputs in $S$. This can be computed anlaytically for FCNs, but for more complex archiectures this is unfeasible, so we rely on the Monte Carlo approximation used in \citet{novak2018bayesian}. We approximate the kernel by an estimator of the empirical covariance matrix for the vector of outputs of the DNN, computed from $M$ random initializations of the DNN.
The marginal likelihood is then approximated using the expectation-propagation (EP) approximation, as was done in previous works \citep{valle2018deep,mingard2020sgd}. See \cref{app:marginal_likelihood} for more details. We note that for some experiments, running the network to convergence or computing the EP approximation turned out to be too computationally expensive to achieve within our budget, which is why several architectures are not present for some of the experiments for larger training set sizes (e.g.\ those in \cref{error_vs_arch}).  Code for the experiments can be found in \url{https://github.com/guillefix/nn-pacbayes}

In the following experiments, the datasets are binarized, and all DNNs are trained using Adam optimizer to 0 training error. The test error is measured on the full test set for the different datasets used, while the training set is sampled from the full training set provided by these datasets. For full experimental details see \cref{experimental_details}.

\subsection{Error versus label corruption (Desideratum D.1)}
\label{error_vs_label_corruption}

Desideratum \textbf{D.1} requires that the bound correlates with the error as we change the dataset complexity. We test this correlation in two ways: by directly corrupting a fraction of the labels of a standard dataset, increasing its complexity, and by comparing different standard datasets differing in complexity.

\begin{figure}[H]
    \centering
    \begin{subfigure}[t]{0.49\textwidth}
        \includegraphics[width=1.0\textwidth]{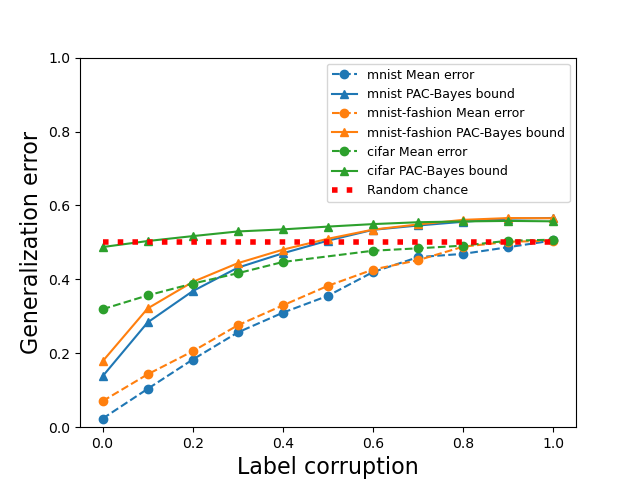}
        \caption{}
    \end{subfigure}
    \begin{subfigure}[t]{0.49\textwidth}
        \includegraphics[width=1.0\textwidth]{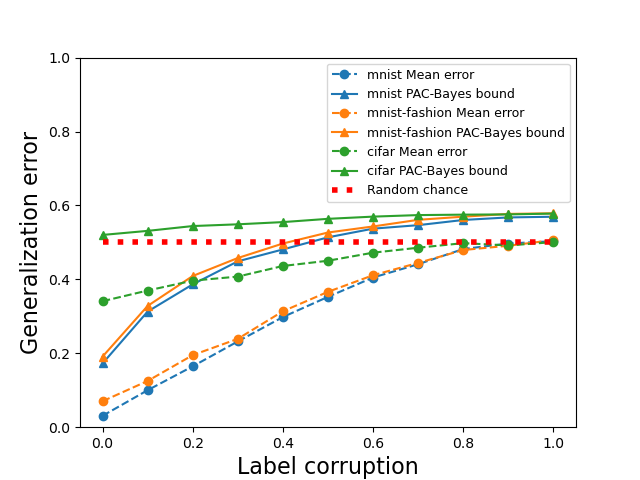}
        \caption{}
    \end{subfigure}
    \caption{\textbf{Generalization error versus amount of label corruption (D1)}. PAC-Bayes bound and generalization error versus amount of label corruption, \textbf{(a)} for a CNN with 4 layers and no pooling and \textbf{(b)} for a FCN with 2 layers, for three datasets and a training set of size $10000$. See \cref{experimental_details} for more details on architectures. The bound follows the expected trends with complexity, both for increasing label corruption, and for the ordering among the datasets.  The dashed red line denotes the 50\% error expected upon random guessing.}
    \label{fig:pac-bayes-compsweep}
\end{figure}

In \cref{fig:pac-bayes-compsweep}, we show the true test error and the PAC-Bayes bound for a CNN and a FCN, as we increase the label corruption for three different datasets (CIFAR10, MNIST, and Fashion MNIST), which have been binarized. We find that the bound is not only relatively tight, but qualitatively follows the behaviour of the true error, increasing with complexity, as well as preserving the order among the three datasets. In \cref{fig:lc_main} and \cref{fig:lc_somenets}, we present more datasets, and observe that for sufficiently large training set size, the bound typically can correctly predict in which datasets the networks will generalize better.

\subsection{Error versus training set size (Desideratum D.2)}
\label{learning_curves}

Desideratum \textbf{D.2} requires that the bound predicts the change in error as we increase the training set size. As mentioned previously, several works \citep{hestness2017deep,novak2019neural,rosenfeld2019constructive,kaplan2020scaling} have found that learning curves for DNNs tend to show a power law behaviour with an exponent that depends on the dataset, but not significantly on the architecture. However, we note that in some work on learning curves for Gaussian processes \citep{sollich2002gaussian,spigler2019asymptotic,bordelon2020spectrum}, more complex types of learning curves have also been observed (sometimes involving several regions of non-monotonicity), suggesting that DNNs may potentially show more complex learning curves, different from power laws, in some regimes. As a simple example, in \citet{hestness2017deep}, it was pointed out that in the low data regime learning curves do not show the asymptotic power law behaviour. Another example is the double-descent behaviour with respect to data size observed in \citet{nakkiran2019deep}. 

\begin{figure}[H]
    \centering
    \includegraphics[width=0.8\textwidth]{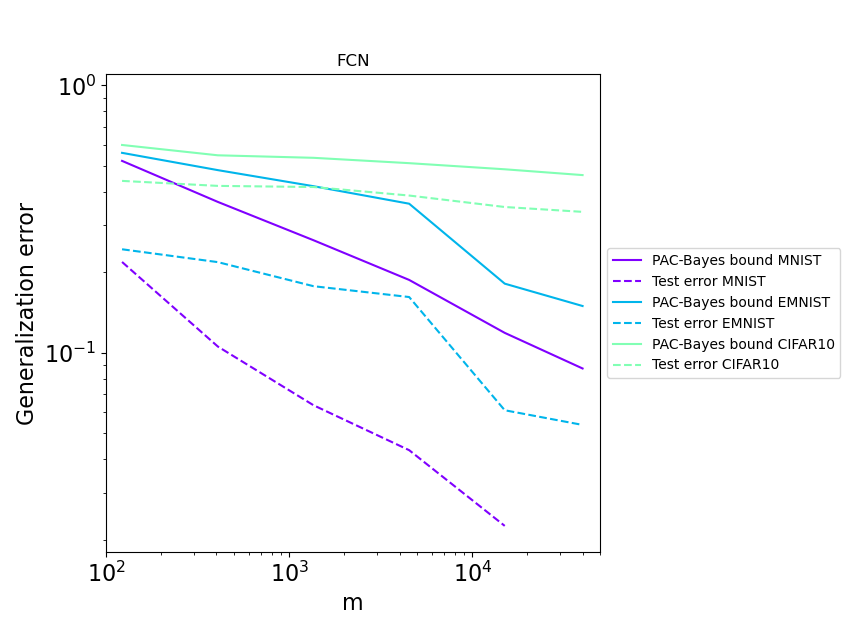}
    \caption{\textbf{Learning curves for fully connected network (batch 256)}. Solid and dashed line show, respectively, the empirical test error and the PAC-Bayes bounds for fully connected network (FCN) with 2 hidden layer, versus training set size $m$, for three different datasets. The DNNs were trained using Adam and batch size 256 to 0 training error.}
    \label{fig:lc_fc_main}
\end{figure}

We empirically computed the learning curves for many combinations of architectures and datasets, as well as the corresponding PAC-Bayes bounds.  In \cref{fig:lc_fc_main} and \cref{fig:lc_main}, we show the learning curves for three representative architectures for different datasets. The exponent of the learning curve clearly depends on the dataset, but less strongly on the architecture. The PAC-Bayes bound approximately matches the power law exponent of the empirical learning curves, as predicted by \cref{main_optimality_theorem}. For \cref{fig:lc_fc_main}, the bound even predicts the quick drop in generalization for EMNIST between $m=4k$ and $m=15k$ for the FCN. However, we found that this fine grained agreement doesn't typically hold for other architectures, for which the bound only matches the overall power law behaviour of the learning curve.

In \cref{fig:lc_somenets}, we show the learning curves for several representative architectures for five different datasets. For each  architecture, the bound and the SGD results  have a very similar learning curve exponent, though there can be a different vertical offset that depends on the architecture.  The relative ordering in generalization performance between different architectures is typically also predicted by the bound, specially for large $m$. We will explore this trend in more detail in \cref{error_vs_arch}.

\begin{figure}[H]
    \centering
    \includegraphics[width=1.0\textwidth]{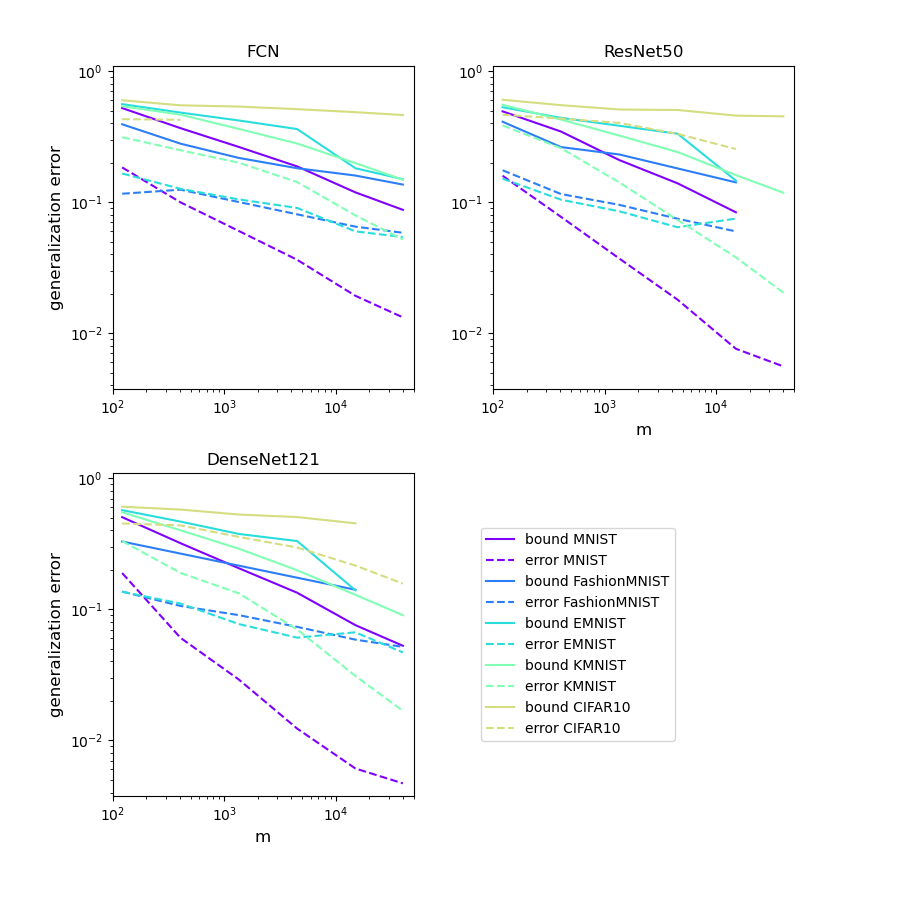}
    \caption{\textbf{Learning curves, compared by dataset, for three architectures}. Solid and dashed lines show, respectively, the empirical test error and the PAC-Bayes bounds. The architectures are a FCN, Resnet50, and Densenet121. The datasets show a range in complexity, from simpler (MNIST) to more complex (CIFAR-10). The DNNs were trained using Adam and batch size 32 to 0 training error.}
    \label{fig:lc_main}
\end{figure}

\begin{figure}[H]
    \centering
    \includegraphics[width=1.0\textwidth]{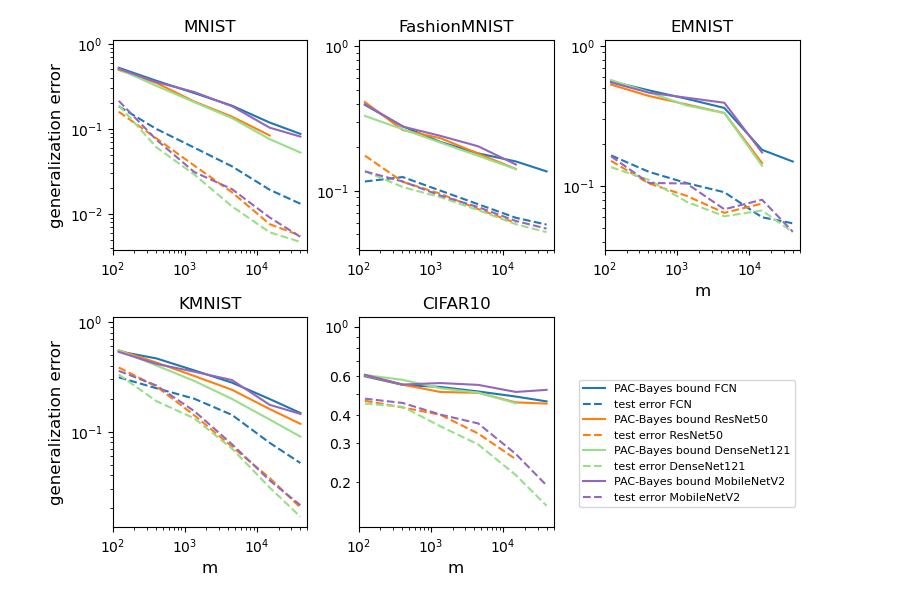}
    \caption{\textbf{Comparing different architectures}. Learning curves for the test error and the PAC-Bayes bounds for representative architectures and different datasets. Solid and dashed line show, respectively, the empirical test error and the PAC-Bayes bounds. The architectures are  FCN, Resnet50, Densenet121, and MobileNetv2. The DNNs were trained using Adam with batch size 32 to 0 training error. Different architectures show similar learning curve power law exponents, which are matched well by the PAC-Bayes bound. Note that we used slightly different y-axis ranges for each dataset, to aid the distinction of different architectures. The ordering of the PAC-Bayes bound also agrees reasonably well with the ordering of the true learning curves, when comparing architectures (Desideratum D.3).  }
    \label{fig:lc_somenets}
\end{figure}

The learning curves we observe in the figs above agree with the previous empirical observations of power law behaviour in learning curves for DNNs, with only a few exceptions, where we observe a deviation from power law behaviour. In particular the learning curve for CIFAR10 for batch 32 appears to deviate from a power law on this range of $m$. However for batch 256 it shows cleaner power law behaviour (see \cref{fig:lc_somenets_256}, \cref{fig:lc_resnets_256}, \cref{fig:lc_densenets_256} in \cref{app:lc_batch256}) that agrees better with the PAC-Bayes  bound exponent.

In \cref{fig:lc_resnets} and \cref{fig:lc_densenets}, shown in \cref{extra_learning_cuves}, we present the learning curves for several variants of ResNets and DenseNets, respectively. Within each family of similar architectures, the learning curve is even more similar. The PAC-Bayes bound matches the behaviour of the true error rather closely for the entire range of architectures and datasets used.  In particular, the power law exponent of the PAC-Bayes bound is close to that of the true learning curves for these 14 different architectures, just as was found in \cref{fig:lc_main} for three representative ones, showing that our generalization error theory is robust and widely applicable.



\begin{figure}[H]
    \centering
    \begin{subfigure}[t]{0.49\textwidth}
    \includegraphics[width=1.0\textwidth]{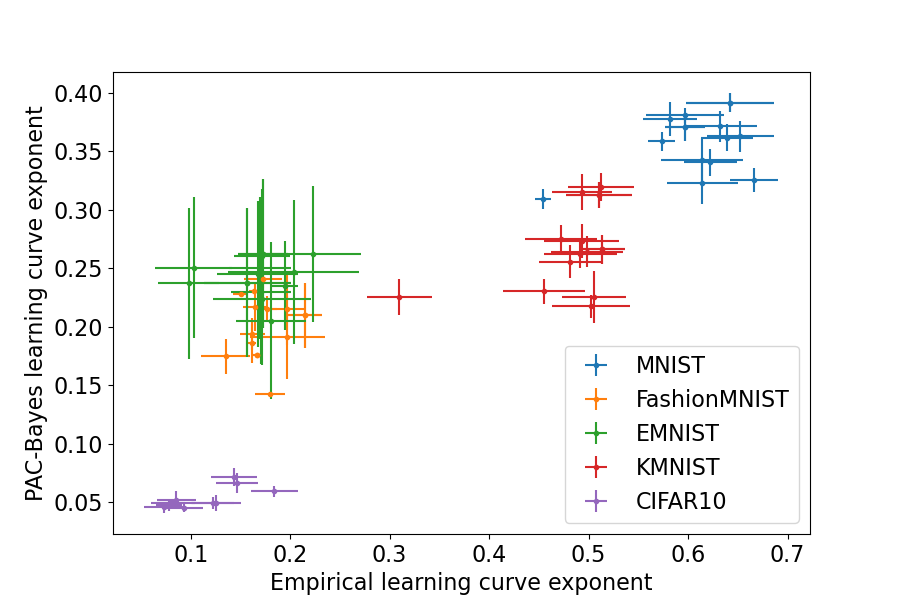}
    \caption{\label{fig:lc_exponents_estimation_exp}}
    \end{subfigure}
        \begin{subfigure}[t]{0.49\textwidth}
    \includegraphics[width=1.0\textwidth]{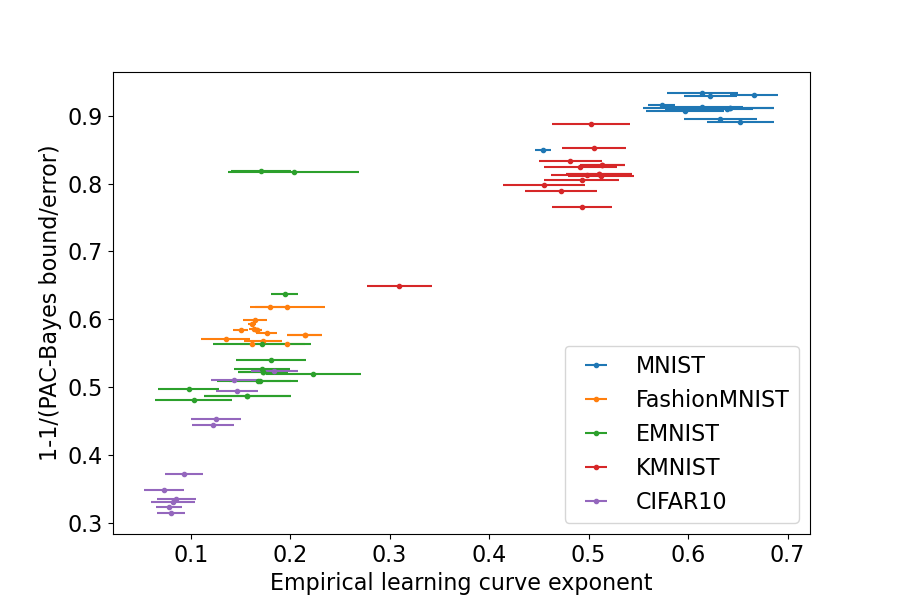}
    \caption{\label{fig:lc_exponents_estimation_ratio}}
    \end{subfigure}
    \caption{\textbf{Learning curve exponents of test error vs learning curve exponent estimated from PAC-Bayes bound}. Each marker shows the learning curve exponent, $\alpha$, measured for all the 19 architectures and 5 datasets in \cref{experimental_details}.  For the empirical error, the exponent was obtained from a linear fit to the learning curve corresponding to $\log{\epsilon}$ vs $\log{m}$ and is shown on the x-axis. For  the the PAC-Bayes bound, in  \textbf{(a)} the exponent is estimated from a linear fit to the log of the PAC-Bayes bound vs $\log{m}$. In \textbf{(b)} the exponent is estimated indirectly from $C'$ which is the ratio of the PAC-Bayes bound and the error (see \cref{main_optimality_theorem}). Then, this second method makes assumptions about how the bounds scale.
    The error bars are estimated standard errors from the linear fits. For the ratio estimate the errors due to fluctuations in dataset are negligible.  Note that the exponents cluster according to dataset. The outliers for MNIST and KMNIST are both the FCN. The DNNs were trained using Adam and batch size 32 to 0 training error.
    For some datasets (e.g.\ CIFAR10) the power law behaviour is less clear, but we still chose to include the estimated exponent for completeness. }
    \label{fig:lc_exponents_estimation}
\end{figure}

In \cref{fig:lc_exponents_estimation}, we compare the power law exponent $\alpha$ from the empirical learning curves (calculated with Adam and batch size 32), to the exponent from the PAC-Bayes bounds. 
For the empirical learning curves,  $\alpha$ is  estimated from a linear fit of $\log{\epsilon}$ vs $\log{m}$.  We note that the exponents  cluster according to the dataset, as observed in previous work \citep{hestness2017deep}. However, we also observe some smaller, but statistically significant variation in exponents within a dataset, indicating that the architecture nevertheless may play a role in the learning curve behaviour (though less significant than the dataset). One exception is the FCN which shows a significantly different exponent than other architectures -- a deviation which is also predicted by the PAC-Bayes bound, but for which we do not yet have an explanation.

To estimate the learning curve exponent for the PAC-Bayes curves, we used two methods. In \cref{fig:lc_exponents_estimation_exp}, we used linear fit to the log of the PAC-Bayes bound vs $\log{m}$, as was done for the empirical exponents. In \cref{fig:lc_exponents_estimation_ratio}, we estimate it as $\alpha=1-1/C'$ where $C'$ is the ratio of the PAC-Bayes bound and the error, which is obtained from the expression in \cref{main_optimality_theorem}. This way of deriving the exponent assumes that the condition $\text{Var}(P_e(m))=o(\langle \epsilon \rangle)$ on the variance of the error holds (although one may still expect $1-1/C'$ and the exponent to correlate even if the condition doesn't hold exactly). For both ways of estimating the error, there is a good correlation between the estimate of the exponent and the empirical exponent. The absolute value of the estimated exponent in \cref{fig:lc_exponents_estimation_exp} does show deviation from the true value, which is probably due to systematic errors in the EP approximation used to compute the marginal likelihood (this was discussed and empirically investigated in \citet{mingard2020sgd}).  It should be kept in mind that power law exponents can be sensitive to the protocol used to measure them \citep{clauset2009power,stumpf2012critical}, so the exact values we find in \cref{fig:lc_exponents_estimation} may not be as meaningful as the correlation between the empirical values and those from the bound.

\subsection{Error versus architecture (Desideratum D.3)}
\label{error_vs_arch}

Desideratum \textbf{D.3} requires that the bound correlates with the error when changing the architecture. We explore this in two ways, by varying certain common architecture hyperparameters (pooling type and depth), and by comparing several state-of-the-art (SOTA) architectures to each other. In \cref{fig:poolsweep}, we vary the pooling type, and find that the bound correctly predicts that the error is higher for max pooling than avg pooling, and both are lower than no pooling, on this particular dataset. In \cref{fig:layersweep}, we vary the number of hidden layers of a CNN trained on MNIST, and find that the bound closely tracks the change in generalization error with numbers of layers.

\begin{figure}[H]
\centering
\begin{subfigure}[t]{0.49\textwidth}
    \includegraphics[width=\textwidth]{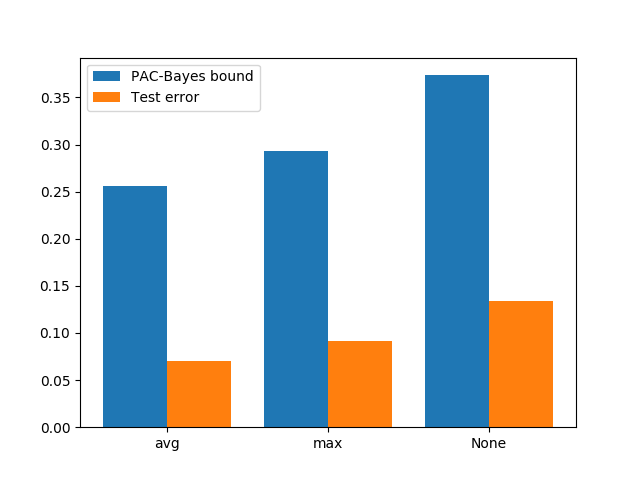}
    \caption{\label{fig:poolsweep} 
    }
\end{subfigure}
\begin{subfigure}[t]{0.49\textwidth}
    \includegraphics[width=\textwidth]{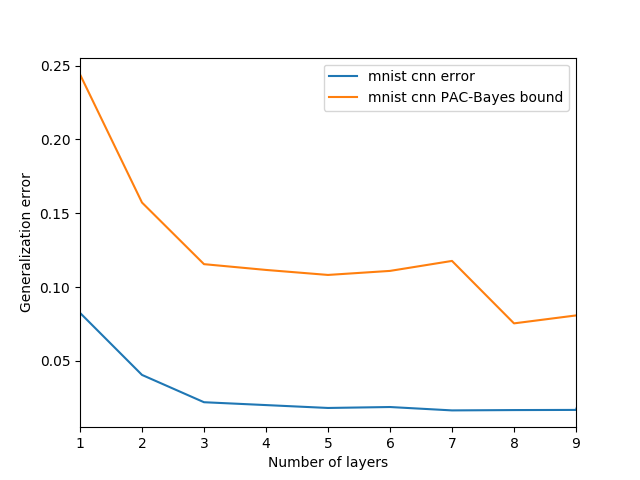}
    \caption{\label{fig:layersweep}
    }
\end{subfigure}
\caption{\label{fig:pool-layer-sweep} \textbf{PAC-Bayes bound and generalization error versus different architecture hyperparameters}. (a) Error versus pooling type, for a CNN trained on a sample of 1k images from KMNIST. (b) Error versus number of layers for a CNN trained on a sample of size 10k from MNIST. Training set error is $0$ in all experiments. We used SGD with batch 32 for both of these experiments.
}
\end{figure}

To explore more complex changes to the architecture, we plot in \cref{error_bound_32_all_datasets} the bound and error against each other for five datasets, for a set of state-of-the-art architectures, including several resnets and densenet variants (see \cref{app:training_details} for architecture details), at a fixed training set size of 15K. The results display a clear correlation, showing that our PAC-Bayes bound can help explain  why some architectures generalize better than others.

\begin{figure}[H]
    \centering
\begin{subfigure}[t]{0.49\textwidth}
    \includegraphics[width=1.0\textwidth]{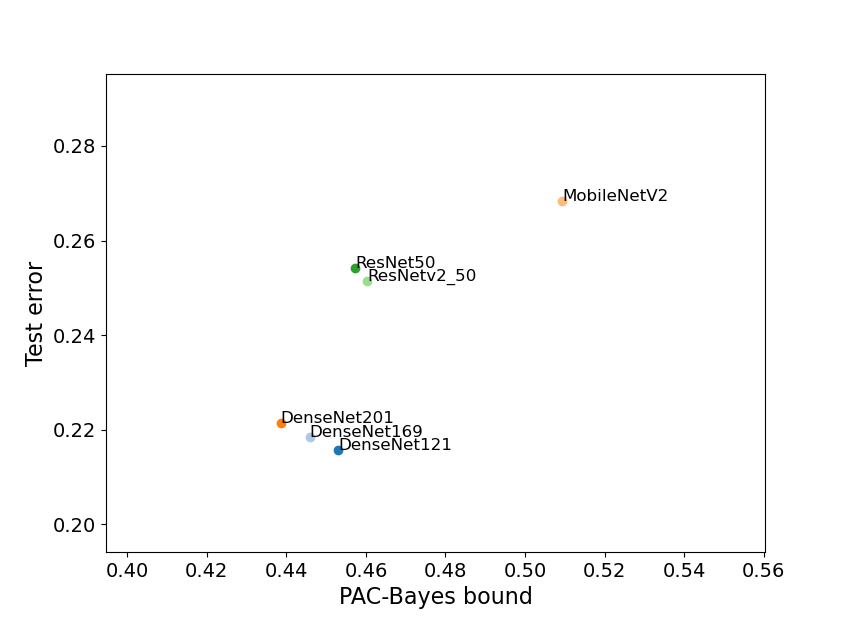}
    \caption{CIFAR10 \label{error_bound_32_cifar}}
\end{subfigure}
\foreach \dataset in {EMNIST,KMNIST,MNIST} {
\begin{subfigure}[t]{0.49\textwidth}
    \includegraphics[width=1.0\textwidth]{figures/bound_vs_error/batch32/without_cnn_fc/bound_vs_error_\dataset_15026_32.png}
    \caption{\dataset \label{error_bound_32_\dataset}}
\end{subfigure}
}
\begin{subfigure}[t]{0.49\textwidth}
    \includegraphics[width=1.0\textwidth]{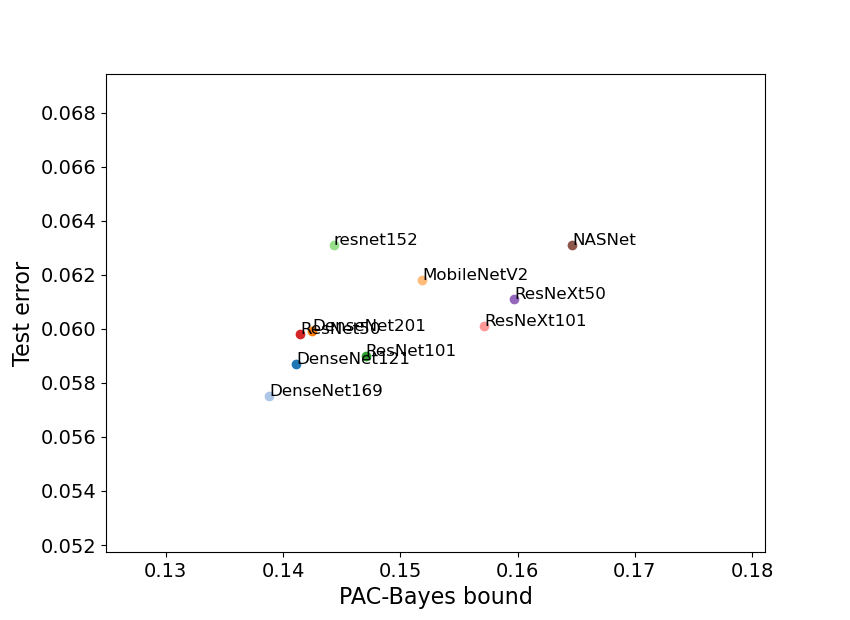}
    \caption{Fashion-MNIST \label{error_bound_32_fashion-mnist}}
\end{subfigure}
    \caption{{\bf PAC-Bayes bound versus test error for different models trained on a sample~of size 15k for the 5 datasets we study.}  The empirical test error was calculated using Adam with batch size 32. FCN and CNNs are removed for clarity as they often have the relatively extreme values of test error and/or bound, which would make the finer grained differences on these plots harder to see. In \cref{app:error_bound_with_cnn_fc}, we show these  plots with FCN and CNN included.}
    \label{error_bound_32_all_datasets}
\end{figure}

Nevertheless,  the empirical differences between architectures are rather small.  For that reason, we can't disentangle whether the deviations from the bound predictions are due to deviations from the bound assumptions (for example SGD not behaving as a Bayesian sampler), or from the different approximations used in computing the bound (for example, the EP approximation used in computing the marginal likelihood, see \cref{app:marginal_likelihood}).

\section{Evaluating the marginal-likelihood bound against the seven desiderata}
\label{marglik_bound_desiderata}

We now evaluate the marginal-likelihood PAC-Bayes bound against the seven desiderata, in the same manner as we did for the other families of bounds that we studied.

\begin{itemize}
    \item \textbf{D.1} \cmark~In \cref{error_vs_label_corruption}, and \cref{learning_curves}, we saw that the bound correctly predicts the relative performance between different datasets for all the architectures we tried, at least for sufficiently large training sets.  Thus the bound captures trends with data complexity.
    
    \item \textbf{D.2} \cmark~In \cref{learning_curves} we showed that the bound correctly predicts the overall behaviour of the learning curve for all the architectures and datasets we studied. In particular, it captures the relative ordering of the power law exponents for different datasets/architecture pairs (\cref{fig:lc_exponents_estimation}).  Thus the bound captures trends with training set size.
    
    \item \textbf{D.3} \cmark~In \cref{error_vs_arch} we showed that the bound shows a good correlation with the generalization when varying the architecture, for all the datasets. We still don't know whether any deviations are fundamental to our approach or due to the various approximations used in computing the bound. We do know that the NNGP approximation can't capture the dependence of generalization on layer width. However, \cref{pac_bayes_theorem} may still  capture these effects if the marginal likelihood could be computed for finite-width DNNs, perhaps using finite-width corrections to NNGPs~\citep{antognini2019finite,yaida2019non}.  Overall, the bound does a good job in tracking the effect of architecture changes.

    \item \textbf{D.4} \xmark~Our approach currently cannot capture effects on generalization caused by different DNN optimization algorithms, because the marginal likelihood, as we calculate it, assumes a Bayesian posterior and does not depend on the fine details of the SGD based optimization algorithms normally employed in DNNs.  In this context it is useful to consider the distinction made in  \citep{mingard2020sgd} between questions of type 1) that ask why DNNs generalize at all in the overparameterised regime, and  questions of type 2), that ask how to further fine-tune generalisation, when the algorithm already performs well, by for example changing optimiser hyperparameters such as batch size or learning rate. Effects of type 2) that are sensitive to the optimiser are hard to capture with the current version of our bound.  We are effectively relying on the observation  that, to first order, different SGD variants seem to perform similarly, and all appear  to approximate Bayesian inference \citep{mingard2020sgd}. Capturing the second order effects, e.g.\ deviations from approximate Bayesian inference due to optimiser hyperparameter tuning, would require an extension to our approach. We note that one of the main hypotheses used to explain differences in generalization among different optimization algorithms is that some algorithms are more biased towards flat solutions in parameter space than others \citep{hochreiter1997flat,keskar2016large,jastrzebski2018finding,wu2017towards,zhang2018energy,wei2019noise}. Therefore, one possibility could be to combine our PAC-Bayes approach (based on probabilities of functions), with the more standard approaches based on flatness, to capture this effect.
     
    \item \textbf{D.5} \cmark~The results in \cref{experimental_results} clearly show that the bound is non-vacuous. In fact, the logarithm in the left hand side of \cref{pac_bayes_theorem} ensures  the bound is less than one\footnote{Note that  PAC-Bayes bounds of the form derived in \cref{pac_bayes_general} are non-vacuous when the KL divergence on the left hand side is directly inverted, rather than using Pinkster's inequality. In the realizable case, the KL divergence reduces to a logarithmic form that can be easily inverted.}.
     More importantly,  we find that the bound can be relatively tight.  
    
    \item \textbf{D.6} \omark~Our bound lies on the higher end of computation cost among bounds we compare here. The most expensive steps are computing the kernel via Monte Carlo sampling and computing the marginal likelihood. The former has a complexity proportional to $O(m^2)$ times the cost of running the model, but the constant can be reduced by making the last layer wider, and it can be heavily parallelized. On the other hand, computing the marginal likelihood using the NNGP approach has a complexity of $O(m^3)$ because it requires inverting a matrix, and may only be improved by making assumptions on the kernel matrix (like being low rank). Therefore, the computational complexity of the bound scales similarly to inference in Gaussian processes. This means we could compute up to $m\approx 10^5$, but larger training set sizes would be difficult. In concurrent work, \citet{park2020towards} showed that for small enough training set sizes computations based on NNGP are competitive relative to training the corresponding DNN.
    
    \item \textbf{D.7} \omark~\cref{pac_bayes_theorem} is fully rigorous for DNNs trained with exact Bayesian inference. In \cref{pac_bayes_bound} we argue that the theorem is probably applicable to DNNs trained with SGD (and several of its variants), based on empirical evidence and arguments from \citet{valle2018deep} and \citet{mingard2020sgd}, as well as the new evidence in the current paper. Furthermore, unlike the bound used in \citet{valle2018deep}, \cref{pac_bayes_theorem} applies with high probability over the posterior (with only logarithmic dependence on the confidence parameter). This property could aid in the analysis of algorithms that sample parameter space with a distribution sufficiently similar to the Bayesian posterior. However, further work is needed to further justify the application of the bound to other optimizers. In addition, the NNGP and EP approximations which we used to evaluate the bound  may have  introduced errors for which we don't have rigorous guarantees yet.
\end{itemize}


\section{Discussion}
\label{discussion}

In this paper we provide a general framework for comparing generalization bounds for deep learning, which complements two other recent large-scale studies \citet{jiang2019fantastic,dziugaite2020search}.   Our framework has two parts. Firstly,  we introduce seven desiderata in \cref{desiderata} against which generalization theories (including the  bounds we focus on here)  should be compared.  Secondly, we classify, in  \cref{generalization_theory},  different bound types by the assumptions (if any) they make on the data and the algorithm. This classification is summarised in~\cref{tab:multicol}.

Next, in \cref{comparing_bounds}, we apply this systematic framework to analyse the performance of bounds across all the different classes.  One conclusion of this analysis is  that generalization theories for DNNs must take into account both the algorithm and the data.  On the one hand, this is not a surprise.  It is known that DNNs are highly expressive.  As  discussed in \cref{comparing_bounds}, classical worst-case analyses such as VC dimension (which ignore properties of the algorithm other than the hypothesis class) miss out on the fact that DNNs have a strong inductive bias towards certain hypotheses (for example simple ones~\citep{valle2018deep}), and therefore will pick out poor predictions that a DNN could make in principle, but does not do in practice.   More generally, this implies that bias-variance style arguments which link the capacity term in the bounds with the number of parameters in a DNN are barking up the wrong tree.  Another way of expressing these theme comes from  recent work, \citet{wilson2020bayesian}, where it is argued  that one should not conflate flexibility (corresponding to the support of the prior, or the ``size'' of the hypothesis class) with model complexity. As an example, they point out that many classes of Gaussian processes can express any function (satisfying some weak regularity conditions), yet are able to generalize. Similarly,  \citet{valle2018deep} also argued that  strong inductive bias encoded in the parameter-function map of overparametrized DNNs is a key property needed for them to generalize.  

Another corollary of the  strong inductive bias in DNNs is that generalization performance should be poor on data that deviates from this inductive bias. Indeed, since the DNNs are biased towards simple functions, they generalize badly, for example, for  complex data, see e.g.~\cite{zhang2016understanding,valle2018deep}.  Therefore the interaction of the algorithm with the data must be taken into account in an explanatory generalization theory for DNNs.  



Many different approaches to bounds can be found in the literature.  While some show promising results,   they all (as far as we can tell due to a lack of data) fall short on some key desiderata.  Nevertheless, by analysing the bounds in terms of our general framework, we can present some suggestions for ways to improve these bounds.  

Inspired by this analysis, we introduce and evaluate a marginal likelihood PAC-Bayes bound in \cref{pac_bayes_bound}, which extends the bound in \citet{valle2018deep} by bounding the actual error with high probability. We show in \cref{optimality_nonuniform} that under a few reasonable assumptions (in particular that learning curves show power law scaling behaviour), it is optimal for large training sets. We extensive test  the bound on image data sets of varying complexity and a wide range of state-of-the-art architectures, finding an overall performance that, to our knowledge, is significantly better than other bounds currently available. 

Perhaps one of the most interesting findings is that  our PAC-Bayes bound is able to capture the power law scaling of the learning curve of DNNs. This scaling has been the subject of recent important work in deep learning theory \citep{hestness2017deep,kaplan2020scaling}. However, the work of \citet{kaplan2020scaling} also highlights the importance of predicting the error as a function of data size, model size (number of parameters) \textit{and} training iterations, for the purposes of optimizing the loss for a given compute budget. Our theory works in the limit of infinitely large models, trained to convergence ($0$ training error) so it can only capture the scaling when data is the bottleneck. Extending our theory to finite-sized models and non-zero training error is an important direction for future work.

On the other hand, in a recent follow-up to their earlier study, \citet{henighan2020scaling}, argue that the learning curve for the case when very large models are trained to convergence (so that performance is bottlenecked by amount of data) poses a lower bound to the optimal-compute learning curve. They therefore suggest that beyond a certain compute budget, the data-limited learning curve determines the long term evolution of the system. If this hypothesis is true, it would mean that the kind of learning curves which we predict in this paper are the ones that asymptotically determine the fundamental limits of learning with DNN, even when optimal use of the compute budget is taken into account.

\subsection{Why the marginal-likelihood PAC-Bayes bound works so well}

The excellent performance of the  marginal-likelihood PAC-Bayes bound  when compared to other bounds in the literature begs the question of why it works so well.    We now discuss several reasons that could explain this success, and which also suggest directions where further progress in developing new generalization error bounds could be made.

Firstly, as discussed above, successful theories of generalization for DNNs must take both the algorithm and the data into account. 
Algorithm-dependent bounds
(\cref{algorithm-dependent-bounds}), including our PAC-Bayes bound, are designed to take the inductive bias of the learning algorithm into account, and 
this is a requisite feature for bounds that satisfy the desiderata.

The marginal likelihood or Bayesian evidence, which controls the behaviour of the the bound in \cref{pac_bayes_bound}, captures the interplay between the inductive bias (encoded in the prior) and the data distribution. The connection between marginal likelihood and generalization has a long history in the Bayesian literature, and can be traced back to the work of MacKay and Neal \citep{mackay1992practical,neal1994priors,mackay2003information}, who empirically found that  they were correlated and provided heuristic explanations based on  Occam's razor. There are also links to the early work on PAC-Bayes \citep{shawe1997pac,mcallester1998some} who proved generalization error bounds based on marginal likelihood or similar quantities.  Recently, a formal equivalence between marginal likelihood and cross-validation has been shown \citep{fong2020marginal}, strengthening the connection between marginal likelihood and generalization. \citet{wilson2020bayesian} also argued that the Bayesian perspective, afforded by quantities like the Bayesian evidence, may be a promising way to understand generalization in deep learning. Finally, \citet{smith2018bayesian,germain2016pac} have also drawn connections between Bayesian evidence and PAC-Bayes-related concepts, including the flatness of the minimum found by the optimizer.


A key difference with most bounds in the literature is that our marginal-likelihood PAC-Bayes bound is formulated in terms of a prior and posterior in function space.   
Several recent works on Bayesian deep learning, NNGPs, and the neural tangent kernel (NTK) \citep{lee2017deep,jacot2018neural,yang2019fine,wilson2020bayesian} have also suggested that a function space perspective on neural networks may be fruitful to understand the behaviour of DNNs, and in particular their generalization, as it is the function encoded by the network which uniquely determines the behaviour of the trained DNN. 

Most PAC-Bayes bounds in the deep learning theory literature use distributions in parameter space, mainly as a way to measure flatness (see for example section \ref{sensitivity_bounds}). As we argue in \cref{pac_bayes_bound}, using results from \cref{proof:kl_div_inequality}, the function space KL divergence should always provide a bound that is at least as tight as the KL divergence of the corresponding distributions in parameter space, but we may expect that the function space bound may sometimes be much tighter. In particular, if we interpreted our function-space bound in parameter space, we would need to find the KL divergence between the prior and a posterior with support on \emph{the whole region of parameter space with zero training error}. This region is much larger than the individual minima around which the posterior in PAC-Bayes bounds is usually concentrated -- for example, exact symmetries when swapping hidden neurons imply combinatorially many ``copies'' of almost all solutions, and recent results on ``mode-connectivity'' show that there are many neighbouring solutions which are not accessible by a Gaussian centered around a minimum, because of the topology of the level sets of the loss surface \citep{garipov2018loss}. We expect a distribution supported on a much larger region of parameter space should typically have a much smaller KL divergence with a diffuse prior such as the isotropic Gaussian which we consider, and this could be one of the main reasons why our bound gives much tighter predictions than other PAC-Bayes bounds.

In light of these arguments, we suggest studying bounds (not just PAC-Bayes bounds) which rely on the function-space picture. These could rely on the recent developments for NNGPs and the NTK. Combining function-space with parameter-space analyses (for example to take flatness into account) could also be an interesting direction for future work.


\citet{mingard2020sgd} have also shown the utility of the function space perspective to study the behaviour of stochastic gradient descent (SGD). They demonstrate that SGD, when considered as a sampling process involving initialization followed by stochastic optimization, samples functions with probabilities very close to Bayesian inference. This empirical finding explains why our PAC-Bayes bound is able to predict the learning curves of SGD-trained DNNs, even though the proof only rigorously applies to Bayesian DNNs.  Conversely, the good overall performance of our bounds also strengthens the hypothesis put forward in \citet{mingard2020sgd} that, to first order at least, SGD samples functions with probabilities close to Bayesian inference.  More broadly, these findings suggests that studying Bayesian DNNs is a promising direction to understand DNNs as trained in practice.


Our PAC-Bayes bound 
bounds the actual error, with high probability over the Bayesian posterior. This differs from other deterministic PAC-Bayes approaches which bound the average error uniformly over a large class of posteriors \citep{nagarajan2019uniform}. This new result suggests that the quantifier reversal lemma of \citet{mcallester1998some} can give qualitatively different kinds of high-probability bounds than those obtained from the more commonly used PAC-Bayesian model-averaging bounds developed later \citep{mcallester1999pac}, and we believe that it could offer interesting insights into obtaining high-probability bounds not based on uniform (or non-uniform) convergence.

\citet{nagarajan2019uniform} also showed a large class of  doubled-sided bounds is fundamentally unable to predict the generalization for SGD-trained DNNs, at least for certain synthetic datasets. On the other hand, our bound assumes realizability which automatically allows for an analysis using one-sided bounds, rather than double-sided bounds -- i.e.\ bounds in the generalization error rather than bounds on the absolute value of the generalization gap (see e.g.\ the realizable PAC bound for a simple example \citep{shalev2014understanding}). This fact, and the good results of our bound, suggest that realizability may be an important property of successful bounds for deep learning.

\subsection{Future work}

In discussing the desiderata for the marginal-likelihood PAC-Bayes bound, we pointed out several aspects on which the bound could be improved. In terms of the theory, we are still lacking a more rigorous understanding of why our bound for Bayesian DNNs applies so well to SGD-trained DNNs. There is increasing evidence that Bayesian DNNs approximate SGD-trained DNNs, even at the level of individual function probabilities \citep{mingard2020sgd}, but a theoretical understanding of this beyond simple arrival of the frequent arguments such as those found in~\citet{schaper2014arrival} is still lacking. Another important area where more work is needed is in finite-width corrections. There has been recent work on finite-width corrections to the NTK theory \citep{huang2019dynamics,bai2019beyond,dyer2019asymptotics}. Finite-width analysis have also been done for NNGPs \citep{antognini2019finite,yaida2019non}, and understanding how these analyses could be used to calculate the marginal likelihood in our theory is an interesting direction for future work.

On the applied side,  methods to estimate the marginal likelihood that are more accurate and more computationally efficient would broaden the potential applications of our theory, for example to data agumentation or neural architecture search (NAS).

In recent work, \citet{lee2020finite} show that NNGPs are able to capture some of the improvements in generalization due to data augmentation. In light of that work and our theory, a promising direction of future work could be to study whether  marginal likelihood approximations could be used to automatically search for promising data augmentation techniques.

\citet{park2020towards} show that computing NNGP predictions is actually more computationally efficient than training the DNN for small datasets, which they use to improve on certain NAS techniques. \citet{buratti2019ordalia} showed that NAS can also be improved by estimating learning curve exponents, which they do by training a DNN on a few small datasets of different sizes and fitting a power law to the learning curve. As we have shown in \cref{learning_curves}, our PAC-Bayes bound allows for accurate prediction of the relative sizes of learning curve exponents, estimated from a single training set. This suggests that our methods may be applicable to the development of more efficient and accurate NAS techniques.

Finally, in this paper we focused on binary classification problems, in part because these are the easiest to understand.   It will be important in the future to also consider generalization theories for multi-class classification (see also ~\cref{binary_vs_multiclass} for a brief discussion of this question) as well as regression problems and  other types of datasets.

\medskip

\small

\bibliography{bib}

\begin{thebibliography}{137}
\providecommand{\natexlab}[1]{#1}
\providecommand{\url}[1]{\texttt{#1}}
\expandafter\ifx\csname urlstyle\endcsname\relax
  \providecommand{\doi}[1]{doi: #1}\else
  \providecommand{\doi}{doi: \begingroup \urlstyle{rm}\Url}\fi

\bibitem[Antognini(2019)]{antognini2019finite}
Joseph~M Antognini.
\newblock Finite size corrections for neural network gaussian processes.
\newblock \emph{arXiv preprint arXiv:1908.10030}, 2019.

\bibitem[Arora et~al.(2018)Arora, Ge, Neyshabur, and Zhang]{arora2018stronger}
Sanjeev Arora, Rong Ge, Behnam Neyshabur, and Yi~Zhang.
\newblock Stronger generalization bounds for deep nets via a compression
  approach.
\newblock In \emph{Proceedings of the 35th International Conference on Machine
  Learning}, volume~80 of \emph{Proceedings of Machine Learning Research},
  pages 254--263. PMLR, 10--15 Jul 2018.
\newblock URL \url{http://proceedings.mlr.press/v80/arora18b.html}.

\bibitem[Arora et~al.(2019)Arora, Du, Hu, Li, and Wang]{arora2019fine}
Sanjeev Arora, Simon~S Du, Wei Hu, Zhiyuan Li, and Ruosong Wang.
\newblock Fine-grained analysis of optimization and generalization for
  overparameterized two-layer neural networks.
\newblock \emph{arXiv preprint arXiv:1901.08584}, 2019.

\bibitem[Bai and Lee(2019)]{bai2019beyond}
Yu~Bai and Jason~D Lee.
\newblock Beyond linearization: On quadratic and higher-order approximation of
  wide neural networks.
\newblock In \emph{International Conference on Learning Representations}, 2019.

\bibitem[Banerjee et~al.(2020)Banerjee, Chen, and Zhou]{banerjee2020randomized}
Arindam Banerjee, Tiancong Chen, and Yingxue Zhou.
\newblock De-randomized pac-bayes margin bounds: Applications to non-convex and
  non-smooth predictors.
\newblock \emph{arXiv preprint arXiv:2002.09956}, 2020.

\bibitem[Barron and Klusowski(2019)]{barron2019complexity}
Andrew~R Barron and Jason~M Klusowski.
\newblock Complexity, statistical risk, and metric entropy of deep nets using
  total path variation.
\newblock \emph{arXiv preprint arXiv:1902.00800}, 2019.

\bibitem[Bartlett(1997)]{bartlett1997valid}
Peter~L Bartlett.
\newblock For valid generalization the size of the weights is more important
  than the size of the network.
\newblock In \emph{Advances in neural information processing systems}, pages
  134--140, 1997.

\bibitem[Bartlett(1998)]{bartlett1998sample}
Peter~L Bartlett.
\newblock The sample complexity of pattern classification with neural networks:
  the size of the weights is more important than the size of the network.
\newblock \emph{IEEE transactions on Information Theory}, 44\penalty0
  (2):\penalty0 525--536, 1998.

\bibitem[Bartlett and Mendelson(2002)]{bartlett2002rademacher}
Peter~L Bartlett and Shahar Mendelson.
\newblock Rademacher and gaussian complexities: Risk bounds and structural
  results.
\newblock \emph{Journal of Machine Learning Research}, 3\penalty0
  (Nov):\penalty0 463--482, 2002.

\bibitem[Bartlett et~al.(2017)Bartlett, Foster, and
  Telgarsky]{bartlett2017spectrally}
Peter~L Bartlett, Dylan~J Foster, and Matus~J Telgarsky.
\newblock Spectrally-normalized margin bounds for neural networks.
\newblock In \emph{Advances in Neural Information Processing Systems}, pages
  6240--6249, 2017.

\bibitem[Baum and Haussler(1989)]{baum1989size}
Eric~B Baum and David Haussler.
\newblock What size net gives valid generalization?
\newblock In \emph{Advances in neural information processing systems}, pages
  81--90, 1989.

\bibitem[Belkin et~al.(2019)Belkin, Hsu, Ma, and Mandal]{belkin2019reconciling}
Mikhail Belkin, Daniel Hsu, Siyuan Ma, and Soumik Mandal.
\newblock Reconciling modern machine-learning practice and the classical
  bias--variance trade-off.
\newblock \emph{Proceedings of the National Academy of Sciences}, 116\penalty0
  (32):\penalty0 15849--15854, 2019.

\bibitem[Blumer et~al.(1989)Blumer, Ehrenfeucht, Haussler, and
  Warmuth]{blumer1989learnability}
Anselm Blumer, Andrzej Ehrenfeucht, David Haussler, and Manfred~K Warmuth.
\newblock Learnability and the vapnik-chervonenkis dimension.
\newblock \emph{Journal of the ACM (JACM)}, 36\penalty0 (4):\penalty0 929--965,
  1989.

\bibitem[Bordelon et~al.(2020)Bordelon, Canatar, and
  Pehlevan]{bordelon2020spectrum}
Blake Bordelon, Abdulkadir Canatar, and Cengiz Pehlevan.
\newblock Spectrum dependent learning curves in kernel regression and wide
  neural networks.
\newblock \emph{arXiv preprint arXiv:2002.02561}, 2020.

\bibitem[Bousquet and Elisseeff(2002)]{bousquet2002stability}
Olivier Bousquet and Andr{\'e} Elisseeff.
\newblock Stability and generalization.
\newblock \emph{Journal of machine learning research}, 2\penalty0
  (Mar):\penalty0 499--526, 2002.

\bibitem[Bousquet et~al.(2020)Bousquet, Hanneke, Moran, van Handel, and
  Yehudayoff]{bousquet2020theory}
Olivier Bousquet, Steve Hanneke, Shay Moran, Ramon van Handel, and Amir
  Yehudayoff.
\newblock A theory of universal learning, 2020.

\bibitem[Brutzkus et~al.(2018)Brutzkus, Globerson, Malach, and
  Shalev-Shwartz]{brutzkus2018sgd}
Alon Brutzkus, Amir Globerson, Eran Malach, and Shai Shalev-Shwartz.
\newblock {SGD} learns over-parameterized networks that provably generalize on
  linearly separable data.
\newblock In \emph{International Conference on Learning Representations}, 2018.
\newblock URL \url{https://openreview.net/forum?id=rJ33wwxRb}.

\bibitem[Buratti and Upfal(2019)]{buratti2019ordalia}
Benedetto~J Buratti and Eli Upfal.
\newblock Ordalia: Deep learning hyperparameter search via generalization error
  bounds extrapolation.
\newblock In \emph{2019 IEEE International Conference on Big Data (Big Data)},
  pages 180--187. IEEE, 2019.

\bibitem[Cao and Gu(2019)]{cao2019generalization}
Yuan Cao and Quanquan Gu.
\newblock Generalization bounds of stochastic gradient descent for wide and
  deep neural networks.
\newblock In \emph{Advances in Neural Information Processing Systems}, pages
  10836--10846, 2019.

\bibitem[Charles and Papailiopoulos(2018)]{charles2018stability}
Zachary Charles and Dimitris Papailiopoulos.
\newblock Stability and generalization of learning algorithms that converge to
  global optima.
\newblock In \emph{International Conference on Machine Learning}, pages
  744--753, 2018.

\bibitem[Clanuwat et~al.(2018)Clanuwat, Bober-Irizar, Kitamoto, Lamb, Yamamoto,
  and Ha]{clanuwat2018deep}
Tarin Clanuwat, Mikel Bober-Irizar, Asanobu Kitamoto, Alex Lamb, Kazuaki
  Yamamoto, and David Ha.
\newblock Deep learning for classical japanese literature, 2018.

\bibitem[Clauset et~al.(2009)Clauset, Shalizi, and Newman]{clauset2009power}
Aaron Clauset, Cosma~Rohilla Shalizi, and Mark~EJ Newman.
\newblock Power-law distributions in empirical data.
\newblock \emph{SIAM review}, 51\penalty0 (4):\penalty0 661--703, 2009.

\bibitem[Cohen et~al.(2017)Cohen, Afshar, Tapson, and
  Van~Schaik]{cohen2017emnist}
Gregory Cohen, Saeed Afshar, Jonathan Tapson, and Andre Van~Schaik.
\newblock Emnist: Extending mnist to handwritten letters.
\newblock In \emph{2017 International Joint Conference on Neural Networks
  (IJCNN)}, pages 2921--2926. IEEE, 2017.

\bibitem[Cohen et~al.(2019)Cohen, Malka, and
  Ringel]{DBLP:journals/corr/abs-1906-05301}
Omry Cohen, Or~Malka, and Zohar Ringel.
\newblock Learning curves for deep neural networks: {A} gaussian field theory
  perspective.
\newblock \emph{CoRR}, abs/1906.05301, 2019.
\newblock URL \url{http://arxiv.org/abs/1906.05301}.

\bibitem[Cortes and Vapnik(1995)]{cortes1995support}
Corinna Cortes and Vladimir Vapnik.
\newblock Support-vector networks.
\newblock \emph{Machine learning}, 20\penalty0 (3):\penalty0 273--297, 1995.

\bibitem[Daniely et~al.(2011)Daniely, Sabato, Ben-David, and
  Shalev-Shwartz]{daniely2011multiclass}
Amit Daniely, Sivan Sabato, Shai Ben-David, and Shai Shalev-Shwartz.
\newblock Multiclass learnability and the erm principle.
\newblock In \emph{Proceedings of the 24th Annual Conference on Learning
  Theory}, pages 207--232. JMLR Workshop and Conference Proceedings, 2011.

\bibitem[Dyer and Gur-Ari(2019)]{dyer2019asymptotics}
Ethan Dyer and Guy Gur-Ari.
\newblock Asymptotics of wide networks from feynman diagrams.
\newblock In \emph{International Conference on Learning Representations}, 2019.

\bibitem[Dziugaite and Roy(2017)]{dziugaite2017computing}
Gintare~Karolina Dziugaite and Daniel~M. Roy.
\newblock Computing nonvacuous generalization bounds for deep (stochastic)
  neural networks with many more parameters than training data.
\newblock In \emph{Proceedings of the Thirty-Third Conference on Uncertainty in
  Artificial Intelligence, {UAI} 2017, Sydney, Australia, August 11-15, 2017},
  2017.
\newblock URL \url{http://auai.org/uai2017/proceedings/papers/173.pdf}.

\bibitem[Dziugaite and Roy(2018)]{dziugaite2018data}
Gintare~Karolina Dziugaite and Daniel~M Roy.
\newblock Data-dependent pac-bayes priors via differential privacy.
\newblock In \emph{Advances in Neural Information Processing Systems}, pages
  8430--8441, 2018.

\bibitem[Dziugaite et~al.(2020)Dziugaite, Drouin, Neal, Rajkumar, Caballero,
  Wang, Mitliagkas, and Roy]{dziugaite2020search}
Gintare~Karolina Dziugaite, Alexandre Drouin, Brady Neal, Nitarshan Rajkumar,
  Ethan Caballero, Linbo Wang, Ioannis Mitliagkas, and Daniel~M Roy.
\newblock In search of robust measures of generalization.
\newblock \emph{arXiv preprint arXiv:2010.11924}, 2020.

\bibitem[Feldman and Vondr{\'a}k(2019)]{feldman2019high}
Vitaly Feldman and Jan Vondr{\'a}k.
\newblock High probability generalization bounds for uniformly stable
  algorithms with nearly optimal rate.
\newblock \emph{arXiv preprint arXiv:1902.10710}, 2019.

\bibitem[Fong and Holmes(2020)]{fong2020marginal}
Edwin Fong and CC~Holmes.
\newblock On the marginal likelihood and cross-validation.
\newblock \emph{Biometrika}, 107\penalty0 (2):\penalty0 489--496, 2020.

\bibitem[Gao et~al.(2011)Gao, Buldyrev, Havlin, and Stanley]{gao2011robustness}
Jianxi Gao, Sergey~V Buldyrev, Shlomo Havlin, and H~Eugene Stanley.
\newblock Robustness of a network of networks.
\newblock \emph{Physical Review Letters}, 107\penalty0 (19):\penalty0 195701,
  2011.

\bibitem[Garipov et~al.(2018)Garipov, Izmailov, Podoprikhin, Vetrov, and
  Wilson]{garipov2018loss}
Timur Garipov, Pavel Izmailov, Dmitrii Podoprikhin, Dmitry~P Vetrov, and
  Andrew~G Wilson.
\newblock Loss surfaces, mode connectivity, and fast ensembling of dnns.
\newblock In \emph{Advances in Neural Information Processing Systems}, pages
  8789--8798, 2018.

\bibitem[{Garriga-Alonso} et~al.(2018){Garriga-Alonso}, Aitchison, and
  Rasmussen]{garriga2018convnets}
Adri{\`a} {Garriga-Alonso}, Laurence Aitchison, and Carl~Edward Rasmussen.
\newblock Deep convolutional networks as shallow {G}aussian processes.
\newblock \emph{arXiv preprint arXiv:1808.05587}, aug 2018.
\newblock URL \url{https://arxiv.org/abs/1808.05587}.

\bibitem[Germain et~al.(2016)Germain, Bach, Lacoste, and
  Lacoste-Julien]{germain2016pac}
Pascal Germain, Francis Bach, Alexandre Lacoste, and Simon Lacoste-Julien.
\newblock Pac-bayesian theory meets bayesian inference.
\newblock In \emph{Advances in Neural Information Processing Systems}, pages
  1884--1892, 2016.

\bibitem[Golowich et~al.(2018)Golowich, Rakhlin, and Shamir]{golowich2017size}
Noah Golowich, Alexander Rakhlin, and Ohad Shamir.
\newblock Size-independent sample complexity of neural networks.
\newblock In \emph{Proceedings of the 31st Conference On Learning Theory},
  volume~75 of \emph{Proceedings of Machine Learning Research}, pages 297--299.
  PMLR, 06--09 Jul 2018.
\newblock URL \url{http://proceedings.mlr.press/v75/golowich18a.html}.

\bibitem[{GPy}(since 2012)]{gpy2014}
{GPy}.
\newblock {GPy}: A gaussian process framework in python.
\newblock \url{http://github.com/SheffieldML/GPy}, since 2012.

\bibitem[Guedj(2019)]{guedj2019primer}
Benjamin Guedj.
\newblock A primer on pac-bayesian learning.
\newblock \emph{arXiv preprint arXiv:1901.05353}, 2019.

\bibitem[Hardt et~al.(2016)Hardt, Recht, and Singer]{pmlr-v48-hardt16}
Moritz Hardt, Ben Recht, and Yoram Singer.
\newblock Train faster, generalize better: Stability of stochastic gradient
  descent.
\newblock In \emph{Proceedings of The 33rd International Conference on Machine
  Learning}, volume~48 of \emph{Proceedings of Machine Learning Research},
  pages 1225--1234. PMLR, 20--22 Jun 2016.
\newblock URL \url{http://proceedings.mlr.press/v48/hardt16.html}.

\bibitem[Harvey et~al.(2017)Harvey, Liaw, and Mehrabian]{bartlett2017nearly}
Nick Harvey, Christopher Liaw, and Abbas Mehrabian.
\newblock Nearly-tight {VC}-dimension bounds for piecewise linear neural
  networks.
\newblock In \emph{Proceedings of the 2017 Conference on Learning Theory},
  volume~65 of \emph{Proceedings of Machine Learning Research}, pages
  1064--1068. PMLR, 07--10 Jul 2017.
\newblock URL \url{http://proceedings.mlr.press/v65/harvey17a.html}.

\bibitem[He et~al.(2016{\natexlab{a}})He, Zhang, Ren, and Sun]{he2016deep}
Kaiming He, Xiangyu Zhang, Shaoqing Ren, and Jian Sun.
\newblock Deep residual learning for image recognition.
\newblock In \emph{Proceedings of the IEEE conference on computer vision and
  pattern recognition}, pages 770--778, 2016{\natexlab{a}}.

\bibitem[He et~al.(2016{\natexlab{b}})He, Zhang, Ren, and Sun]{he2016identity}
Kaiming He, Xiangyu Zhang, Shaoqing Ren, and Jian Sun.
\newblock Identity mappings in deep residual networks.
\newblock In \emph{European conference on computer vision}, pages 630--645.
  Springer, 2016{\natexlab{b}}.

\bibitem[Henighan et~al.(2020)Henighan, Kaplan, Katz, Chen, Hesse, Jackson,
  Jun, Brown, Dhariwal, Gray, et~al.]{henighan2020scaling}
Tom Henighan, Jared Kaplan, Mor Katz, Mark Chen, Christopher Hesse, Jacob
  Jackson, Heewoo Jun, Tom~B Brown, Prafulla Dhariwal, Scott Gray, et~al.
\newblock Scaling laws for autoregressive generative modeling.
\newblock \emph{arXiv preprint arXiv:2010.14701}, 2020.

\bibitem[Hestness et~al.(2017)Hestness, Narang, Ardalani, Diamos, Jun,
  Kianinejad, Patwary, Ali, Yang, and Zhou]{hestness2017deep}
Joel Hestness, Sharan Narang, Newsha Ardalani, Gregory Diamos, Heewoo Jun,
  Hassan Kianinejad, Md~Patwary, Mostofa Ali, Yang Yang, and Yanqi Zhou.
\newblock Deep learning scaling is predictable, empirically.
\newblock \emph{arXiv preprint arXiv:1712.00409}, 2017.

\bibitem[Hinton and van Camp(1993)]{Hinton:1993:KNN:168304.168306}
Geoffrey~E. Hinton and Drew van Camp.
\newblock Keeping the neural networks simple by minimizing the description
  length of the weights.
\newblock In \emph{Proceedings of the Sixth Annual Conference on Computational
  Learning Theory}, COLT '93, pages 5--13, New York, NY, USA, 1993. ACM.
\newblock ISBN 0-89791-611-5.
\newblock \doi{10.1145/168304.168306}.
\newblock URL \url{http://doi.acm.org/10.1145/168304.168306}.

\bibitem[Hochreiter and Schmidhuber(1997)]{hochreiter1997flat}
Sepp Hochreiter and J{\"u}rgen Schmidhuber.
\newblock Flat minima.
\newblock \emph{Neural Computation}, 9\penalty0 (1):\penalty0 1--42, 1997.

\bibitem[Hoffer et~al.(2017)Hoffer, Hubara, and Soudry]{hoffer2017train}
Elad Hoffer, Itay Hubara, and Daniel Soudry.
\newblock Train longer, generalize better: closing the generalization gap in
  large batch training of neural networks.
\newblock In \emph{Advances in Neural Information Processing Systems}, pages
  1731--1741, 2017.

\bibitem[Howard et~al.(2017)Howard, Zhu, Chen, Kalenichenko, Wang, Weyand,
  Andreetto, and Adam]{howard2017mobilenets}
Andrew~G Howard, Menglong Zhu, Bo~Chen, Dmitry Kalenichenko, Weijun Wang,
  Tobias Weyand, Marco Andreetto, and Hartwig Adam.
\newblock Mobilenets: Efficient convolutional neural networks for mobile vision
  applications.
\newblock \emph{arXiv preprint arXiv:1704.04861}, 2017.

\bibitem[Huang et~al.(2017)Huang, Liu, Van Der~Maaten, and
  Weinberger]{huang2017densely}
Gao Huang, Zhuang Liu, Laurens Van Der~Maaten, and Kilian~Q Weinberger.
\newblock Densely connected convolutional networks.
\newblock In \emph{Proceedings of the IEEE conference on computer vision and
  pattern recognition}, pages 4700--4708, 2017.

\bibitem[Huang and Yau(2019)]{huang2019dynamics}
Jiaoyang Huang and Horng-Tzer Yau.
\newblock Dynamics of deep neural networks and neural tangent hierarchy.
\newblock \emph{arXiv preprint arXiv:1909.08156}, 2019.

\bibitem[Jacot et~al.(2018)Jacot, Gabriel, and Hongler]{jacot2018neural}
Arthur Jacot, Franck Gabriel, and Cl{\'e}ment Hongler.
\newblock Neural tangent kernel: Convergence and generalization in neural
  networks.
\newblock In \emph{Advances in neural information processing systems}, pages
  8571--8580, 2018.

\bibitem[Jastrzebski et~al.(2018)Jastrzebski, Kenton, Arpit, Ballas, Fischer,
  Bengio, and Storkey]{jastrzebski2018finding}
Stanislaw Jastrzebski, Zachary Kenton, Devansh Arpit, Nicolas Ballas, Asja
  Fischer, Yoshua Bengio, and Amos~J Storkey.
\newblock Finding flatter minima with sgd.
\newblock In \emph{ICLR (Workshop)}, 2018.

\bibitem[Jiang et~al.(2018)Jiang, Krishnan, Mobahi, and
  Bengio]{jiang2018predicting}
Yiding Jiang, Dilip Krishnan, Hossein Mobahi, and Samy Bengio.
\newblock Predicting the generalization gap in deep networks with margin
  distributions.
\newblock \emph{arXiv preprint arXiv:1810.00113}, 2018.

\bibitem[Jiang et~al.(2019)Jiang, Neyshabur, Mobahi, Krishnan, and
  Bengio]{jiang2019fantastic}
Yiding Jiang, Behnam Neyshabur, Hossein Mobahi, Dilip Krishnan, and Samy
  Bengio.
\newblock Fantastic generalization measures and where to find them.
\newblock \emph{arXiv preprint arXiv:1912.02178}, 2019.

\bibitem[Kaplan et~al.(2020)Kaplan, McCandlish, Henighan, Brown, Chess, Child,
  Gray, Radford, Wu, and Amodei]{kaplan2020scaling}
Jared Kaplan, Sam McCandlish, Tom Henighan, Tom~B Brown, Benjamin Chess, Rewon
  Child, Scott Gray, Alec Radford, Jeffrey Wu, and Dario Amodei.
\newblock Scaling laws for neural language models.
\newblock \emph{arXiv preprint arXiv:2001.08361}, 2020.

\bibitem[Keskar et~al.(2016)Keskar, Mudigere, Nocedal, Smelyanskiy, and
  Tang]{keskar2016large}
Nitish~Shirish Keskar, Dheevatsa Mudigere, Jorge Nocedal, Mikhail Smelyanskiy,
  and Ping Tak~Peter Tang.
\newblock On large-batch training for deep learning: Generalization gap and
  sharp minima.
\newblock \emph{CoRR}, abs/1609.04836, 2016.
\newblock URL \url{http://arxiv.org/abs/1609.04836}.

\bibitem[Koltchinskii(2011)]{koltchinskii2011oracle}
Vladimir Koltchinskii.
\newblock \emph{Oracle Inequalities in Empirical Risk Minimization and Sparse
  Recovery Problems: Ecole d’Et{\'e} de Probabilit{\'e}s de Saint-Flour
  XXXVIII-2008}, volume 2033.
\newblock Springer Science \& Business Media, 2011.

\bibitem[Krizhevsky et~al.()Krizhevsky, Nair, and Hinton]{cifar10dataset}
Alex Krizhevsky, Vinod Nair, and Geoffrey Hinton.
\newblock Cifar-10 (canadian institute for advanced research).
\newblock URL \url{http://www.cs.toronto.edu/~kriz/cifar.html}.

\bibitem[Krueger et~al.(2017)Krueger, Ballas, Jastrzebski, Arpit, Kanwal,
  Maharaj, Bengio, Fischer, Courville, Lacoste-Julien, and
  Bengio]{arpit2017closer}
David Krueger, Nicolas Ballas, Stanislaw Jastrzebski, Devansh Arpit,
  Maxinder~S. Kanwal, Tegan Maharaj, Emmanuel Bengio, Asja Fischer, Aaron
  Courville, Simon Lacoste-Julien, and Yoshua Bengio.
\newblock A closer look at memorization in deep networks.
\newblock Proceedings of the 34th International Conference on Machine Learning
  (ICML'17), 2017.
\newblock URL \url{https://arxiv.org/abs/1706.05394}.

\bibitem[Kuzborskij and Lampert(2017)]{kuzborskij2017data}
Ilja Kuzborskij and Christoph~H Lampert.
\newblock Data-dependent stability of stochastic gradient descent.
\newblock \emph{arXiv preprint arXiv:1703.01678}, 2017.

\bibitem[Langford(2005)]{langford2005tutorial}
John Langford.
\newblock Tutorial on practical prediction theory for classification.
\newblock \emph{Journal of machine learning research}, 6\penalty0
  (Mar):\penalty0 273--306, 2005.

\bibitem[Langford and Seeger(2001)]{langford2001bounds}
John Langford and Matthias Seeger.
\newblock Bounds for averaging classifiers.
\newblock 2001.

\bibitem[Langford and Shawe-Taylor(2003)]{langford2003pac}
John Langford and John Shawe-Taylor.
\newblock Pac-bayes \& margins.
\newblock In \emph{Advances in neural information processing systems}, pages
  439--446, 2003.

\bibitem[Lawrence et~al.(1998)Lawrence, Giles, and Tsoi]{lawrence1998size}
Steve Lawrence, C~Lee Giles, and Ah~Chung Tsoi.
\newblock What size neural network gives optimal generalization? convergence
  properties of backpropagation.
\newblock Technical report, 1998.

\bibitem[LeCun and Cortes(2010)]{lecun-mnisthandwrittendigit-2010}
Yann LeCun and Corinna Cortes.
\newblock {MNIST} handwritten digit database.
\newblock 2010.
\newblock URL \url{http://yann.lecun.com/exdb/mnist/}.

\bibitem[Lee et~al.(2017)Lee, Bahri, Novak, Schoenholz, Pennington, and
  Sohl-Dickstein]{lee2017deep}
Jaehoon Lee, Yasaman Bahri, Roman Novak, Samuel~S Schoenholz, Jeffrey
  Pennington, and Jascha Sohl-Dickstein.
\newblock Deep neural networks as gaussian processes.
\newblock \emph{arXiv preprint arXiv:1711.00165}, 2017.

\bibitem[Lee et~al.(2019)Lee, Xiao, Schoenholz, Bahri, Novak, Sohl-Dickstein,
  and Pennington]{lee2019wide}
Jaehoon Lee, Lechao Xiao, Samuel Schoenholz, Yasaman Bahri, Roman Novak, Jascha
  Sohl-Dickstein, and Jeffrey Pennington.
\newblock Wide neural networks of any depth evolve as linear models under
  gradient descent.
\newblock In \emph{Advances in neural information processing systems}, pages
  8572--8583, 2019.

\bibitem[Lee et~al.(2020)Lee, Schoenholz, Pennington, Adlam, Xiao, Novak, and
  Sohl-Dickstein]{lee2020finite}
Jaehoon Lee, Samuel Schoenholz, Jeffrey Pennington, Ben Adlam, Lechao Xiao,
  Roman Novak, and Jascha Sohl-Dickstein.
\newblock Finite versus infinite neural networks: an empirical study.
\newblock \emph{Advances in Neural Information Processing Systems}, 33, 2020.

\bibitem[Li et~al.(2018)Li, Liu, Yin, Zhang, Ding, and Wang]{li2018multi}
Jian Li, Yong Liu, Rong Yin, Hua Zhang, Lizhong Ding, and Weiping Wang.
\newblock Multi-class learning: from theory to algorithm.
\newblock In \emph{Advances in Neural Information Processing Systems}, pages
  1586--1595, 2018.

\bibitem[Li et~al.(2019)Li, Luo, and Qiao]{li2019generalization}
Jian Li, Xuanyuan Luo, and Mingda Qiao.
\newblock On generalization error bounds of noisy gradient methods for
  non-convex learning.
\newblock \emph{arXiv preprint arXiv:1902.00621}, 2019.

\bibitem[Littlestone and Warmuth(1986)]{littlestone1986relating}
Nick Littlestone and Manfred Warmuth.
\newblock Relating data compression and learnability.
\newblock 1986.

\bibitem[London(2017)]{london2017pac}
Ben London.
\newblock A pac-bayesian analysis of randomized learning with application to
  stochastic gradient descent.
\newblock In \emph{Advances in Neural Information Processing Systems}, pages
  2931--2940, 2017.

\bibitem[MacKay(1992)]{mackay1992practical}
David~JC MacKay.
\newblock A practical bayesian framework for backpropagation networks.
\newblock \emph{Neural computation}, 4\penalty0 (3):\penalty0 448--472, 1992.

\bibitem[MacKay and Mac~Kay(2003)]{mackay2003information}
David~JC MacKay and David~JC Mac~Kay.
\newblock \emph{Information theory, inference and learning algorithms}.
\newblock Cambridge university press, 2003.

\bibitem[Maddox et~al.(2020)Maddox, Benton, and Wilson]{maddox2020rethinking}
Wesley~J Maddox, Gregory Benton, and Andrew~Gordon Wilson.
\newblock Rethinking parameter counting in deep models: Effective
  dimensionality revisited.
\newblock \emph{arXiv preprint arXiv:2003.02139}, 2020.

\bibitem[Matthews et~al.(2018)Matthews, Rowland, Hron, Turner, and
  Ghahramani]{matthews2018gaussian}
Alexander G de~G Matthews, Mark Rowland, Jiri Hron, Richard~E Turner, and
  Zoubin Ghahramani.
\newblock Gaussian process behaviour in wide deep neural networks.
\newblock \emph{arXiv preprint arXiv:1804.11271}, 2018.

\bibitem[Maurer(2004)]{maurer2004note}
Andreas Maurer.
\newblock A note on the pac bayesian theorem.
\newblock \emph{arXiv preprint cs/0411099}, 2004.

\bibitem[McAllester(2013)]{mcallester2013pac}
David McAllester.
\newblock A pac-bayesian tutorial with a dropout bound.
\newblock \emph{arXiv preprint arXiv:1307.2118}, 2013.

\bibitem[McAllester(1998)]{mcallester1998some}
David~A McAllester.
\newblock Some pac-bayesian theorems.
\newblock In \emph{Proceedings of the eleventh annual conference on
  Computational learning theory}, pages 230--234. ACM, 1998.

\bibitem[McAllester(1999)]{mcallester1999pac}
David~A McAllester.
\newblock Pac-bayesian model averaging.
\newblock In \emph{COLT}, volume~99, pages 164--170. Citeseer, 1999.

\bibitem[Mingard et~al.(2020)Mingard, Valle-P{\'e}rez, Skalse, and
  Louis]{mingard2020sgd}
Chris Mingard, Guillermo Valle-P{\'e}rez, Joar Skalse, and Ard~A Louis.
\newblock Is sgd a bayesian sampler? well, almost.
\newblock \emph{arXiv preprint arXiv:2006.15191}, 2020.

\bibitem[Mou et~al.(2018)Mou, Wang, Zhai, and Zheng]{mou2018generalization}
Wenlong Mou, Liwei Wang, Xiyu Zhai, and Kai Zheng.
\newblock Generalization bounds of sgld for non-convex learning: Two
  theoretical viewpoints.
\newblock In \emph{Conference On Learning Theory}, pages 605--638, 2018.

\bibitem[Nagarajan and Kolter(2019)]{nagarajan2019uniform}
Vaishnavh Nagarajan and J~Zico Kolter.
\newblock Uniform convergence may be unable to explain generalization in deep
  learning.
\newblock In \emph{Advances in Neural Information Processing Systems}, pages
  11611--11622, 2019.

\bibitem[Nagarajan and Kolter(2018)]{nagarajan2018deterministic}
Vaishnavh Nagarajan and Zico Kolter.
\newblock Deterministic pac-bayesian generalization bounds for deep networks
  via generalizing noise-resilience.
\newblock 2018.

\bibitem[Nakkiran et~al.(2019)Nakkiran, Kaplun, Bansal, Yang, Barak, and
  Sutskever]{nakkiran2019deep}
Preetum Nakkiran, Gal Kaplun, Yamini Bansal, Tristan Yang, Boaz Barak, and Ilya
  Sutskever.
\newblock Deep double descent: Where bigger models and more data hurt.
\newblock \emph{arXiv preprint arXiv:1912.02292}, 2019.

\bibitem[Natarajan(1989)]{natarajan1989learning}
Balas~K Natarajan.
\newblock On learning sets and functions.
\newblock \emph{Machine Learning}, 4\penalty0 (1):\penalty0 67--97, 1989.

\bibitem[Neal(2019)]{neal2019bias}
Brady Neal.
\newblock On the bias-variance tradeoff: Textbooks need an update.
\newblock \emph{arXiv preprint arXiv:1912.08286}, 2019.

\bibitem[Neal et~al.(2018)Neal, Mittal, Baratin, Tantia, Scicluna,
  Lacoste-Julien, and Mitliagkas]{neal2018modern}
Brady Neal, Sarthak Mittal, Aristide Baratin, Vinayak Tantia, Matthew Scicluna,
  Simon Lacoste-Julien, and Ioannis Mitliagkas.
\newblock A modern take on the bias-variance tradeoff in neural networks.
\newblock \emph{arXiv preprint arXiv:1810.08591}, 2018.

\bibitem[Neal(1994)]{neal1994priors}
Radford~M Neal.
\newblock Priors for infinite networks (tech. rep. no. crg-tr-94-1).
\newblock \emph{University of Toronto}, 1994.

\bibitem[Negrea et~al.(2019)Negrea, Dziugaite, and Roy]{negrea2019defense}
Jeffrey Negrea, Gintare~Karolina Dziugaite, and Daniel~M Roy.
\newblock In defense of uniform convergence: Generalization via derandomization
  with an application to interpolating predictors.
\newblock \emph{arXiv preprint arXiv:1912.04265}, 2019.

\bibitem[Neyshabur et~al.(2015)Neyshabur, Tomioka, and
  Srebro]{neyshabur2015norm}
Behnam Neyshabur, Ryota Tomioka, and Nathan Srebro.
\newblock Norm-based capacity control in neural networks.
\newblock In \emph{Conference on Learning Theory}, pages 1376--1401, 2015.

\bibitem[Neyshabur et~al.(2017)Neyshabur, Bhojanapalli, McAllester, and
  Srebro]{neyshabur2017exploring}
Behnam Neyshabur, Srinadh Bhojanapalli, David McAllester, and Nati Srebro.
\newblock Exploring generalization in deep learning.
\newblock In \emph{Advances in Neural Information Processing Systems}, pages
  5949--5958, 2017.

\bibitem[Neyshabur et~al.(2018{\natexlab{a}})Neyshabur, Bhojanapalli, and
  Srebro]{neyshabur2018a}
Behnam Neyshabur, Srinadh Bhojanapalli, and Nathan Srebro.
\newblock A {PAC}-bayesian approach to spectrally-normalized margin bounds for
  neural networks.
\newblock In \emph{International Conference on Learning Representations},
  2018{\natexlab{a}}.
\newblock URL \url{https://openreview.net/forum?id=Skz_WfbCZ}.

\bibitem[Neyshabur et~al.(2018{\natexlab{b}})Neyshabur, Li, Bhojanapalli,
  LeCun, and Srebro]{neyshabur2018towards}
Behnam Neyshabur, Zhiyuan Li, Srinadh Bhojanapalli, Yann LeCun, and Nathan
  Srebro.
\newblock Towards understanding the role of over-parametrization in
  generalization of neural networks.
\newblock \emph{arXiv preprint arXiv:1805.12076}, 2018{\natexlab{b}}.

\bibitem[Neyshabur et~al.(2019)Neyshabur, Li, Bhojanapalli, LeCun, and
  Srebro]{neyshabur2018the}
Behnam Neyshabur, Zhiyuan Li, Srinadh Bhojanapalli, Yann LeCun, and Nathan
  Srebro.
\newblock The role of over-parametrization in generalization of neural
  networks.
\newblock In \emph{International Conference on Learning Representations}, 2019.
\newblock URL \url{https://openreview.net/forum?id=BygfghAcYX}.

\bibitem[Novak et~al.(2018)Novak, Xiao, Lee, Bahri, Abolafia, Pennington, and
  Sohl-Dickstein]{novak2018bayesian}
Roman Novak, Lechao Xiao, Jaehoon Lee, Yasaman Bahri, Daniel~A Abolafia,
  Jeffrey Pennington, and Jascha Sohl-Dickstein.
\newblock Bayesian deep convolutional neural networks with many channels are
  gaussian processes.
\newblock \emph{arXiv preprint arXiv:1810.05148}, 2018.

\bibitem[Novak et~al.(2019)Novak, Xiao, Hron, Lee, Alemi, Sohl-Dickstein, and
  Schoenholz]{novak2019neural}
Roman Novak, Lechao Xiao, Jiri Hron, Jaehoon Lee, Alexander~A. Alemi, Jascha
  Sohl-Dickstein, and Samuel~S. Schoenholz.
\newblock Neural tangents: Fast and easy infinite neural networks in python,
  2019.

\bibitem[Park et~al.(2020)Park, Lee, Peng, Cao, and
  Sohl-Dickstein]{park2020towards}
Daniel~S Park, Jaehoon Lee, Daiyi Peng, Yuan Cao, and Jascha Sohl-Dickstein.
\newblock Towards nngp-guided neural architecture search.
\newblock \emph{arXiv preprint arXiv:2011.06006}, 2020.

\bibitem[Rasmussen(2004)]{rasmussen2004gaussian}
Carl~Edward Rasmussen.
\newblock Gaussian processes in machine learning.
\newblock In \emph{Advanced lectures on machine learning}, pages 63--71.
  Springer, 2004.

\bibitem[Rissanen(1978)]{rissanen1978modeling}
Jorma Rissanen.
\newblock Modeling by shortest data description.
\newblock \emph{Automatica}, 14\penalty0 (5):\penalty0 465--471, 1978.

\bibitem[Rivasplata et~al.(2020)Rivasplata, Kuzborskij, Szepesv{\'a}ri, and
  Shawe-Taylor]{rivasplata2020pac}
Omar Rivasplata, Ilja Kuzborskij, Csaba Szepesv{\'a}ri, and John Shawe-Taylor.
\newblock Pac-bayes analysis beyond the usual bounds.
\newblock \emph{arXiv preprint arXiv:2006.13057}, 2020.

\bibitem[Rosenfeld et~al.(2019)Rosenfeld, Rosenfeld, Belinkov, and
  Shavit]{rosenfeld2019constructive}
Jonathan~S Rosenfeld, Amir Rosenfeld, Yonatan Belinkov, and Nir Shavit.
\newblock A constructive prediction of the generalization error across scales.
\newblock \emph{arXiv preprint arXiv:1909.12673}, 2019.

\bibitem[Schaper and Louis(2014)]{schaper2014arrival}
Steffen Schaper and Ard~A Louis.
\newblock The arrival of the frequent: how bias in genotype-phenotype maps can
  steer populations to local optima.
\newblock \emph{PloS one}, 9\penalty0 (2):\penalty0 e86635, 2014.

\bibitem[Shalev-Shwartz and Ben-David(2014)]{shalev2014understanding}
Shai Shalev-Shwartz and Shai Ben-David.
\newblock \emph{Understanding machine learning: From theory to algorithms}.
\newblock Cambridge university press, 2014.

\bibitem[Shalev-Shwartz et~al.(2010)Shalev-Shwartz, Shamir, Srebro, and
  Sridharan]{shalev2010learnability}
Shai Shalev-Shwartz, Ohad Shamir, Nathan Srebro, and Karthik Sridharan.
\newblock Learnability, stability and uniform convergence.
\newblock \emph{Journal of Machine Learning Research}, 11\penalty0
  (Oct):\penalty0 2635--2670, 2010.

\bibitem[Shawe-Taylor(2019)]{shawe2019primer}
Benjamin Guedj~John Shawe-Taylor.
\newblock A primer on pac-bayesian learning.
\newblock 2019.

\bibitem[Shawe-Taylor and Williamson(1997)]{shawe1997pac}
John Shawe-Taylor and Robert~C Williamson.
\newblock A pac analysis of a bayesian estimator.
\newblock In \emph{Annual Workshop on Computational Learning Theory:
  Proceedings of the tenth annual conference on Computational learning theory},
  volume~6, pages 2--9, 1997.

\bibitem[Shawe-Taylor et~al.(1998)Shawe-Taylor, Bartlett, Williamson, and
  Anthony]{shawe1998structural}
John Shawe-Taylor, Peter~L Bartlett, Robert~C Williamson, and Martin Anthony.
\newblock Structural risk minimization over data-dependent hierarchies.
\newblock \emph{IEEE transactions on Information Theory}, 44\penalty0
  (5):\penalty0 1926--1940, 1998.

\bibitem[Simonyan and Zisserman(2015)]{simonyan2014very}
Karen Simonyan and Andrew Zisserman.
\newblock Very deep convolutional networks for large-scale image recognition.
\newblock In \emph{Proceedings of the International Conference on Learning
  Representations (ICLR)}, 2015.
\newblock URL \url{https://arxiv.org/abs/1409.1556}.

\bibitem[Smith and Le(2017)]{smith2018bayesian}
Samuel~L. Smith and Quoc~V. Le.
\newblock A bayesian perspective on generalization and stochastic gradient
  descent.
\newblock \emph{CoRR}, abs/1710.06451, 2017.
\newblock URL \url{http://arxiv.org/abs/1710.06451}.

\bibitem[Sollich(2002)]{sollich2002gaussian}
Peter Sollich.
\newblock Gaussian process regression with mismatched models.
\newblock In \emph{Advances in Neural Information Processing Systems}, pages
  519--526, 2002.

\bibitem[Spigler et~al.(2019)Spigler, Geiger, and Wyart]{spigler2019asymptotic}
Stefano Spigler, Mario Geiger, and Matthieu Wyart.
\newblock Asymptotic learning curves of kernel methods: empirical data vs
  teacher-student paradigm.
\newblock \emph{arXiv preprint arXiv:1905.10843}, 2019.

\bibitem[Stumpf and Porter(2012)]{stumpf2012critical}
Michael~PH Stumpf and Mason~A Porter.
\newblock Critical truths about power laws.
\newblock \emph{Science}, 335\penalty0 (6069):\penalty0 665--666, 2012.

\bibitem[Valiant(1984)]{valiant1984theory}
Leslie~G Valiant.
\newblock A theory of the learnable.
\newblock \emph{Communications of the ACM}, 27\penalty0 (11):\penalty0
  1134--1142, 1984.

\bibitem[Valle-P{\'e}rez et~al.(2018)Valle-P{\'e}rez, Camargo, and
  Louis]{valle2018deep}
Guillermo Valle-P{\'e}rez, Chico~Q Camargo, and Ard~A Louis.
\newblock Deep learning generalizes because the parameter-function map is
  biased towards simple functions.
\newblock \emph{arXiv preprint arXiv:1805.08522}, 2018.

\bibitem[Vapnik(1968)]{vapnik1968uniform}
Vladimir Vapnik.
\newblock On the uniform convergence of relative frequencies of events to their
  probabilities.
\newblock In \emph{Doklady Akademii Nauk USSR}, volume 181, pages 781--787,
  1968.

\bibitem[Vapnik and Chervonenkis(1974)]{vapnik1974theory}
Vladimir Vapnik and Alexey Chervonenkis.
\newblock Theory of pattern recognition, 1974.

\bibitem[Vapnik(1995)]{vapnik1995nature}
Vladimir~N Vapnik.
\newblock The nature of statistical learning theory.
\newblock 1995.

\bibitem[Wei and Ma(2019{\natexlab{a}})]{wei2019data}
Colin Wei and Tengyu Ma.
\newblock Data-dependent sample complexity of deep neural networks via
  lipschitz augmentation.
\newblock \emph{arXiv preprint arXiv:1905.03684}, 2019{\natexlab{a}}.

\bibitem[Wei and Ma(2019{\natexlab{b}})]{wei2019improved}
Colin Wei and Tengyu Ma.
\newblock Improved sample complexities for deep networks and robust
  classification via an all-layer margin.
\newblock \emph{arXiv preprint arXiv:1910.04284}, 2019{\natexlab{b}}.

\bibitem[Wei and Schwab(2019)]{wei2019noise}
Mingwei Wei and David~J Schwab.
\newblock How noise affects the hessian spectrum in overparameterized neural
  networks.
\newblock \emph{arXiv preprint arXiv:1910.00195}, 2019.

\bibitem[Welling and Teh(2011)]{welling2011bayesian}
Max Welling and Yee~W Teh.
\newblock Bayesian learning via stochastic gradient langevin dynamics.
\newblock In \emph{Proceedings of the 28th international conference on machine
  learning (ICML-11)}, pages 681--688, 2011.

\bibitem[Wilson and Izmailov(2020)]{wilson2020bayesian}
Andrew~Gordon Wilson and Pavel Izmailov.
\newblock Bayesian deep learning and a probabilistic perspective of
  generalization.
\newblock \emph{arXiv preprint arXiv:2002.08791}, 2020.

\bibitem[Wolpert and Waters(1994)]{wolpert1994relationship}
David~H Wolpert and R~Waters.
\newblock The relationship between pac, the statistical physics framework, the
  bayesian framework, and the vc framework.
\newblock In \emph{In}. Citeseer, 1994.

\bibitem[Wu et~al.(2017)Wu, Zhu, et~al.]{wu2017towards}
Lei Wu, Zhanxing Zhu, et~al.
\newblock Towards understanding generalization of deep learning: Perspective of
  loss landscapes.
\newblock \emph{arXiv preprint arXiv:1706.10239}, 2017.

\bibitem[Xiao et~al.(2017)Xiao, Rasul, and Vollgraf]{xiao2017fashion}
Han Xiao, Kashif Rasul, and Roland Vollgraf.
\newblock Fashion-mnist: a novel image dataset for benchmarking machine
  learning algorithms.
\newblock \emph{arXiv preprint arXiv:1708.07747}, 2017.

\bibitem[Xie et~al.(2017)Xie, Girshick, Doll{\'a}r, Tu, and
  He]{xie2017aggregated}
Saining Xie, Ross Girshick, Piotr Doll{\'a}r, Zhuowen Tu, and Kaiming He.
\newblock Aggregated residual transformations for deep neural networks.
\newblock In \emph{Proceedings of the IEEE conference on computer vision and
  pattern recognition}, pages 1492--1500, 2017.

\bibitem[Yaida(2019)]{yaida2019non}
Sho Yaida.
\newblock Non-gaussian processes and neural networks at finite widths.
\newblock \emph{arXiv preprint arXiv:1910.00019}, 2019.

\bibitem[Yang(2019{\natexlab{a}})]{yang2019scaling}
Greg Yang.
\newblock Scaling limits of wide neural networks with weight sharing: Gaussian
  process behavior, gradient independence, and neural tangent kernel
  derivation.
\newblock \emph{arXiv preprint arXiv:1902.04760}, 2019{\natexlab{a}}.

\bibitem[Yang(2019{\natexlab{b}})]{yang2019tensor}
Greg Yang.
\newblock Tensor programs i: Wide feedforward or recurrent neural networks of
  any architecture are gaussian processes.
\newblock \emph{arXiv preprint arXiv:1910.12478}, 2019{\natexlab{b}}.

\bibitem[Yang and Salman(2019)]{yang2019fine}
Greg Yang and Hadi Salman.
\newblock A fine-grained spectral perspective on neural networks.
\newblock \emph{arXiv preprint arXiv:1907.10599}, 2019.

\bibitem[Zhang et~al.(2017)Zhang, Bengio, Hardt, Recht, and
  Vinyals]{zhang2016understanding}
Chiyuan Zhang, Samy Bengio, Moritz Hardt, Benjamin Recht, and Oriol Vinyals.
\newblock Understanding deep learning requires rethinking generalization.
\newblock In \emph{Proceedings of the International Conference on Learning
  Representations (ICLR)}, 2017.
\newblock URL \url{https://arxiv.org/abs/1611.03530}.

\bibitem[Zhang et~al.(2018)Zhang, Saxe, Advani, and Lee]{zhang2018energy}
Yao Zhang, Andrew~M Saxe, Madhu~S Advani, and Alpha~A Lee.
\newblock Energy--entropy competition and the effectiveness of stochastic
  gradient descent in machine learning.
\newblock \emph{Molecular Physics}, pages 1--10, 2018.

\bibitem[Zhou et~al.(2018)Zhou, Veitch, Austern, Adams, and
  Orbanz]{zhou2018non}
Wenda Zhou, Victor Veitch, Morgane Austern, Ryan~P Adams, and Peter Orbanz.
\newblock Non-vacuous generalization bounds at the imagenet scale: a
  pac-bayesian compression approach.
\newblock 2018.

\bibitem[Zhou et~al.(2019)Zhou, Liang, and Zhang]{zhou2019understanding}
Yi~Zhou, Yingbin Liang, and Huishuai Zhang.
\newblock Understanding generalization error of sgd in nonconvex optimization.
\newblock \emph{stat}, 1050:\penalty0 7, 2019.

\bibitem[Zoph et~al.(2018)Zoph, Vasudevan, Shlens, and Le]{zoph2018learning}
Barret Zoph, Vijay Vasudevan, Jonathon Shlens, and Quoc~V Le.
\newblock Learning transferable architectures for scalable image recognition.
\newblock In \emph{Proceedings of the IEEE conference on computer vision and
  pattern recognition}, pages 8697--8710, 2018.

\end{thebibliography}

\appendix
\section{Proofs}

\subsection{Proof of \cref{pac_bayes_theorem}}
\label{proof:pac_bayes_theorem}

\begin{proof}

Consider a concept $c$ with generalization error $\epsilon(c)$. The probability that it has zero training error for a sample of $m$ instances is $P\left[c \in C(S)\right] = (1-\epsilon(c))^{-m}$. This is smaller or equal to some $\delta>0$ if
\begin{equation*}
    -\ln{(1-\epsilon(c))} \geq \frac{\ln{\frac{1}{\delta}}}{m}
\end{equation*}
This can be written as follows.
\begin{equation*}
    \forall c \forall \delta > 0 \forall^\delta S \left [c \in C(S)\text{ implies }-\ln{(1-\epsilon(c))} < \frac{\ln{\frac{1}{\delta}}}{m} \right]
\end{equation*}
By the quantifier reversal lemma \cite{mcallester1998some},

\begin{equation*}
    \forall^\delta S \forall \alpha>0 \forall^\alpha c \left [c \in C(S)\text{ implies }-\ln{(1-\epsilon(c))} < \frac{\ln{\frac{1}{\alpha\beta\delta}}}{(1-\beta)m} \right]
\end{equation*}

Let $B(S) \subseteq \mathcal{H}$ be the set of concepts violating the formula, then $P(B(S)) \leq \alpha$. Now, the conditional probability of a concept in $C(S)$ violating the formula is $\frac{P(B(S) \cap C(S))}{P(C(S))} \leq \frac{P(B(S))}{P(C(S))} \leq \frac{\alpha}{P(C(S))}$. This probability is therefore smaller than $\gamma$ if $\alpha=\gamma P(C(S))$. We thus get 

\begin{equation*}
    \forall^\delta S \forall \alpha>0 \forall^\gamma c \in C(S) \left [-\ln{(1-\epsilon(c))} < \frac{\ln{\frac{1}{P(C(S))}} + \ln{\frac{1}{\gamma\beta\delta}}}{(1-\beta)m} \right]
\end{equation*}
We get the final result, \cref{pac_bayes_theorem}, by choosing $\beta=\frac{1}{m}$.

\end{proof}

\subsection{Proof of \cref{lemma:logP_epsilon}}
\label{proof:logP_epsilon}

\begin{proof}
We begin by decomposing $P(S_{m+1})$ as follows

$$P(S_{m+1}) = P(S_m)(1-P_e(m))$$

Denote $P_{m}=P(S_{m})$. Take the natural logarithm of this and average over $S_m$ to obtain

$$ - \langle\log{\left(1-P_e(m)\right)}\rangle = \langle \log{P_m}\rangle - \langle \log{P_{m+1}} \rangle$$

Using $-\log{\left(1-P_e(m)\right)} \geq P_e(m)$, we obtain $ \langle \epsilon(m) \rangle \leq \langle \log{P_m}\rangle - \langle \log{P_{m+1}} \rangle$.

Now, using Markov's inequality, we know that $\mathbf{P}[P_e(m) > 0.6] < \frac{\langle\epsilon(m)\rangle}{0.6}$. By our assumption, $P_e(m) < E$. By letting $E'=-\log{(1-E)}$, and using the fact that $-\log{(1-x)}\leq x+x^2$ for $x<0.6$, we can write

$$-\langle \log{(1-P_e(m))} \rangle \leq \langle \epsilon(m)\rangle + \langle P_e(m)^2\rangle + \frac{\langle \epsilon(m) \rangle}{0.6}E'$$
As $P_e(m)\leq 1$, $\langle P_e(m)^2\rangle \leq \langle P_e(m)\rangle = \langle \epsilon(m)\rangle$, we obtain the first statement of the theorem.

To obtain the second, we use Chebysev's inequality, which says $\mathbf{P}[P_e(m) > 0.6] < \frac{\text{Var}(P_e(m))}{(0.6-\langle\epsilon(m)\rangle)^2}$. Following the same decomposition as above (into cases where $P_e(m)\leq 0.6$ and $P_e(m)>0.6$, we have

$$-\langle \log{(1-P_e(m))} \rangle \leq \langle \epsilon(m)\rangle + \langle P_e(m)^2\rangle + \frac{\text{Var}(P_e(m))}{(0.6-\langle\epsilon(m)\rangle)^2}E'$$

Now, $\langle P_e(m)^2\rangle = \text{Var}(P_e(m)) + \langle P_e(m) \rangle^2$ and $\langle P_e(m) \rangle^2 = o(\langle P_e(m) \rangle)$. Also, $\langle\epsilon(m)\rangle=o(1)$. Therefore, as long as $\text{Var}(P_e(m))=o(\langle \epsilon(m)\rangle)$, we have

$$-\langle \log{(1-P_e(m))} \rangle \leq \langle \epsilon(m)\rangle + o(\langle \epsilon(m)\rangle),$$
and the second part of the theorem follows.

\end{proof}

\subsection{Derivation of KL divergence inequality}
\label{proof:kl_div_inequality}

Consider a parametrized family of functions $\mathcal{F}=\{f_\theta\}_{\theta\in\Theta}$ parametrized by parameters $\theta \in \Theta$ (for example the set of functions expressible by a DNN). We can define the parameter-function map as in \citet{valle2018deep} $\mathcal{M}$ as:
\begin{align*}
    \mathcal{M} : \Theta &\to \mathcal{F}\\
    \theta &\mapsto f_\theta.
\end{align*}
We can then define the pre-image $\mathcal{M}^{-1}(f)$ which is the set of all $\theta$ which produce $f$ under $\mathcal{M}$. Furthermore, for any distribution $P$ on $\Theta$, we define $\bar{P}(f):=P(\mathcal{M}^{-1}(f))=\sum_{\theta\in\mathcal{M}^{-1}(f)}P(\theta)$

Let $Q$, and $P$ be two distributions on $\Theta$, then the KL divergence is
\begin{align*}
    KL(Q||P) &= -\sum_{\theta\in\Theta}Q(\theta)\log{\left(\frac{P(\theta)}{Q(\theta)}\right)}\\
    &=-\sum_{f\in\mathcal{F}}\bar{Q}(f)\sum_{\theta\in\mathcal{M}^{-1}(f)}\frac{Q(\theta)}{\bar{Q}(f)}\log{\left(\frac{P(\theta)}{Q(\theta)}\right)}\\
    &\geq - \sum_{f\in\mathcal{F}}\bar{Q}(f) \log{\left(\sum_{\theta\in\mathcal{M}^{-1}(f)}\frac{Q(\theta)}{\bar{Q}(f)}\frac{P(\theta)}{Q(\theta)}\right)}\\
    &= - \sum_{f\in\mathcal{F}}\bar{Q}(f) \log{\left(\frac{\bar{P}(f)}{\bar{Q}(f)}\right)}\\
    &=KL(\bar{Q}||\bar{P})
\end{align*}
where we the third line follows from Jensen's inequality. We thus see that the KL divergence between distributions in parameter space is no less than the KL divergence between the induced distributions in parameter space.

\section{Experimental details}
\label{experimental_details}

Here we describe in more detail the experiments carried out in \cref{experimental_results}. The results we plot for both test errors and PAC-Bayes bounds are averaged over 8 random initializations and samples of training set, except for some of the large $m$ points for which we we only used one random initialization due to limited computation time. We find that the variance in both the error and PAC-Bayes bound is very small and decreases for large $m$, as is expected from concentration inequalities (at least for the test error).

\subsection{Marginal likelihood calculation}
\label{app:marginal_likelihood}

The PAC-Bayes bound \cref{pac_bayes_bound} requires computing $P(C(S))$, the marginal likelihood of the data $S$ for prior of the Bayesian DNN, which we approximate  by the NNGP Gaussian process. More precisely, we follow the approach in \citet{valle2018deep,mingard2020sgd}, which we review in the following. Let $\tilde{P}$ be the distribution of the parameters of the Bayesian DNN, then
\begin{equation}
    P(S) = P_{\mathbf{\theta}\sim \tilde{P}} \left(\sigma(f_\mathbf{\theta}(x_1))=y_1,...,\sigma(f_\mathbf{\theta}(x_m))=y_m\right)
\end{equation}
where $f_\theta$ is the function the DNN expresses for those parameters (according to the \emph{parameter-function map} \citep{valle2018deep}). Because $y\in \{0,1\}$, we use a Bernoulli likelihood, with $P(y(x)=1)=p(x)$. The probability $p(x)$ is determined by $f_\theta(x)$ via a \emph{linking function}. We use a Heaviside linking function defined as
\begin{align}
    p(x)=\begin{cases}
        1\text{ if }f_\theta(x)>0 \\
        0\text{ otherwise}
    \end{cases}
\end{align}
$f_\theta$ itself is distributed according to a Gaussian process (the NNGP). The marginal likelihood $P(S)$ defined above for Bernoulli likelihood has no analytical expression, and we use the expectation-propagation (EP) approximation \citep{rasmussen2004gaussian}, as implemented in the GPy library \citep{gpy2014}. We use for $\tilde{P}$ the same intialization distribution used for training (see \cref{app:training_details}).

\subsection{Estimation of the kernel}
\label{app:kernel_estimation}

The NNGP itself is determined by the kernel of the Gaussian process. We use the method proposed in \citet{novak2018bayesian} to empirically estimate the kernel $K(x,x')$. We sample $M$ sets of parameters $\{\theta_m\}_{m=1}^M$ i.i.d., and compute the empirical estimate of the kernel by
\begin{equation}
    \tilde{K}(x,x') := \frac{\sigma_w^2}{Mn}\sum_{m=1}^M\sum_{c=1}^n (h^{L-1}_{\theta_m}(x))_c (h^{L-1}_{\theta_m}(x'))_c + \sigma_b^2
\end{equation}
where $(h^{L-1}_{\theta_m}(x))_c$ is the activations of the $c$th neuron in the last hidden layer ($L$ is the number of layers) for the network with parameters $\theta_m$ at input $x$, and $n$ is the number of neurons in the last hidden layer. $\sigma_w^2$ and $\sigma_b^2$ are the weight and bias variances, respectively. This can be derived from the expression, in equation 19\footnote{we don't have the spatial parameters $\alpha,\alpha'$ because we assume the last hidden layer is a ``flat layer'', and the last layer of weights if fully connected} in \citet{novak2018bayesian} by applying one step of the NNGP recurrence relation \citep{lee2017deep} to find the estimate $\tilde{K}$ of the real-valued network outputs $f_\theta(x)$ from the covariance of the previous layer's activations $K^{L-1}$, which simply corresponds to multiplying by $\sigma_w^2$ and adding $\sigma_b^2$.

In \citet{novak2018bayesian}, it is argued that this sampling incurs an error of order $1/\sqrt{Mn}$. In our experiments we use $M=0.1m$ and $n$ is defined by the architecture which we define in the next subsection

\subsection{Training, dataset, and architecture details}
\label{app:training_details}

\textbf{Training} For all the experiments, we train the networks with Adam, cross-entropy loss, using either batch size 32 or 256, with a learning rate of 0.01, and we train until the first time we reach exactly $0$ training error. The exception is VGG and CNN networks for which we used a learning rate between 1e-5 and 1e-5, which was necessary in order to converge to $0$ training error.

\textbf{Datasets}. We use binarized versions of five datasets: MNIST \citep{lecun-mnisthandwrittendigit-2010}, Fashion-MNIST \citep{xiao2017fashion}, EMNIST (ByClass split) \citep{cohen2017emnist}, KMNIST \citep{clanuwat2018deep}, and CIFAR10 (which we refer to as CIFAR) \citep{cifar10dataset}. The datasets where binarized by assigning label 0 to the first half of classes, and label 1 to the second half. The input images are normalized to lie in the range $[0,1]$.

\textbf{Architectures} We studied the following architectures
\begin{itemize}
    \item Fully connected networks (FCNs), with varying number of layers. The size of the hidden layers was 1024, and the nonlinearity was ReLU. The last layer was a single Softmax neuron. If the number of layers is not specified, we use $2$ hidden layers.
    \item Convolutional neural networks (CNNs), with varying number of layers. The number of filters was 1024, and the nonlinearity was ReLU. The output layer was a single Softmax neuron fully connected to a flattening layer. The filter sizes alternated between $(2,2)$ and $(5,5)$, and the padding between SAME and VALID, the strides were $1$ (same default settings as in the code for \cite{garriga2018convnets}). We used CNNs without pooling, and with two kinds of pooling: avg and max. Pooling was applied at every layer with a pooling size of $2$ and padding SAME, and globally at the last layer, before flattening. If the number of hidden layers is not specified, we use $4$ hidden layers.
    \item vgg16. Keras implementation (\url{https://keras.io/api/applications/vgg/#vgg16-function}). \cite{simonyan2014very}
    \item vgg19. Keras implementation (\url{https://keras.io/api/applications/vgg/#vgg19-function}). \cite{simonyan2014very}
    \item resnet50. Keras implementation (\url{https://keras.io/api/applications/resnet/#resnet50-function}). \cite{he2016deep}
    \item resnet101. Keras implementation (\url{https://keras.io/api/applications/resnet/#resnet101-function}). \cite{he2016deep}
    \item resnet152. Keras implementation (\url{https://keras.io/api/applications/resnet/#resnet152-function}). \cite{he2016deep}
    \item resnetv2\_50. Keras implementation (\url{https://keras.io/api/applications/resnet/#resnet50v2-function}). \cite{he2016identity}
    \item resnetv2\_101. Keras implementation (\url{https://keras.io/api/applications/resnet/#resnet101v2-function}). \cite{he2016identity}
    \item resnetv2\_152. Keras implementation (\url{https://keras.io/api/applications/resnet/#resnet152v2-function}). \cite{he2016identity}
    \item resnext50. Keras applications implementation (\url{https://github.com/keras-team/keras-applications/blob/master/keras_applications/resnext.py}). \cite{xie2017aggregated}
    \item resnetxt101. Keras applications implementation (\url{https://github.com/keras-team/keras-applications/blob/master/keras_applications/resnext.py}). \cite{xie2017aggregated}
    \item densenet121 Keras implementation (\url{https://keras.io/api/applications/densenet/#densenet121-function}). \cite{huang2017densely}
    \item densenet169 Keras implementation (\url{https://keras.io/api/applications/densenet/#densenet169-function}). \cite{huang2017densely}
    \item densenet201 Keras implementation (\url{https://keras.io/api/applications/densenet/#densenet201-function}). \cite{huang2017densely}
    \item mobilenetv2. Keras implementation (\url{https://keras.io/api/applications/mobilenet/#mobilenetv2-function}). \cite{howard2017mobilenets}
    \item nasnet. Keras implementation (\url{https://keras.io/api/applications/nasnet/#nasnetlarge-function}). \cite{zoph2018learning}
\end{itemize}

For vgg16,vgg19,resnets ,densenets, mobilenetv2, and nasnet, we used average global pooling in the last layer, unless otherwise specified. All networks weights, except normalization layer parameters, were initialized from a Gaussian with variance $\sigma_w^2 = 1.41$, and the biases were initalized with variance $\sigma_b^=0$. Normalization layers were initialized using their Keras default initialization.

\section{A note on binary versus multi-class classification}\label{binary_vs_multiclass}

In \cref{notation}, we narrowed the focus of our paper to the problem of binary classification. However, in practice many tasks involve multi-class classification, and losses other than the misclassification probability. Although it is beyond the scope of this paper to analyze how the different approaches extend to different losses and output spaces, here we briefly comment on some notable ones.

The VC dimension framework (section \ref{vc_dimension}) has an extension to multi-class classification in which the capacity grows with the Natarajan dimension rather than the VC dimension \citep{natarajan1989learning,daniely2011multiclass,shalev2014understanding}. There are also extensions to the other approaches we study in \cref{comparing_bounds}, including margin and stability bounds, which typically scale like $O(K^2/\sqrt{m})$, where $K$ is the number of classes \citep{li2018multi}. In particular, PAC-Bayesian analyses can be extended to the multi-class setting \citep{mcallester2013pac}. The marginal likelihood PAC-Bayes bound in \cref{pac_bayes_bound} could probably be extended to the multi-class setting by using the natural extensions of the missclassification loss and marginal likelihood to the multi-class setting. However, the computation of the marginal likelihood using the NNGP approximation (\cref{app:marginal_likelihood}) would be more involved. We leave such an extension for future work.

\section{Learning curves versus dataset for all architectures}
\label{extra_learning_cuves}




Here we show all the learning curves for the rest of architectures.  In \cref{fig:lc_resnets} and \cref{fig:lc_densenets} we show the learning curves for all the ResNets and DenseNets respectively, plotted separately for each dataset.  In \cref{fig:lc_all1} and \cref{fig:lc_all2} we show similar learning curve data, but plotted separately for each architecture.

\begin{figure}[H]
    \centering
    \includegraphics[width=1.0\textwidth]{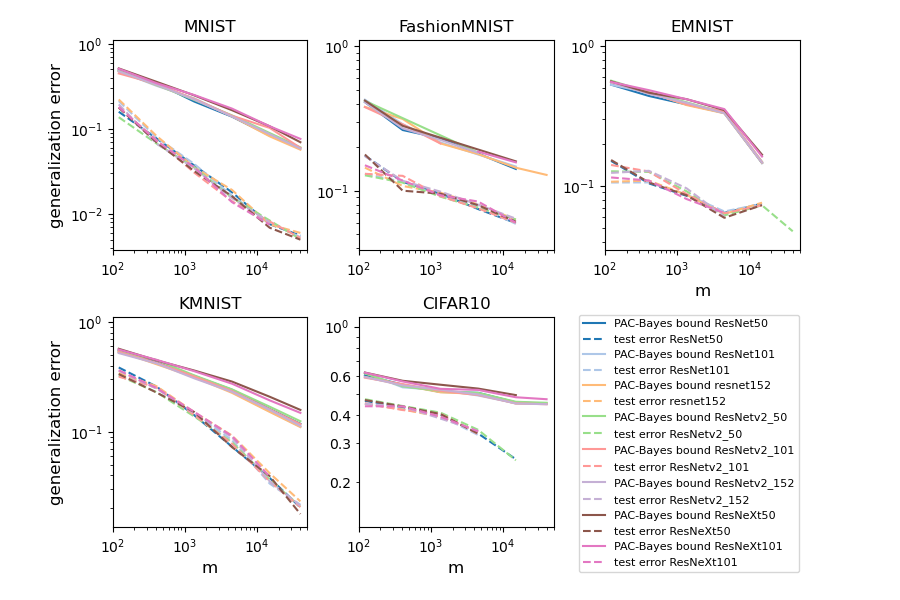}
    \caption{\textbf{Comparing different resnet architectures}. Learning curves for the test error and the PAC-Bayes bounds for different resnet architectures for different datasets. The DNNs were trained using Adam and batch size 32 to 0 training error.}
    \label{fig:lc_resnets}
\end{figure}

\begin{figure}[H]
    \centering
    \includegraphics[width=1.0\textwidth]{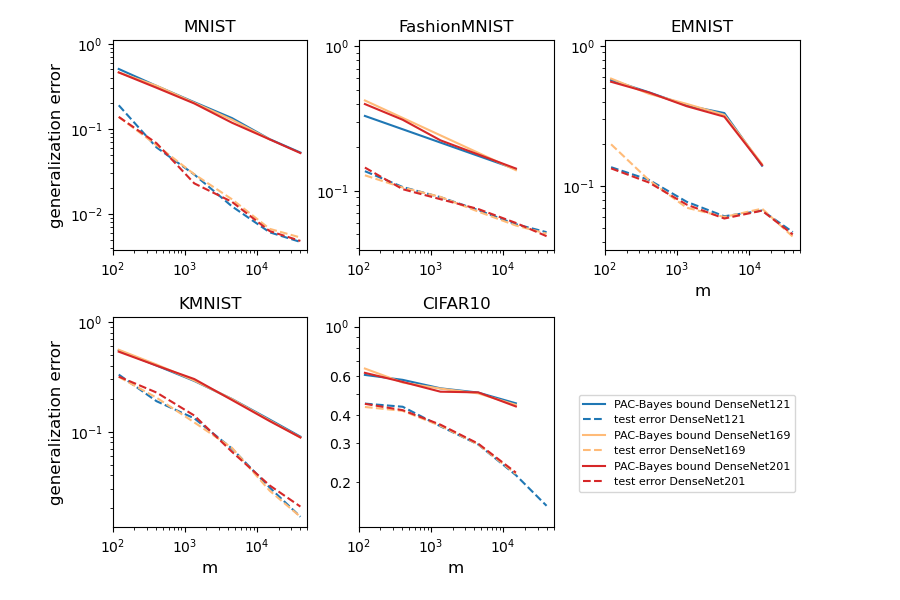}
    \caption{\textbf{Comparing different densenet architectures}. Learning curves for the empirical test error and the PAC-Bayes bounds for different densenet architectures for different datasets. The DNNs were trained using Adam and batch size 32 to 0 training error.}
    \label{fig:lc_densenets}
\end{figure}

\begin{figure}[H]
    \centering
    \includegraphics[width=1.0\textwidth]{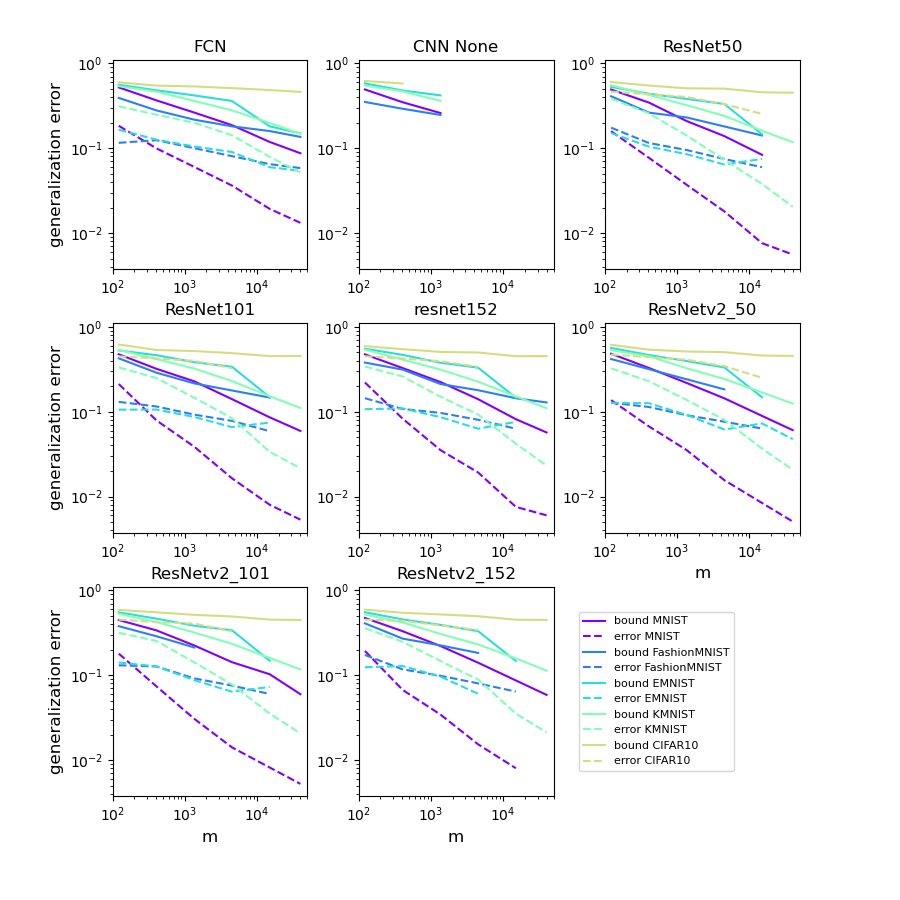}
    \caption{\textbf{Comparing datasets for different architectures (1/2)} Learning curves for the test error and the PAC-Bayes bounds for different architectures and for different datasets. The DNNs were trained using Adam and batch size 32 to 0 training error.}
    \label{fig:lc_all1}
\end{figure}

\begin{figure}[H]
    \centering
    \includegraphics[width=1.0\textwidth]{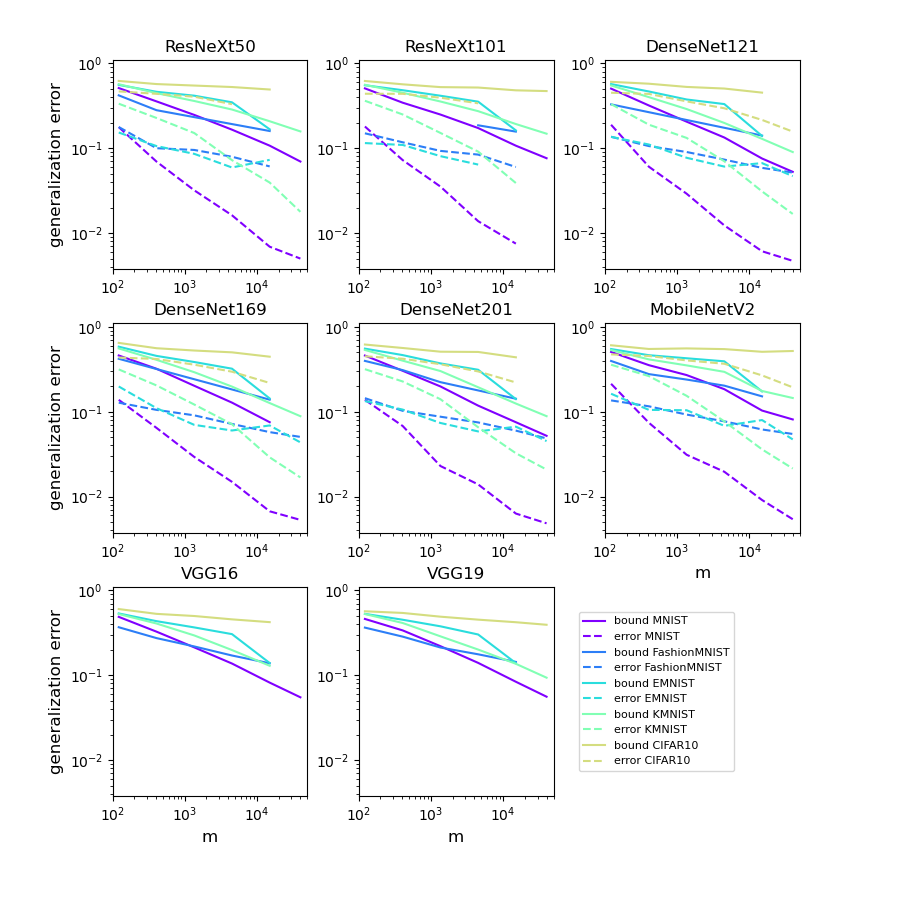}
    \caption{\textbf{Comparing datasets for different architectures (2/2)}Learning curves for the test error and the PAC-Bayes bounds for different architectures and for different datasets. The DNNs were trained using Adam and batch size 32 to 0 training error.}
    \label{fig:lc_all2}
\end{figure}

\section{Learning curves for batch size 256}
\label{app:lc_batch256}

In this section, we show the results of the same set experiments as in \cref{learning_curves}, but performed with batch size 256. Note that in some cases we have more or less data relative for batch 32, depending on which experiments could finish within our compute budget. We observe qualitatively similar results as for batch 32.

\begin{figure}[H]
    \centering
    \includegraphics[width=1.0\textwidth]{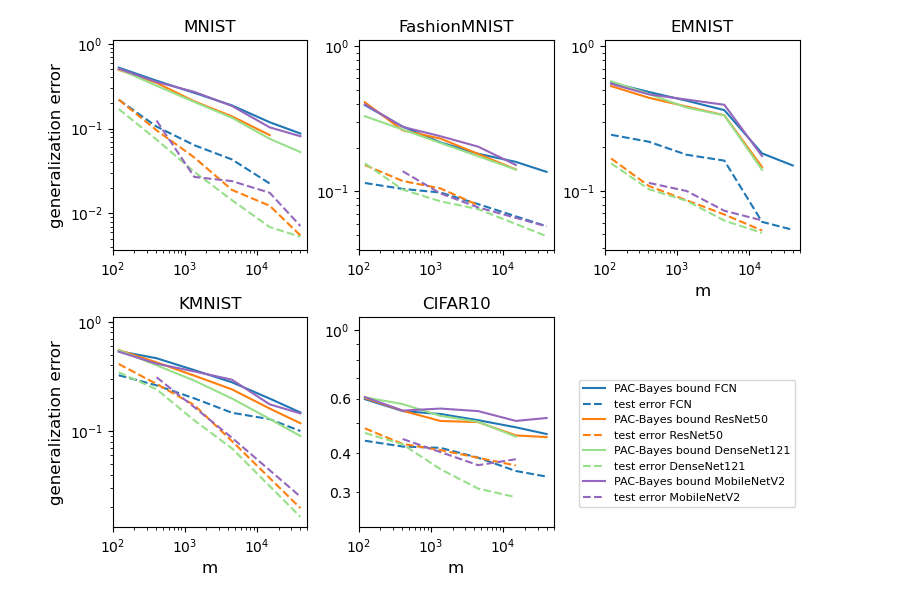}
    \caption{\textbf{Comparing different architecture families (batch size 256)} Learning curves for the test error and the PAC-Bayes bounds for representative architectures for different datasets. The DNNs were trained using Adam and batch size 256 to 0 training error. We see the different architectures show similar learning curve power law exponent, which is matched closely by the PAC-Bayes bound.}
    \label{fig:lc_somenets_256}
\end{figure}

\begin{figure}[H]
    \centering
    \includegraphics[width=1.0\textwidth]{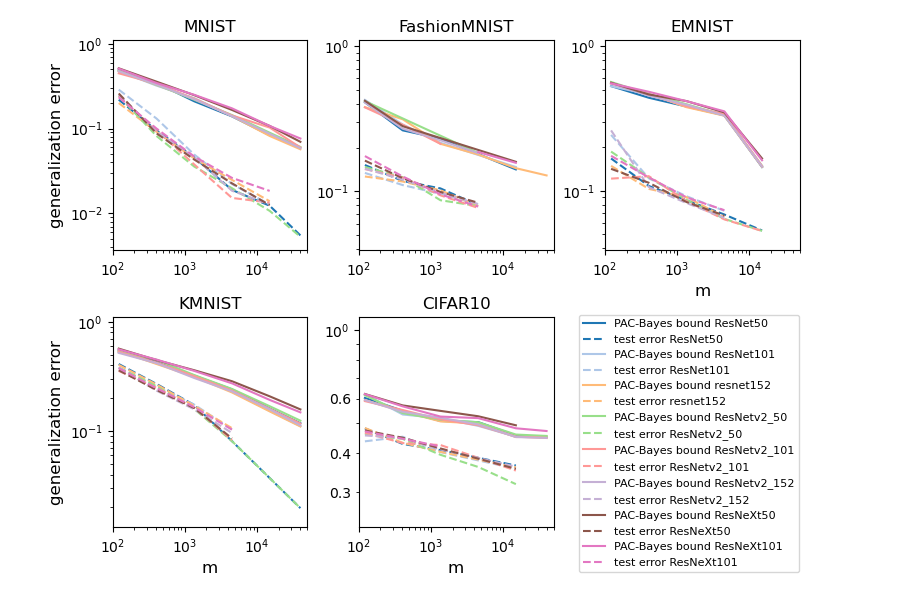}
    \caption{\textbf{Comparing different resnet architectures (batch size 256)} Learning curves for the test error and the PAC-Bayes bounds for different resnet architectures for different datasets. The DNNs were trained using Adam and batch size 256 to 0 training error.}
    \label{fig:lc_resnets_256}
\end{figure}

\begin{figure}[H]
    \centering
    \includegraphics[width=1.0\textwidth]{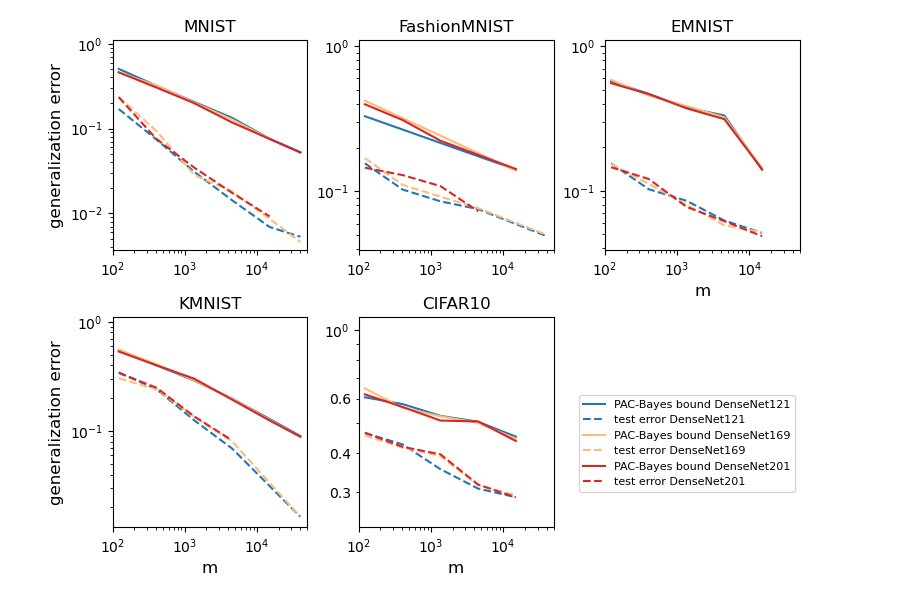}
    \caption{\textbf{Comparing different densenet architectures (batch size 256)} Learning curves for the empirical test error and the PAC-Bayes bounds for different densenet architectures for different datasets. The DNNs were trained using Adam and batch size 256 to 0 training error.}
    \label{fig:lc_densenets_256}
\end{figure}


\begin{figure}[H]
    \centering
    \includegraphics[width=1.0\textwidth]{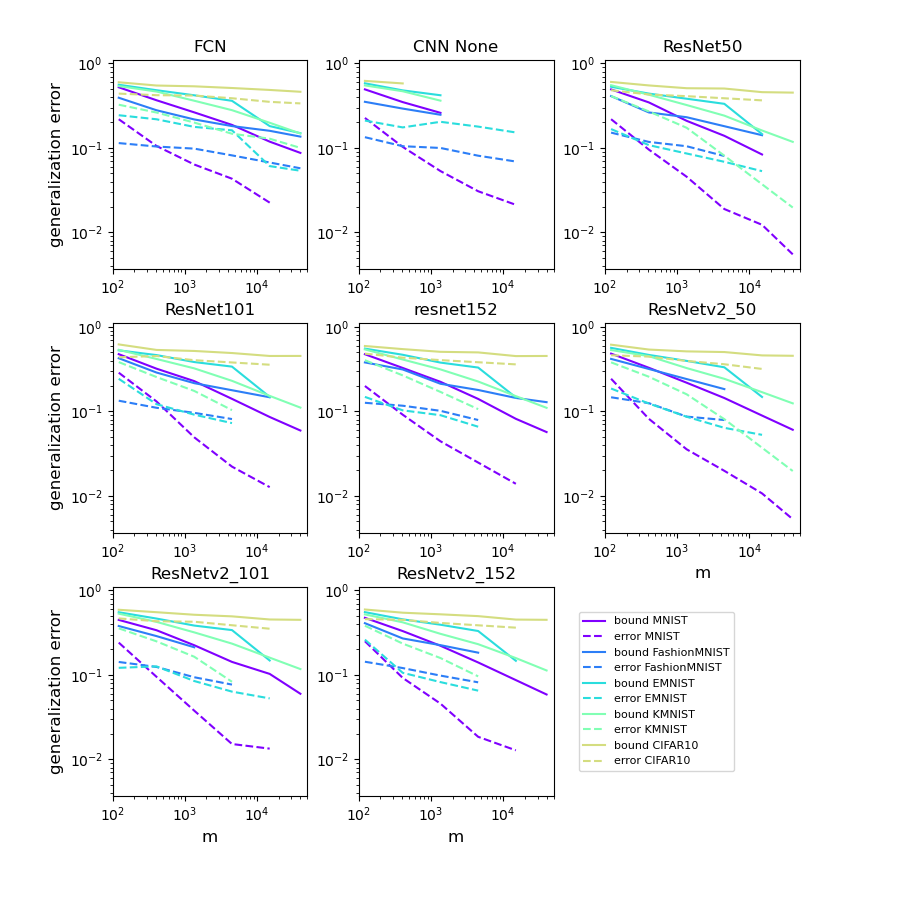}
    \caption{\textbf{Comparing datasets for different architectures (batch 256) (1/2)} Learning curves for the test error and the PAC-Bayes bounds for different architectures and for different datasets. The DNNs were trained using Adam and batch size 256 to 0 training error.}
    \label{fig:lc_all1_256}
\end{figure}

\begin{figure}[H]
    \centering
    \includegraphics[width=1.0\textwidth]{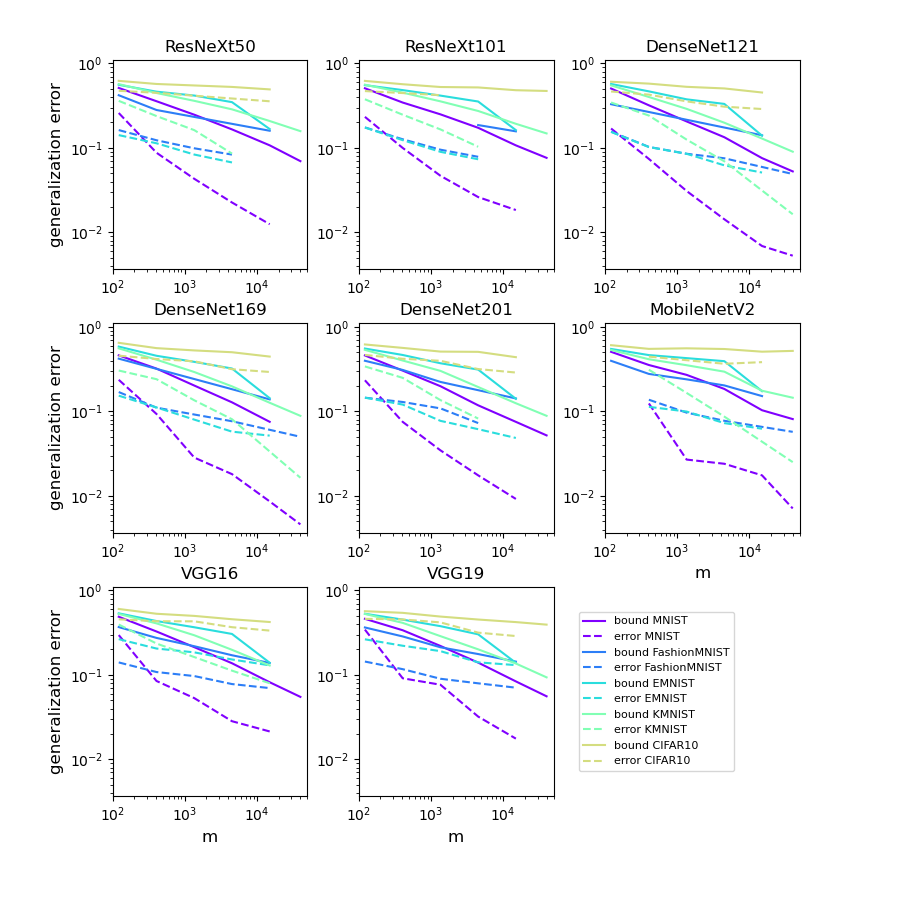}
    \caption{\textbf{Comparing datasets for different architectures (batch 256) (2/2)} Learning curves for the test error and the PAC-Bayes bounds for different architectures and for different datasets. The DNNs were trained using Adam and batch size 256 to 0 training error.}
    \label{fig:lc_all2_256}
\end{figure}


\begin{figure}[H]
    \centering
    \begin{subfigure}[t]{0.49\textwidth}
    \includegraphics[width=1.0\textwidth]{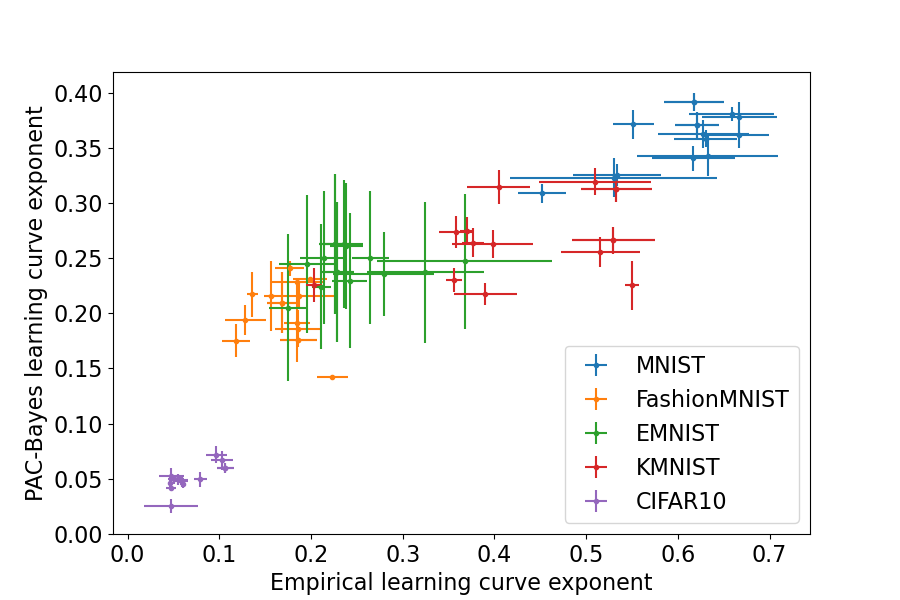}
    \caption{\label{fig:lc_exponents_estimation_exp_256}}
    \end{subfigure}
    \begin{subfigure}[t]{0.49\textwidth}
    \includegraphics[width=1.0\textwidth]{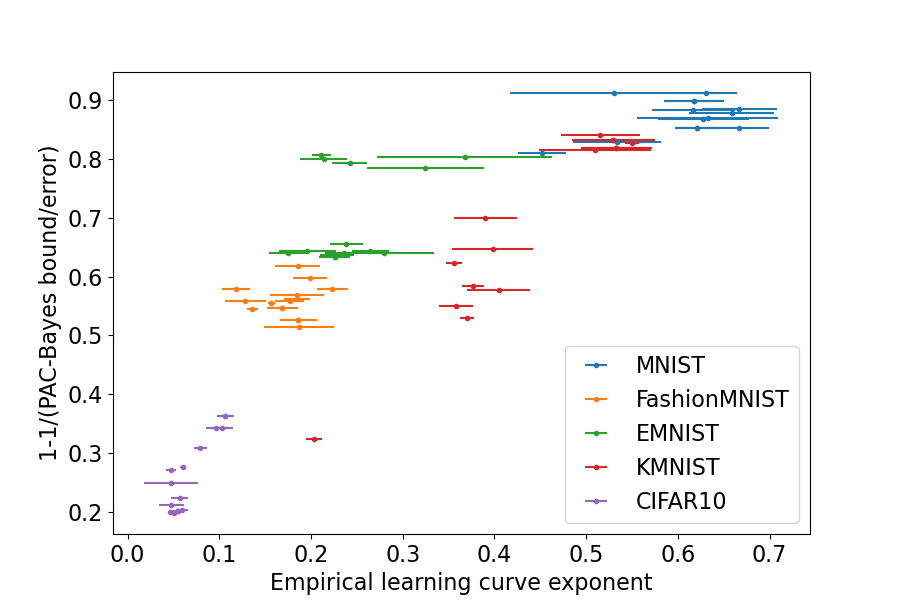}
    \caption{\label{fig:lc_exponents_estimation_ratio_256}}
    \end{subfigure}
    \caption{\textbf{Learning curve exponents of test error vs learning curve exponent estimated from PAC-Bayes bound for batch 256}. In both plots, each marker shows the learning curve exponent, $\alpha$ obtained from a linear fit to either the learning curve corresponding to $\log{\epsilon}$ vs $\log{m}$ or the estimated exponent from the PAC-Bayes bound, for all the architectures and datasets in \cref{experimental_details}. \textbf{(a)} The exponent is estimated from a linear fit to the log of the PAC-Bayes bound vs $\log{m}$. \textbf{(b)} The exponent is estimated from $C'$, which is the ratio of the PAC-Bayes bound and the error (see \cref{main_optimality_theorem}).
    The error bars are estimated standard errors from the linear fits. For the ratio estimate the errors due to fluctuations in dataset are negligible. Note that the errors do not take into account errors due to the EP approximation. Note that the exponents cluster according to dataset. The outliers for MNIST and KMNIST are both the FCN. The DNNs were trained using Adam and batch size 256 to 0 training error.}
    \label{fig:lc_exponents_estimation_256}
\end{figure}

\section{Further results comparing architectures}
\label{app:error_bound_256}

In this section, we show further experiments relevant to the effect of changing architectures that was described in  \cref{error_vs_arch}.  \cref{full_error_bound_256_all} shows the same experiments plotting the bound versus test error for different architectures, as in \cref{error_vs_arch}, but for batch size 256. The results also show correlation but we have less data available for those experiments. For completeness, we also show in \cref{app:error_bound_with_cnn_fc} and in \cref{app:error_bound_with_cnn} experiments plotting the bound versus test error, as in \cref{error_vs_arch}, but including the CNN and FCN (which are often outliers) for batch 32 and batch 256, repsectively. 

 
We also plotted the error and bounds for other values of $m$, and the results are similar, except that the correlation disappears for too small values of $m$, depending on the dataset. Furthermore, CNNs and FCNs tend to have relatively large values of test error which are sometimes not well captured by the bound, so we chose not to show them on these plots, as to be able to clearly see the variation in test error among the majority of architectures.

\begin{figure}[H]
    \centering
\begin{subfigure}[t]{0.49\textwidth}
    \includegraphics[width=1.0\textwidth]{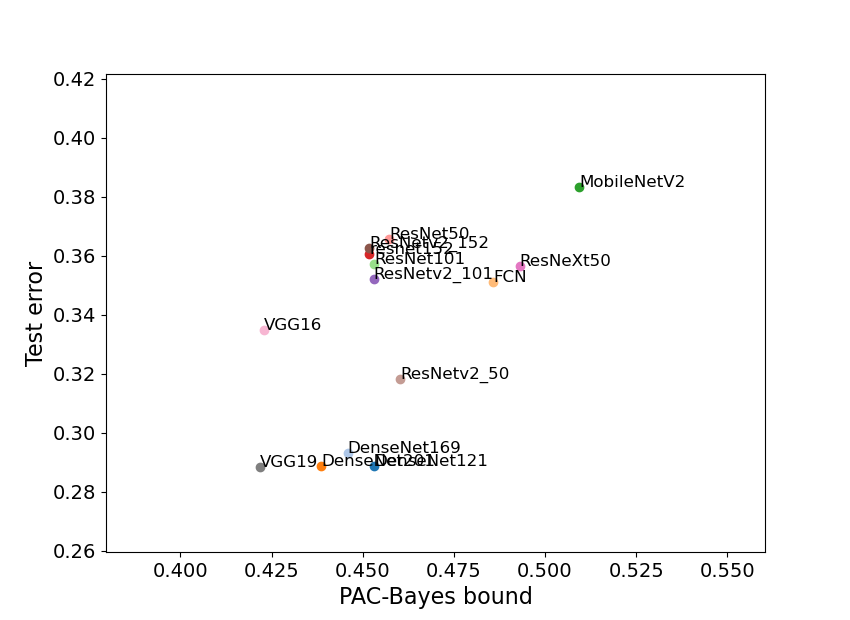}
    \caption{CIFAR10 \label{error_bound_256_cifar}}
\end{subfigure}
\foreach \dataset in {EMNIST,KMNIST,MNIST} {
\begin{subfigure}[t]{0.49\textwidth}
    \includegraphics[width=1.0\textwidth]{figures/bound_vs_error/batch256/without_cnn/bound_vs_error_\dataset_15026_256.png}
    \caption{\dataset \label{error_bound_256_\dataset}}
\end{subfigure}
}
\begin{subfigure}[t]{0.49\textwidth}
    \includegraphics[width=1.0\textwidth]{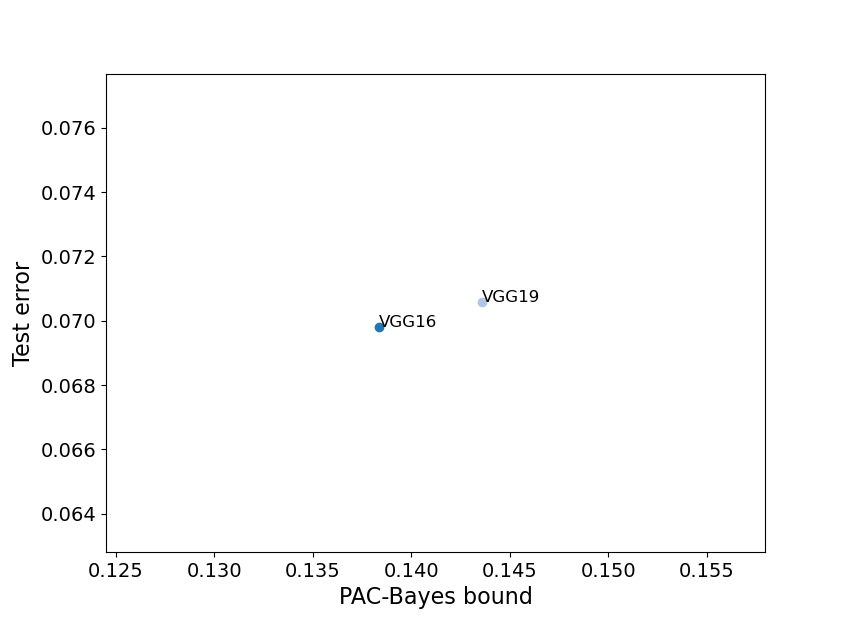}
    \caption{Fashion-MNIST \label{error_bound_256_fashion-mnist}}
\end{subfigure}
    \caption{\small{\textbf{PAC-Bayes bound versus test error for different models trained on a sample~of size 15k for the 5 datasets we study, with batch size 256}. CNNs are removed for clarity as they often have relatively extreme values of test error and/or bound. See \cref{app:error_bound_with_cnn} for plots with the CNNs included.}}
    \label{full_error_bound_256_all}
\end{figure}


\begin{figure}[H]
    \centering
\begin{subfigure}[t]{0.49\textwidth}
    \includegraphics[width=1.0\textwidth]{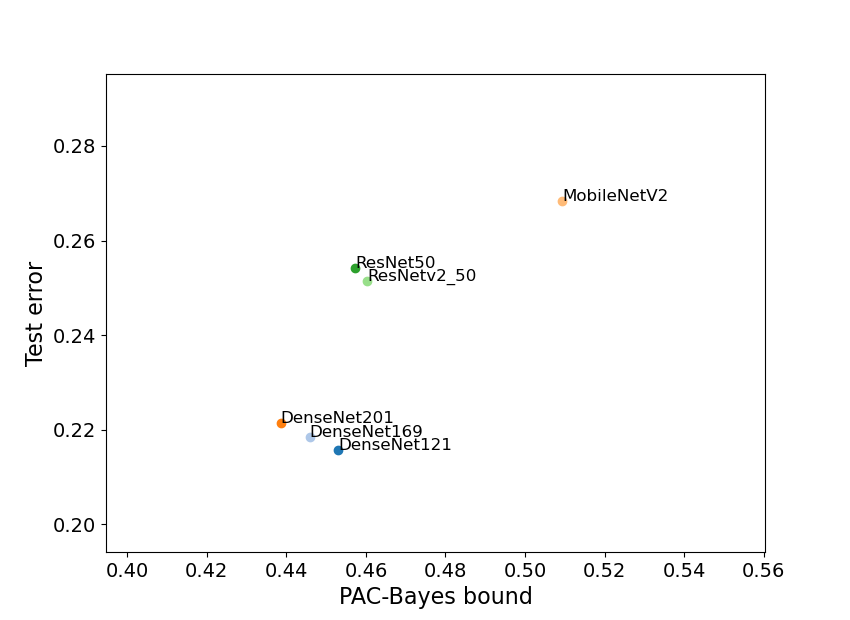}
    \caption{CIFAR10 \label{full_error_bound_32_with_cnn_fc_cifar}}
\end{subfigure}
\foreach \dataset in {EMNIST,KMNIST,MNIST} {
\begin{subfigure}[t]{0.49\textwidth}
    \includegraphics[width=1.0\textwidth]{figures/bound_vs_error/batch32/with_cnn_fc/bound_vs_error_\dataset_15026_32.png}
    \caption{\dataset \label{full_error_bound_32_with_cnn_fc_\dataset}}
\end{subfigure}
}
\begin{subfigure}[t]{0.49\textwidth}
    \includegraphics[width=1.0\textwidth]{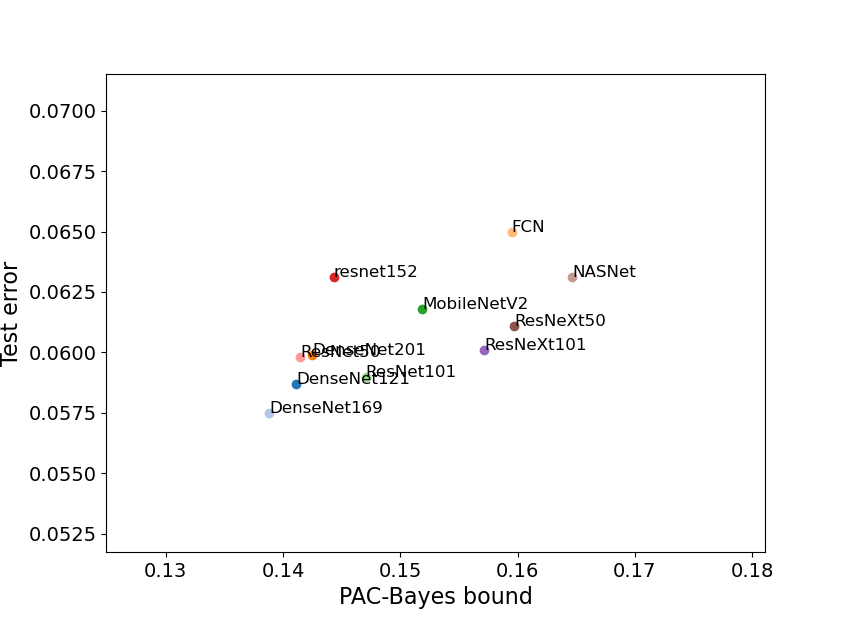}
    \caption{Fashion-MNIST \label{full_error_bound_32_with_cnn_fc_fashion-mnist}}
\end{subfigure}
    \caption{\small{\textbf{PAC-Bayes bound versus test error for different models trained on a sample~of size 15k for the 5 datasets we study, with batch size 32}. CNN and FC are included..}}
    \label{app:error_bound_with_cnn_fc}
\end{figure}


\begin{figure}[H]
    \centering
\begin{subfigure}[t]{0.49\textwidth}
    \includegraphics[width=1.0\textwidth]{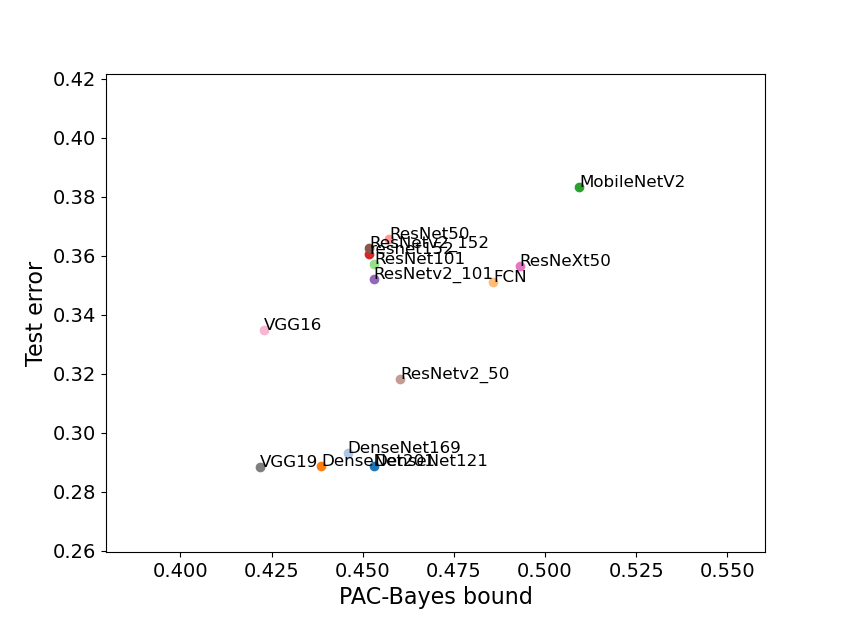}
    \caption{CIFAR10 \label{full_error_bound_256_with_cnn_cifar}}
\end{subfigure}
\foreach \dataset in {EMNIST,KMNIST,MNIST} {
\begin{subfigure}[t]{0.49\textwidth}
    \includegraphics[width=1.0\textwidth]{figures/bound_vs_error/batch256/with_cnn/bound_vs_error_\dataset_15026_256.png}
    \caption{\dataset \label{full_error_bound_256_with_cnn_\dataset}}
\end{subfigure}
}
\begin{subfigure}[t]{0.49\textwidth}
    \includegraphics[width=1.0\textwidth]{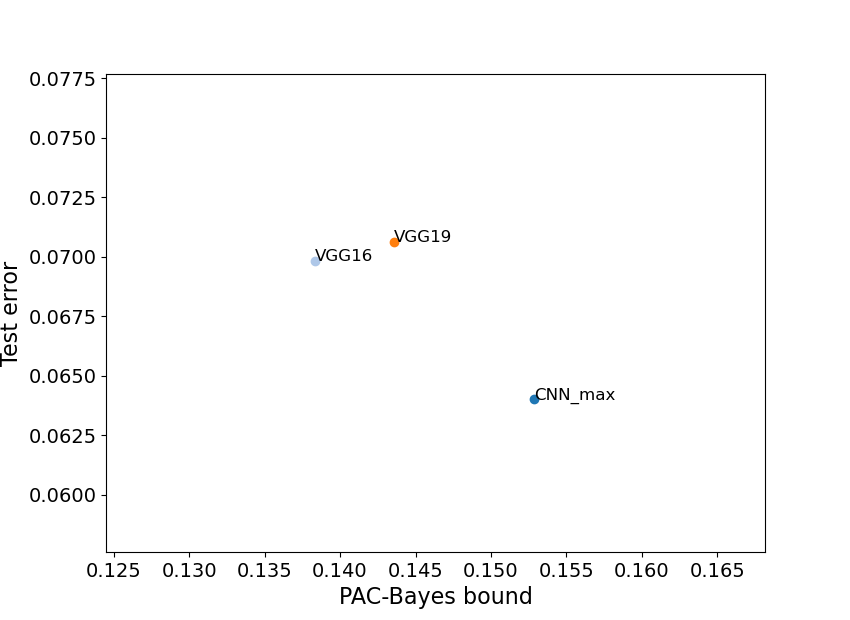}
    \caption{Fashion-MNIST \label{full_error_bound_256_with_cnn_fashion-mnist}}
\end{subfigure}
    \caption{\textbf{PAC-Bayes bound versus test error for different models trained on a sample~of size 15k for the 5 datasets we study, with batch size 256.} CNN are included}
   \label{app:error_bound_with_cnn}
\end{figure}

\end{document}